\documentclass[10pt]{article}
\usepackage[margin=1in]{geometry} 
\usepackage[utf8x]{inputenc} 
\usepackage[T1]{fontenc}    

\usepackage[dvipsnames]{xcolor}
\usepackage{graphicx}
\usepackage{caption}
\usepackage{subcaption}
\usepackage{enumitem}
\usepackage{wrapfig}
\usepackage{algorithm}
\usepackage[noend]{algpseudocode}
\usepackage{amsmath}
\usepackage{amssymb}
\usepackage{amsthm}
\usepackage[title]{appendix}
\usepackage{titletoc}
\usepackage[round, authoryear]{natbib}
\usepackage[
    colorlinks=true, 
    linkcolor=blue!70!black, 
    citecolor=blue!70!black,
    urlcolor=blue!70!black,
    breaklinks=true]{hyperref} 
\usepackage{authblk}
\usepackage{multirow}

\newtheorem{theorem}{Theorem}

\newtheorem{assumption}{Assumption}
\newtheorem{proposition}{Proposition}
\newtheorem{lemma}{Lemma}
\newtheorem{remark}{Remark}

\DeclareMathOperator{\E}{E}

\newcommand{\Gnl}{L_{\Gamma_{n}^{(t)}}}
\newcommand{\Gl}{L_{\Gamma^{(t)}}}
\newcommand{\Gt}{L_{\Gamma^{(0)}}}

\newcommand{\mbP}{\mathbb{P}}
\newcommand{\mbR}{\mathbb{R}}

\newcommand{\mcA}{\mathcal{A}}
\newcommand{\mcB}{\mathcal{B}}

\newcommand{\mcD}{\mathcal{D}}
\newcommand{\mcE}{\mathcal{E}}
\newcommand{\mcF}{\mathcal{F}}
\newcommand{\mcL}{\mathcal{L}}
\newcommand{\mcH}{\mathcal{H}}
\newcommand{\mcN}{\mathcal{N}}
\newcommand{\mcI}{\mathcal{I}}
\newcommand{\mcT}{\mathcal{T}}
\newcommand{\mcS}{\mathcal{S}}
\newcommand{\mcX}{\mathcal{X}}
\newcommand{\mcY}{\mathcal{Y}}
\newcommand{\mcZ}{\mathcal{Z}}

\title{On Hypothesis Transfer Learning of Functional Linear Models}

\author[1]{Haotian Lin}
\author[1,2]{Matthew Reimherr}

\affil[1]{Department of Statistics, The Pennsylvania State University}
\affil[2]{Amazon Science}

\date{}

\date{}

\begin{document}
\maketitle




\begin{abstract}
We study the \textit{transfer learning} (TL) for the \textit{functional linear regression} (FLR) under the Reproducing Kernel Hilbert Space (RKHS) framework, observing that the TL techniques in existing high-dimensional linear regression are not compatible with the truncation-based FLR methods, as functional data are intrinsically infinite-dimensional and generated by smooth underlying processes. We measure the similarity across tasks using RKHS distance, allowing the type of information being transferred to be tied to the properties of the imposed RKHS. Building on the hypothesis offset transfer learning paradigm, two algorithms are proposed: one conducts the transfer when positive sources are known, while the other leverages aggregation techniques to achieve robust transfer without prior information about the sources. We establish asymptotic lower bounds for this learning problem and show that the proposed algorithms enjoy a matching upper bound. These analyses provide statistical insights into factors that contribute to the dynamics of the transfer. We also extend the results to functional generalized linear models. The effectiveness of the proposed algorithms is demonstrated via extensive synthetic data as well as real-world data applications. 
\end{abstract}

\section{Introduction}\label{sec: intro}
Advances in technologies enable us to collect and process densely observed data over some temporal or spatial domains, which are coined functional data \citep{ramsay2005functional,kokoszka2017introduction}. While functional data analysis (FDA) has been proven useful in various fields like finance, genetics, etc., and has been researched widely in the statistical community, its effectiveness relies on having sufficient training samples drawn from the same distribution. However, this may not hold for functional data under some applications due to collection expenses or other constraints. Transfer learning (TL)\citep{torrey2010transfer} leverages additional information from some similar (source) tasks to enhance the learning procedure on the original (target) task and is an appealing mechanism when there is a lack of training samples. The goal of this paper is to develop TL algorithms for functional linear regression (FLR), one of the most prevalent models in the FDA. The FLR concerned in this paper is Scalar-on-Function regression, which takes the form:
\begin{equation*}
    Y = \alpha + \langle \beta, X \rangle_{L^{2}} + \epsilon
    = \alpha + \int_{\mathcal T} X(s) \beta(s) \ ds + \epsilon, 
\end{equation*}
where $Y$ is a scalar response, $X:\mathcal{T}\to \mathbb{R}$ and $\beta:\mathcal{T}\to \mathbb{R}$ are the square-integrable functional predictor and coefficient function respectively over a compact domain $\mathcal{T}\subset \mathbb{R}$, and $\epsilon$ is a random noise with zero mean.

A classical approach to estimating $\beta$ is to reduce the problem to classical multivariate linear regression by expanding the $X$ and $\beta$ under the same finite basis, like deterministic basis functions, e.g. Fourier basis, or the eigenbasis of the covariance function of $X$ \citep{cardot1999functional,yao2005functional,hall2006properties,hall2007methodology}, which we refer to truncation-based FLR methods in this paper. Conceptually, the offset transfer learning techniques developed in the existing multivariate/high-dimensional linear regression framework \citep{kuzborskij2013stability,kuzborskij2017fast,li2022transfer,bastani2021predicting} can be applied to truncation-based FLR methods to conduct transfer learning in FLR, though they lack a theoretical foundation in this context due to the truncation error inherent in a basis expansion of $\beta$. In particular, a key property distinguishing functional data from multivariate data is that they are inherently infinite-dimensional and generated through smooth underlying processes. Omitting this fact, using finite-dimensional approximations to $\beta$ and leveraging existing multivariate OTL techniques on the finite coefficients, loses the benefit that data are generated from smooth processes and are less interpretable for the transfer process; see detailed discussion in Section~\ref{sec: preliminaries section}. Observing these limitations, we develop the first TL algorithms for FLR with statistical convergence rate guarantees under the supervised learning setting.

We summarize our main contributions as follows.
\begin{enumerate}
    \item We propose using the Reproducing Kernel Hilbert Space (RKHS) distance between tasks' coefficients as a measure of task similarity. The transferred information is thus tied to the RKHS's properties, which makes the transfer more interpretable. One can tailor the employed RKHS to the task's nature, offering flexibility to embed diverse structural elements, like smoothness or periodicity, into TL processes.

    \item Built upon the offset transfer learning (OTL) paradigm, we propose TL-FLR, a variant of OTL for multiple positive transfer sources. We establish the minimax optimality for TL-FLR. 
    Intriguingly, the result reveals that the faster statistical rate of TL-FLR, compared to non-transfer learning, not only depends on source sample size and the magnitude of discrepancy across tasks, like most existing works, but also on the signal ratio between offset and source model.
    
    \item To deal with the practical scenario in which there is no available prior task similarity information, we propose Aggregation-based TL-FLR (ATL-FLR), utilizing sparse aggregation to mitigate negative transfer effects. We establish the upper bound for ATL-FLR and show that the aggregation cost decreases faster than the transfer learning risk, demonstrating an ability to identify optimal sources without too much extra cost compared to TL-FLR. We further extend this framework to Functional Generalized Linear Models (FGLM) with theoretical guarantees, broadening its applicability.

    \item In developing statistical guarantees, we uncovered unique requirements for making OTL theoretically feasible in the functional data context. These include the necessity for covariate functions across tasks to exhibit similar structural properties to ensure statistical convergence, and the coefficient functions of negative sources can be separated from positive ones within a finite-dimensional space to ensure optimal source identification.

\end{enumerate}

\paragraph{Literature review.}
Apart from truncation-based FLR approaches mentioned before, another line of research proposed that one can obtain a smooth estimator via smoothness regularization \citep{yuan2010reproducing,cai2012minimax}, and has been widely used in other functional models like the FGLM, functional Cox-model, etc. \citep{cheng2015joint,qu2016optimal,reimherr2018optimal,sun2018optimal}.

Turning to the TL regime in supervised learning, the hypothesis transfer learning (HTL) framework has become popular \citep{li2007bayesian,orabona2009model,kuzborskij2013stability,perrot2015theoretical,du2017hypothesis}. Offset transfer learning (OTL) (a.k.a. biases regularization transfer learning) has been widely analyzed and applied as one of the most popular HTL paradigms. It assumes the target's function/parameter is a summation of the source's and the offset's function/parameter. A series of works have derived theoretical analyses under different settings. For example, in \citet{kuzborskij2013stability,kuzborskij2017fast}, the authors provide the first theoretical study of OTL in the context of linear regression with stability analysis and generalization bounds. Later, in \citet{wang2015generalization,wang2016nonparametric}, the authors derive similar theoretical guarantees for non-parametric regression via Kernel Ridge Regression. A unified framework that generalizes many previous works is proposed in \citet{du2017hypothesis}, and the authors also present an excess risk analysis for their framework. Apart from the regression setting, generalization bounds for classification with surrogate losses have been studied in \citet{aghbalou2023hypothesis}. Other results that study HTL outside OTL can be found in \citet{li2007bayesian,cheng2015joint}. Besides, OTL can also be viewed as a case of representation learning \citep{du2020few,tripuraneni2020theory,xu2021representation} by viewing the estimated source model as a representation for target tasks. Finally, the bias regularization technique on which OTL relies is also widely used in other learning settings, e.g., Meta, Multi-task, and unsupervised learning, see \citet{denevi2018learning,denevi2019learning,balcan2019provable,tian2022unsupervised,tian2023learning}.

The statistics community has recently adopted OTL for various high-dimensional models with statistical risk guarantees. For example, \citet{bastani2021predicting} proposed using OTL for high-dimensional (generalized) linear regression, but only includes one positive transfer source. Later, \citet{li2022transfer} extended this idea to the multiple sources scenario and leveraged aggregation to alleviate negative transfer effects. In \citet{tian2022transfer}, the learning procedure gets extended to the high-dimensional generalized linear model, and the authors also proposed a positive sources detection algorithm via a validation approach. In these works, the similarity among tasks is quantified via $\ell^{1}$-norm, which captures the sparsity structure in high-dimensional parameters.
There is no TL for FDA that we are aware of, but the closest work would be in the area of domain adaptation. \citet{zhu2021functional} studied the domain adaptation problem between two separable Hilbert spaces by proposing algorithms to estimate the optimal transport mapping between two spaces.

\paragraph{Notation.}
For two sequences $\{a_{k}\}_{k\geq1}$ and $\{b_{k}\}_{k\geq1}$, we denotes $a_n \asymp b_n$ and $a_n \lesssim b_n$ if $|a_n/b_n| \rightarrow c$ and $|a_n/b_n|\leq c$ for some universal constant $c$ when $n\rightarrow \infty$. For two random variable sequence $\{A_{k}\}_{k\geq1}$ and $\{B_{k}\}_{k\geq1}$, if for any $\delta>0$, there exists $M_{\delta}>0$ and $N_{\delta}>0$ such that $\mbP(A_{k} < M_{\delta} B_{k})\geq 1 - \delta$, $\forall k\geq N_{\delta}$, we say $A_{k} = O_{\mbP}(B_{k})$. For a set $A$, let $|A|$ denote its cardinality, $A^{c}$ denote its complement. For an integer $n$, denote $[n]:= \{1,\cdots,n\}$.

We denote the covariance function of $X$ as $C(s,t) = \E [X(s) -\E X(s) ] [X(t) -\E X(t)]$ for $s,t\in \mcT$. For a real, symmetric, square-integrable, and nonnegative kernel, $K:\mcT\times\mcT\rightarrow \mbR$, we denote its associated RKHS on $\mcT$ as $\mcH_{K}$ and corresponding norm as $\|\cdot\|_{K}$.  We also denote its integral operator as $L_{K}(f) = \int_{\mcT} K(\cdot,t) f(t) dt \quad$ for $f\in L^2$. For two kernels, $K_{1}$ and $K_{2}$, their composition is $(K_{1} K_{2}) (s,t) = \int_{\mcT} K_{1} (s,u) K_{2} (u,t) du$. For a given kernel $K$ and covariance kernel $C$, define bivariate function $\Gamma$ and its integral operator as $\Gamma := K^{1/2}CK^{1/2}$ and $L_{\Gamma}(f)= L_{K^{\frac{1}{2}}}(L_{C}( L_{K^{\frac{1}{2}}}(f) ) )$.

\section{Preliminaries and Backgrounds}\label{sec: preliminaries section}
\paragraph{Problem Set-up.}
We now formally set the stage for the transfer learning problem in the context of FLR. Consider the following series of FLRs,
\begin{equation}\label{eqn: model setup}
    Y_i^{(t)}  = \alpha^{(t)} + \left\langle X_i^{(t)},\beta^{(t)}\right\rangle_{L^{2}}  + \epsilon_{i}^{(t)}  
\end{equation}
for $i \in [n_{t}]$,  $t=0\cup[T]$, where $t=0$ denotes the target model and $t\in[T]$ denotes source models. Denote the sample space $\mcZ$ as the Cartesian product of the covariate space $\mcX$ and response space $\mcY$. For each $t\in 0\cup[T]$, we denote $\mcD^{(t)} = \{ (X_{i}^{(t)}, Y_{i}^{(t)}) \}_{i=1}^{n_{t}} = \{ Z_{i}^{(t)} \}_{i=1}^{n_{t}}$. Throughout the paper we assume $\epsilon_{i}^{(t)}$ are i.i.d. across both $i$ and $t$ with zero mean and finite variance $\sigma^2$.

As estimating $\beta^{(0)}$ is our primary interest, we assume for simplicity that $\alpha^{(t)}=0$ for all $t$. We assume $n_{0} \ll \sum_{t=1}^{T}n_{t}$, a condition commonly validated in most TL literature and numerous practical applications. While our framework is designed primarily for the posterior drift setting, i.e., the marginal distributions of $X^{(t)}$ remain the same, but $\beta^{(t)}$ vary, the excess risk bounds we establish are based on a comparatively more relaxed condition, see Section \ref{sec: theoretical section for FLR}.

In the absence of source data, estimating $\beta^{(0)}$ is termed as target-only learning, and one can obtain a smooth estimator of $\beta$ through the regularized empirical risk minimization (ERM) \citep{yuan2010reproducing,cai2012minimax}, i.e. 
\begin{equation*}
    \hat{\beta} = \underset{  \beta \in \mcH_{K} }{\operatorname{argmin}}\left\{ \frac{1}{n_{0}}\sum_{i=1}^{n_{0}}\ell(\beta,Z_{i}^{(0)}) + \lambda\|\beta\|_{K}^{2}\right\},
\end{equation*}
where $K$ is an employed kernel and $\ell:\mcH_{K}\times\mcZ \rightarrow \mbR^{+}$ is the loss function. This approach has been proven to achieve the optimal rate in terms of excess risk, and we refer to it as \textit{Optimal Functional Linear Regression} (OFLR) in this paper, which serves as a non-transfer learning baseline.

\paragraph{Similarity Measure.}
We first state the limitations of using $\ell^{1}/\ell^{2}$-norm as a similarity measure in the truncation-based FLR method, which converts the problem into a classic multivariate one. For a given series of basis functions $\{\phi_{j}\}_{j\geq 1}$ and truncated number $M$, one can model the $t$-th FLR as
\begin{equation}\label{eqn: truncated FLR model}
    Y_{i}^{(t)} \approx  \sum_{j=1}^{M} X_{ij}^{(t)}\beta_{j}^{(t)} + \epsilon_{i}^{(t)}
\end{equation}
where $X_{ij}^{(t)} = \langle X_{i}^{(t)}, \phi_{j} \rangle_{L^{2}}$ and $\beta_{j}^{(t)}= \langle \beta^{(t)}, \phi_{j} \rangle_{L^{2}}$. Denote $\beta^{(t)}_{\text{trunc}} \in \mathbb{R}^{M}$ as the coefficient vector in (\ref{eqn: truncated FLR model}), one can then measure the similarity between the target and the $t$-th FLR model by the $\ell^{1}$ or $\ell^{2}$ norm of $\beta^{(t)}_{\text{trunc}} - \beta^{(0)}_{\text{trunc}}$ like the previous works did for multivariate linear regression. However, from the functional data analysis literature, since the functional data are generated from some structural underlying process, it is well known that one has to have the same kind of structure in the estimator, like smoothness, for theoretical reliability. For example, when the coefficient functions are smooth, the above approach cannot measure the similarity since $\{\beta^{(t)}_{\text{trunc}}\}_{t=0}^{T}$ are not necessarily sparse or might require regularization via an $\ell^{2}$-norm, but the employed basis functions might not reflect the desired smoothness. Besides, the basis functions and $M$ should be consistent across tasks, which reduces the flexibility of the learning procedure.

To explore the similarity tied to the structure of coefficient functions, one should quantify the similarity between tasks within certain functional spaces that possess the same structures. These structural properties, e.g., continuity/smoothness/periodicity, can be naturally encapsulated via kernels and their corresponding RKHS. Consequently, quantifying the similarity within a certain RKHS provides interpretability since the type of information transferred is tied to the structural properties of the used RKHS. We also note that this method is broadly applicable since the reproducing kernel can be tailored to the application problem accordingly. For example, one can transfer the information about continuity or smoothness by picking $K$ to be a Sobolev kernel, and about periodicity by picking periodic kernels like $K(x_1,x_2) = exp\left( - 2/l^{2} \sin\left(\pi |x_1 - x_2|/p \right) \right)$ where $l$ is the lengthscale and $p$ is the period.

Given the reasoning above, for $t=0\cup[T]$, we assume $\beta^{(t)}\in \mcH_{K}$, and define the $t$-th contrast function $\delta^{(t)} := \beta^{(0)}- \beta^{(t)}$. Given a constant $h\geq 0$, we say the $t$-th source model is ``h-transferable'' if $\| \delta^{(t)} \|_{K} \leq h$. The magnitude of $h$ characterizes the similarity between the target model and source models. We also define $\mcS_{h} = \{ t\in [T] : \| \delta^{(t)} \|_{K} \leq h \}$ as a subset of $[T]$, which consists of the indexes of all ``h-transferable'' source models. It is worth mentioning that the quantity $h$ is introduced for theoretical purposes to establish optimality, which is prevalent in recent studies such as \citet{bastani2021predicting,li2022transfer,tian2022transfer,he2024transfusion}. However, for the implementation of the algorithm, it is not necessary to know the actual value of $h$. We abbreviate $\mcS_{h}$ as $\mcS$ to generally represent the h-transferable sources index.

\paragraph{Learning Framework.}
This paper leverages the widely used OTL paradigm; see reviews in Section \ref{sec: intro}. Formally, in the FLR and single source $\beta^{(1)}$ context, the OTL obtains the target function via  $\hat{\beta}^{(0)} = \hat{\beta}^{(1)} + \hat{\delta}$ where $\hat{\beta}^{(1)}$ is the estimator trained on source dataset and $\hat{\delta}$ is obtained from target dataset via following minimization problem:
\begin{equation*}
    \hat{\delta} = \underset{ \delta \in \mathcal{H}_{K}}{\operatorname{argmin }} \frac{1}{n_{0}}\sum_{i=1}^{n_{0}}\ell(\delta + \hat{\beta}^{(0)},Z_{i}^{(0)}) + \lambda\|\delta\|_{K}^{2},
\end{equation*}
where the loss function can be square loss \citep{orabona2009model,kuzborskij2013stability} or surrogate losses \citep{aghbalou2023hypothesis}. The main idea is that the estimator $\hat{\beta}^{(1)}$ can be learned well given sufficiently large source samples, and the simple offset estimator $\hat{\delta}^{(0)}$ can be learned with much fewer target samples. 


\section{Methodology}\label{methodology section}

\subsection{Offset Transfer Learning with Multiple Sources}
For multiple sources, the idea of data fusion inspires us to obtain a centered source estimator $\beta_{\mcS}$ via all source datasets in place of $\beta^{(1)}$. Therefore, we can generalize single-source OTL to the multiple-source scenario as follows.

\begin{algorithm}[ht]
\caption{TL-FLR}\label{algo: A transfer algorithm}
\begin{algorithmic}
   \State {\bfseries Input:} Target/Source datasets
    $\{\mcD^{(t)} \}_{t=0}^{T}$; index set of source datasets $\mcS$; Loss function $\ell$ as square loss. 

    \State {\bfseries Output:} $\hat{\beta}_{S} + \hat{\delta}$. \\
    
    \State \textbf{Transfer Step}: Obtain $\hat{\beta}_{\mcS}$ via 
    \begin{equation}\label{eqn: transfer step}
        \hat{\beta}_{\mcS} = \underset{ \beta \in \mathcal{H}_{K}}{\operatorname{argmin}} \sum_{t\in \mcS} \frac{1}{n_{t}}\sum_{i=1}^{n_{t}} \ell(\beta, Z_{i}^{(t)})  + \lambda_{1}\|\beta\|_{K}^{2}.
    \end{equation}

    \State \textbf{Calibration Step}:
    Obtain offset $\hat{\delta}$ via 
    \begin{equation}\label{eqn: calibration step}
        \hat{\delta} = \underset{ \delta \in \mathcal{H}_{K}}{\operatorname{argmin }} \frac{1}{n_{0}}\sum_{i=1}^{n_{0}}\ell(\delta + \hat{\beta}_{\mcS}, Z_{i}^{(0)}) + \lambda_{2}\|\delta\|_{K}^{2}.
    \end{equation}

\end{algorithmic}
\end{algorithm}

Since the probabilistic limit of $\hat{\beta}_{\mcS}$ is not consistent with $\beta^{(0)}$, calibration of $\hat{\beta}_{\mcS}$ is performed in (\ref{eqn: calibration step}). The regularization term in (\ref{eqn: calibration step}) is consistent with our similarity measure, i.e. it restricts $\hat{\beta}^{(0)}$ to lie in a $\mcH_{K}$ ball centered at $\hat{\beta}_{\mcS}$. Therefore, this term pushes the $\hat{\beta}^{(0)}$ close to $\hat{\beta}_{\mcS}$ while the mean square error loss over the target dataset allows calibration for the bias. Intuitively, if $\hat{\beta}_{\mcS}$ is close to $\beta^{(0)}$, then TL-FLR can boost the learning on the target model.

\subsection{Adaptive Transfer via Sparse Aggregation}

Assuming the index set $\mcS$ is known in Algorithm~\ref{algo: A transfer algorithm} can be unrealistic in practice without prior information or investigation. Moreover, as some source tasks might have little or even a negative contribution to the target one, it could be practically harmful to directly apply Algorithm~\ref{algo: A transfer algorithm} by assuming all sources belong to $\mcS$. Inspired by the idea of aggregating multiple estimators in \citet{li2022transfer}, we develop ATL-FLR, which can be directly applied without knowing $\mcS$ while being robust to negative transfer sources.

The general idea of ATL-FLR is that one can first construct a collection of candidates for $\mcS$, named $\{\hat{\mcS_1}, \hat{\mcS_2}, \cdots, \hat{\mcS_J}\}$, such that there exists at least one $\hat{\mcS}_{j}$ satisfying $\hat{\mcS}_{j} = \mcS$ with high probability and then obtain their corresponding estimators $\mcF = \{ \hat{\beta}(\hat{\mcS_1}),\cdots,\hat{\beta}(\hat{\mcS_J})\}$ via TL-FLR. Then, one aggregates the candidate estimators in $\mcF$ such that the aggregated estimator $\hat{\beta}_{a}$ satisfies the following oracle inequality in high probability,
\begin{equation}\label{aggregation oracle inequality}
    R(\hat{\beta}_{a})\leq \min_{\beta \in \mcF} R(\beta) + r(\mcF,n),
\end{equation}
where $\quad R(f) = \E_{(X,Y)} [\ell(Y,f(X)) | \{\mcD^{(t)}:t\in0\cup[T]\} ]$, and $r(\mcF,n)$ is the aggregation cost. Thus, the $\hat{\beta}_{a}$ can achieve similar performance as TL-FLR up to some aggregation cost. The proposed aggregation-based TL-FLR is as follows:

\begin{algorithm}[ht]
    \caption{Aggregation-based TL-FLR (ATL-FLR)}\label{algo: aggregate transfer algorithm}
    \begin{algorithmic}
    \State {\bfseries Input:} Target/Source datasets
    $\{\mcD^{(t))} \}_{t=0}^{T}$; index set of source datasets $\mcS$; Loss function $\ell$ as square loss; A given integer $M$. \\
    
    \State \textbf{Step 1:} Split the target dataset $\mcD^{(0)}$ into $\mcD_{\mcI}^{(0)}$ and $\mcD_{\mcI^{c}}^{(0)}$ with $\mcI$ be a random subset of $[n_{0}]$ such that $|\mcI| = \lfloor \frac{n_0}{2} \rfloor$.

    \State \textbf{Step 2:} Build candidate sets of $\mcS$, $\{ \hat{\mcS}_0, \hat{\mcS}_1, \cdots, \hat{\mcS}_{T} \}$ as:
    
    \begin{enumerate}
    \item Obtain $\hat{\beta}_{0}$ by OFLR using $\mcD_{\mcI}^{(0)}$ and let $\hat{\mcS}_0 = \emptyset$.

    \item For each $t\in [T]$, obtain $\hat{\beta}_{t}$ by OFLR using $\mcD^{(t)}$ and obtain truncated RKHS norm $\hat{\Delta}_{t} = \|\hat{\beta}_0 -\hat{\beta}_{t} \|_{K^{M}} := \sum_{j=1}^{M} \langle \hat{\beta}_0 -\hat{\beta}_{t}, v_{j} \rangle^{2}/\tau_{j}$.
    
    \item Set $\hat{\mcS}_{t} = \left\{ k: \hat{\Delta}_{k} \text{ is among the first } t \text{ smallest.}  \right\}$
    
    \end{enumerate}
    
    \State \textbf{Step 3:} For $t\in[T]$, fit TL-FLR by setting $\mcS = \hat{\mcS}_{t}$ with dataset $\mcD_{\mcI}^{(0)}$. Let $\mcF = \{\hat{\beta}(\hat{\mcS}_0), \hat{\beta}(\hat{\mcS}_1),\cdots \hat{\beta}(\hat{\mcS}_{T})\}$.
    
    \State \textbf{Step 4:} Implement the sparse aggregation procedure in Algoritm~\ref{algo: sparse aggregation} with $\mcF$ as the dictionary and training dataset as $\mcD_{\mcI^c}^{(0)}$. Obtain the sparse aggregated estimator $\hat{\beta}_{a}$.
  \end{algorithmic}
\end{algorithm}

\begin{remark}
While exploring the estimated similarity across sources to the target in Step 2, we use a truncated RKHS norm, which is the distance between $\hat{\beta}_{0}$ and $\hat{\beta}_{t}$ after projecting them onto the space spanned by the first $M$ eigenfunctions of $K$. Here, $\{\tau_{j}\}_{j\geq 1}$ and $\{v_j\}_{j\geq 1}$ are the eigenvalues and eigenfunctions of $K$. Such a truncated norm guarantees the identifiability of $\mcS$; see Section~\ref{sec: er of ATL-FLR} for details.
\end{remark}

Step 2 ensures the target-only baseline $\hat{\beta}_{0}$ lies in $\mcF$ while the construction of $\hat{\mcS}_{t}$ ensures thorough exploration of $\mcS$. If $\mcS$ can be identified by one of the $\hat{\mcS}_{t}$, then inequality (\ref{aggregation oracle inequality}) indicates that even without knowing $\mcS$, the $\hat{\beta}_{a}$ can mimic the performance of the TL-FLR estimator, while not being worse than the target-only $\hat{\beta}_{0}$, up to an aggregation cost.

The sparse aggregation is adopted from \citet{gaiffas2011hyper}, see Appendix~\ref{apd: sparse aggregation}. Although we note that other aggregation methods like aggregate with cumulated exponential weights (ACEW) \citep{juditsky2008learning,audibert2009fast}, aggregate with exponential weights (AEW) \citep{leung2006information,dalalyan2007aggregation}, and Q-aggregation \citep{dai2012deviation} can replace sparse aggregation in Step 4, sparse aggregation is often preferred due to its computational efficiency and ability to eliminate negative transfer effects. Specifically, the final aggregated estimator is usually represented as a convex combination of elements in $\mcF$, i.e., $\hat{\beta}_{a} = \sum_{j=1}^{J}c_{j} \hat{\beta}(\hat{\mcS}_{j})$. The sparse aggregation sets most of the $c_{j}$ to zero, which effectively excludes the negative transfer sources. On the other hand, none of the ACEW, AEW, and Q-aggregation will set the $c_{j}$ to $0$ most of the time, meaning that negative transfer sources can still affect $\hat{\beta}_{a}$. Although one can manually tune temperature parameters in these approaches to shrink the $c_j$ close to zero, they are less computationally efficient given the fact that sparse aggregation does not require such a tuning process. In Section~\ref{sec: experiments}, we verify that sparse aggregation outperforms other aggregation methods under various settings.

\section{Theoretical Analysis}\label{sec: theoretical section for FLR}
In this section, we study the theoretical properties of the prediction accuracy of the proposed algorithms. We evaluate the proposed algorithms via excess risk, i.e., 
\begin{equation*}
    \mcE(\hat{\beta}^{(0)}) := \E_{Z^{(0)}} \left[  \ell(\hat{\beta}^{(0)},Z^{(0)}) - \ell(\beta^{(0)},Z^{(0)})\right]
\end{equation*}
where the expectation is taken over an independent test data point $Z^{(0)}$ from the target distribution. To study the excess risk of TL-FLR and ATL-FLR, we denote $\beta_{S}$ the population version of $\hat{\beta}_{\mcS}$ which also lies in $\mcH_{K}$ and define the parameter space as 
\begin{equation*}
    \Theta(h,R) = \left\{  \{\beta^{(t)}\}_{t\in \{0\}\cup\mcS}  : \|\beta_{\mcS}\|_{K}\leq R,  \|\delta^{(t)}\|_{K} \leq h \right\}.
\end{equation*}
To establish the theoretical analysis of the proposed algorithms, we first state some assumptions. 
For $t \in 0\cup[T]$, denote $\{s_{j}^{(t)}\}_{j\geq1}$ and $\{\phi_{j}^{(t)}\}_{j\geq 1}$ as the eigenvalues and eigenfunctions of $\Gamma^{(t)}:=K^{\frac{1}{2}}C^{(t)}K^{\frac{1}{2}}$ respectively.

\begin{assumption}[Eigenvalue Decay Rate (EDR)]\label{assump: EDR}
    Suppose that the eigenvalue decay rate (EDR) of $L_{\Gamma^{(0)}}$ is $2r$, i.e. 
    \begin{equation*}
        s_{j}^{(0)}\asymp j^{-2r}, \quad \forall j \geq 1.
    \end{equation*}
\end{assumption}
The polynomial EDR assumption is standard in FLR literature \citep{cai2012minimax,reimherr2018optimal}. RKHSs that satisfy this assumption, like Sobolev spaces, are natural choices when considering smoothness as the structural properties in the TL processes. For target-only learning by minimizing regularized ERM over target data $\mcD^{(0)}$, the minimax convergence rate of the excess risk under this assumption is $n_{0}^{-2r/(2r+1)}$ \citep{cai2012minimax}.

\begin{assumption}\label{assump: eigen assumption}
We assume either one of the following conditions holds.
\begin{enumerate}
    \item \label{commute assumption} $L_{\Gamma^{(t)}}$ commutes with $L_{\Gamma^{(0)}}$, $\forall t\in\mcS$, i.e. $L_{\Gamma^{(0)}} L_{\Gamma^{(t)}} = L_{\Gamma^{(t)}} L_{\Gamma^{(0)}}$, and 
    \begin{equation*}
        a_j^{(t)} := \langle L_{\Gamma^{(t)}}(\phi_{j}^{(0)}),\phi_{j}^{(0)} \rangle \asymp s_{j}^{(0)} \quad \forall j\geq 1.
    \end{equation*}

    \item \label{HS assumption} The following linear operator is Hilbert–Schmidt.
    \begin{equation*}
    \mathbf{I} - (L_{\Gamma^{(0)}})^{-1/2} L_{\Gamma^{(t)}} (L_{\Gamma^{(0)}})^{-1/2}, \quad \forall t\in\mcS.
    \end{equation*}
\end{enumerate}
\end{assumption}

We note that under the posterior drift setting, both conditions in Assumption~\ref{assump: eigen assumption} are satisfied automatically, and thus, our theoretical results are built on assumptions that are more relaxed than posterior drift. Although neither condition implies the other, both conditions primarily focus on how the smoothness of the source kernel $\Gamma^{(t)}$ relates to that of the target kernel $\Gamma^{(0)}$. Specifically, Condition~\ref{commute assumption} implies $L_{\Gamma^{(0)}}$ and $L_{\Gamma^{(t)}}$ not only share the same eigenspace but also have similar magnitudes of the projection onto the $j$-th dimension of the eigenspace, which commonly appears in FDA literature \citep{yuan2010reproducing,balasubramanian2022unified}. Condition~\ref{HS assumption} implies the probability measures of $X^{(0)}$ and $X^{(t)}$ are equivalent, see \citet{baker1973equivalence}. Collectively, both conditions indicate that the feasibility of OTL for FLR relies on the fact that the regularity of the source operator $L_{\Gamma^{(t)}}$ should behave similarly to that of the targets. Either a too ``smooth'' or a too ``rough'' source can degrade the optimality. 

Besides, these conditions help to prevent the excess risk of $\hat{\beta}_{S}$ over the target domain from diverging. Specifically, a technical reason behind such a complex form (an asymptotic behavior between the eigenvalue of $\Gamma^{(0)}$, $s_{j}^{(0)}$, and the projection of $\Gamma^{(0)}$ onto the eigenspace of $\Gamma^{(0)}$, $a_{j}^{(t)}$) is due to the natural difficulty of the functional data problem as one needs to handle quantities in infinite dimensions. Specifically, we are dealing with vectors in $\mbR^{\infty}$ when bounding the approximation error in the excess risk of the learned source model on the target domain $E \langle X^{(0)}, \hat{\beta}_{S} - \beta_{S} \rangle_{L^{2}}^2$. The constant term for the approximation error will be affected by the maximum ratio of $s_{j}^{(0)}$ and $a_{j}^{(t)}$ overall dimensions, i.e., $\forall j\geq 1$, and Assumption~\ref{assump: eigen assumption} indeed gives a fine control about this maximum ratio to avoid an exploded approximation error. In the multivariate case, this is not a problem as these quantities are more easily manipulated. It's crucial to highlight that this unique challenge arises when dealing with functional objects due to their infinite-dimensional nature.

\subsection{Minimax Excess Risk of TL-FLR}\label{sec: er of TL-FLR}
We first provide the upper bound of excess risk on TL-FLR.
\begin{theorem}[Upper Bound]\label{thm: upper bound for TL-FLR}
Suppose Assumption~\ref{assump: eigen assumption} and~\ref{assump: EDR} hold. If $n_0/n_{\mcS}\nrightarrow0$, let $\xi(h,R) = \frac{h^2}{R^2}$, then for the output $\hat{\beta}$ of Algorithm~\ref{algo: A transfer algorithm},
\begin{equation}\label{eqn:upper}
\sup_{ \Theta(h,R) } \mcE\left(\hat{\beta}\right) = O_{\mbP}
\left( n_{\mcS} ^{-\frac{2r}{2r+1}} + n_0^{-\frac{2r}{2r+1}}\xi(h,R)  \right),
\end{equation}
 if $\lambda_1 \asymp n_{\mcS} ^{-\frac{2r}{2r+1}}$ and $\lambda_2 \asymp n_0^{-\frac{2r}{2r+1}}$ where $\lambda_1$ and $\lambda_2$ are regularization parameters in Algorithm~\ref{algo: A transfer algorithm}.
\end{theorem}
Theorem~\ref{thm: upper bound for TL-FLR} provides the excess risk upper bound of $\hat{\beta}$, which bounds the excess risk by two terms. The first term comes from the transfer step and depends on the sample size of sources in $\mcS$, while the second term is due to only using the target dataset to learn the offset. In the trivial case when $\mcS=\emptyset$, the upper bound becomes $O_{\mbP}( n_{0}^{-2r/(2r+1)} )$, which coincides with the upper bound of target-only baseline OFLR \citep{cai2012minimax}. When $\mcS \neq \emptyset$, compared with the excess risk of the target-only baseline, we can see the sample size $n_{\mcS}$ in source models and the factor $\xi(h,R)$ are jointly affecting the transfer learning. The factor $\xi(h,R)$ represents the relative task signal strength between the source and target tasks. Geometrically, one can interpret $\xi(h,R)$ as the factor that roughly controls the angle between the source and target models within the RKHS.

\begin{figure}[ht]
\begin{center}
\centerline{\includegraphics[page = 1,width=0.3\columnwidth]{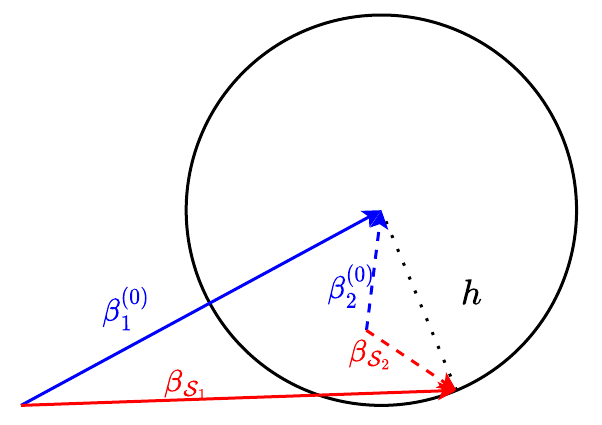}}
\caption{Geometric illustration for how $\xi(h,R)$ will affect the transfer dynamic. The circle represents an RKHS ball centered around $\beta^{(0)}$ with radius $h$. With the same $h$, larger signal strength of $\beta_{\mcS}$, i.e. $\|\beta_{\mcS_{1}}\|_{K}$ leads to smaller $\xi(h,R)$, while smaller signal strength of $\beta_{\mcS}$, i.e. $\|\beta_{\mcS_{2}}\|_{K}$ leads to larger $\xi(h,R)$.}
\label{fig: angle fig}
\end{center}
\vskip -0.2in
\end{figure}

Figure~\ref{fig: angle fig} shows how $n_{\mcS}$ and $\xi(h,R)$ impact the learning rate. When the $\beta_{S}$ and $\beta^{(0)}$ are more concordant ($\beta_{\mcS_{1}}$ and $\beta_{1}^{(0)}$), the angle between them is small and thus so the $\xi(h,R)$, making the second term in the upper bound negligible in the excess risk, and thus the risk converges faster compared to baseline $n_{0}^{-2r/(2r+1)}$ given sufficiently large $n_{\mcS}$. If $\beta_{S}$ and $\beta^{(0)}$ are less concordant ($\beta_{\mcS_{2}}$ and $\beta_{2}^{(0)}$), leveraging $\beta_{\mcS}$ will be less effective since a large $\xi(h,R)$ will make the second term the dominate term. 

It is worth noting that most of the existing literature fails to identify how $\xi(h,R)$ affects the effectiveness of OTL. For example, in \citet{wang2016nonparametric,du2017hypothesis}, this factor does not appear in the upper bound, and they claim $n_{S} \gg n_{0}$ provides successful transfer from source to target. In high-dimensional linear regression \citep{li2022transfer,tian2022transfer}, the authors only identify $\xi(h,R)$ is proportional to $h$ and claims a small $h$ can provide a faster convergence rate excess risk. However, our analysis (Figure~\ref{fig: angle fig}) shows that even with the same $h$, the similarity of the two tasks can be different given different signal strengths of $\beta_{\mcS}$, which will also affect the effectiveness of OTL. To this end, this reveals that one cannot obtain a faster excess risk in OTL by simply including more source datasets (larger $n_{\mcS}$), but should also carefully select or construct the $\mcS$,
i.e., the more source data are available, the more strict one should be with what sources one uses to build $\hat{\beta}_{\mcS}$ in the OTL framework.

\begin{theorem}[Lower Bound]\label{thm: lower bound for TL-FLR}
Under the same condition of Theorem~\ref{thm: upper bound for TL-FLR}, for any possible estimator $\tilde{\beta}$ based on $\{ \mcD^{(t)} : t\in \left\{0 \right\} \cup \mcS \}$, the excess risk of $\tilde{\beta}$ satisfies
\begin{equation}\label{lowerbound formula}
    \lim_{a\rightarrow 0} \lim_{n \rightarrow \infty}   \inf_{\tilde{\beta}} \sup_{\Theta(h,R) } P\left\{ \mcE\left( \tilde{\beta} \right) \geq  a \left( n_{\mcS} ^{-\frac{2r}{2r+1}} + n_{0}^{-\frac{2r}{2r+1}} \xi(h,R) \right) \right\} = 1.
\end{equation}
\end{theorem}
Combining Theorems~\ref{thm: upper bound for TL-FLR} and~\ref{thm: lower bound for TL-FLR} implies that the estimator from TL-FLR is rate-optimal in excess risk. The proof of the lower bound is based on considering the lower bound of two cases: (1) the ideal case where $\beta^{(t)} = \beta^{(0)}$ for all $t\in \mcS$ and (2) the worst case where $\beta^{(t)}\equiv0$, meaning no knowledge should be transferred at all.

\subsection{Excess Risk of ATL-FLR}\label{sec: er of ATL-FLR}
In this subsection, we study the excess risk for ATL-FLR. As we discussed before, making ATL-FLR achieve similar performance to TL-FLR relies on the fact that there exists a $\hat{\mcS}_{t}$ such that it equals to the true $\mcS$ (so $\hat{\beta}(\hat{\mcS}_{t}) = \hat{\beta}(\mcS)$) with high probability. Therefore, to ensure the $\mcF$ constructed in Step 2 of Algorithm~\ref{algo: aggregate transfer algorithm} satisfies such a property, we impose the following assumption to guarantee the identifiability of $\mcS$ and thus ensure the existence of such $\hat{\mcS}_{t}$.
\begin{assumption}[Identifiability of $\mcS$]\label{assump: identifiability}
Suppose for any $h$, there is an integer $M$ such that 
\begin{equation*}
    \min_{t\in \mcS^{c}} \| \beta_0 - \beta_{t}\|_{K^{M}} > h,
\end{equation*}
where $\|\cdot\|_{K^{M}}$ is the truncated version of $\|\cdot\|_{K}$ defined in Algorithm~\ref{algo: aggregate transfer algorithm}.
\end{assumption}
Assumption~\ref{assump: identifiability} ensures that $\forall t\in \mcS^{c}$, there exists a finite-dimensional subspace of $\mcH_{K}$, such that the norm of the projection of the contrast function, $\delta^{(t)}$, on this subspace is already greater than $h$. This assumption indeed eliminates the existence of $\beta^{(t)}$, for $t\in \mcS^{c}$, that live on the boundary of the RKHS-ball centered at $\beta^{(0)}$ with radius $h$ in $\mcH_{K}$. Under Assumption~\ref{assump: identifiability}, we now show the $\mcF$ constructed in Algorithm~\ref{algo: aggregate transfer algorithm} guarantees the existence of $\hat{\mcS}_{t}$.
\begin{theorem}\label{thm: aggregation consistency}
Suppose Assumption~\ref{assump: identifiability} holds, then
\begin{equation*}
    \max_{t\in \mcS} \Delta_{t} < \min_{t\in \mcS^{c}} \Delta_{t} \quad \text{and} \quad \mathbb{P}\left( \max_{t\in \mcS} \hat{\Delta}_{t} < \min_{t\in \mcS^{c}} \hat{\Delta}_{t} \right) \rightarrow 1,
\end{equation*}
and hence there exists a $t$ s.t. $\hat{\mcS}_{t}\in\mcF$ and
\begin{equation*}
    \mathbb{P}\left( \hat{\mcS}_{t} = \mcS \right) \rightarrow 1.
\end{equation*}
\end{theorem}
\begin{remark}
Assumption~\ref{assump: identifiability} ensures a sufficient gap between those $\Delta_{t}$ that belong to $\mcS$ and those that don't, which ensures their estimated counterpart will also possess this gap with high probability, making one of the $\hat{\mcS}$ consistent with $\mcS$.
\end{remark}
With Proposition~\ref{prop: oracle inequality of sparse aggregation}, which states the cost of sparse aggregation in Appendix \ref{sub appendix: prop}, and the excess risk of TL-FLR in Theorem~\ref{thm: aggregation consistency}, we can establish the excess risk for ATL-FLR.

\begin{theorem}[Upper Bound of ATL-FLR]\label{thm: upper bound ATL-FLR}
Let $\hat{\beta}_{a}$ be the output of Algorithm~\ref{algo: aggregate transfer algorithm}, then under the same settings of Theorem~\ref{thm: upper bound for TL-FLR} and Assumption~\ref{assump: identifiability},
\begin{equation*}
 \sup_{ \Theta(h,R) }  \mcE\left(\hat{\beta}_{a}\right)  =   
 O_{\mbP} \bigg(  \underbrace{n_{\mcS}^{-\frac{2r}{2r+1}} + n_0^{-\frac{2r} {2r+1}}\xi(h,R)}_{\text{transfer learning risk}} + \underbrace{\frac{log(T)log(n_0)}{n_0}}_{\text{aggregation cost}} \bigg).
\end{equation*}
\end{theorem}
One interesting note is that the transfer learning risk is the classical nonparametric rate, while the aggregation cost is parametric (or nearly parametric). Therefore, the aggregation cost usually decays substantially faster than the transfer learning risk. However, for transfer learning in high-dimensional linear regression, such a faster-decayed aggregation cost is diminished since the transfer learning risk is also parametric \citep{li2022transfer}.

\section{Extension to Functional Generalized Linear Models}\label{sec: extension to FGLM}
In this section, we show that our approaches in the FLR model can be naturally extended to the functional generalized linear model (FGLM) settings, which includes wider application scenarios like classification. To start, consider the following series of FGLM models similar to the FLR setting (\ref{eqn: model setup}), 
\begin{equation*}\label{eqn: FGLM form}
    \mbP  (Y_{i}^{(t)}|X_{i}^{(t)}) = \rho(Y_{i}^{(t)}) \exp\left\{ \frac{Y_{i}^{(t)} \eta(  \theta_{i}^{(t)} ) - \psi( \theta_{i}^{(t)}) }{d(\tau)}  \right\},
\end{equation*}
where $i \in [n_{t}]$ and $t\in [T]$, $\theta_{i}^{(t)} = \langle X_{i}^{(t)}, \beta^{(t)} \rangle_{L^{2}}$ is the canonical parameter. The functions $\rho, \eta, \psi, d$ are known, and $\tau$ is either known or out-of-interest parameter that is independent of $X^{(t)}$. In this paper, we consider $\eta$ to take the canonical form, i.e., $\eta(x) = x$. The GLMs are characterized by the different $\psi$. For example, in linear regression with Gaussian response, $\psi(x) = x^{2}/2$; in the logistic regression with binary response, $\psi(x) = log(1+e^{x})$; and in Poisson regression with non-negative integer response, $\psi(x) = e^{x}$.

A standard method for addressing GLM involves minimizing the loss function defined as the negative log-likelihood. Therefore, to implement the transfer learning for FGLM, one can naturally substitute the square loss in TL-FLR (Algorithm~\ref{algo: A transfer algorithm}) and ATL-FLR (Algorithm~\ref{algo: aggregate transfer algorithm}) with the negative log-likelihood loss, i.e.
\begin{equation*}
    \ell(\beta, Z_{i}^{(t)}) = - Y_{i}^{(t)} \eta(  \theta_{i}^{(t)} ) + \psi( \theta_{i}^{(t)} ).
\end{equation*}
We refer to these transfer learning algorithms for FGLM as TL-FGLM and ATL-FGLM.  To establish the optimality of TL-FGLM and ATL-FGLM, the following technical assumptions are required.
\begin{assumption}\label{assump: FGLM Lip continuous}
    Assume $\psi$ is Lipschitz continuous on its domain, and $\psi' < \infty$.
\end{assumption}
\begin{assumption}\label{assump: FGLM bounded second derivative}
    Assume there exist constants $0<\mcA_{1}\leq\mcA_{2}<\infty$ such that the function $\psi''$ satisfies 
    \begin{equation*}
        \mcA_{1} \leq \inf_{s \in \mcT} \psi''(s) \leq \psi''(s) \leq \sup_{s \in \mcT} \psi''(s) \leq \mcA_{2}.
    \end{equation*}
\end{assumption}
Assumption~\ref{assump: FGLM Lip continuous} is natural in most GLM literature and is satisfied by many popular exponential families. Assumption~\ref{assump: FGLM bounded second derivative} restricts the $\psi''$ in the bounded region and thus restricts the variance of $y$. 

Since the conditional mean for FGLM is $\E[Y_{i}|X_{i}] = \eta'(\langle \beta, X^{(0)} \rangle_{L^2})$, we evaluate the accuracy by excess risk, i.e. $\mcE(\hat{\beta}):= \E_{X^{(0)}} [ \eta'(\langle \hat{\beta}, X^{(0)} \rangle_{L^2}) -\eta'(\langle \beta^{(0)}, X^{(0)} \rangle_{L^2}) ]^2$. 
\begin{theorem}\label{thm: bounds for TL-FGLM}
Under the same assumption of Theorem~\ref{thm: upper bound for TL-FLR} and suppose Assumptions~\ref{assump: FGLM Lip continuous},~\ref{assump: FGLM bounded second derivative} hold. 
\begin{enumerate}
    \item (Lower Bound) For any possible estimator $\tilde{\beta}$ based on target and source datasets, the excess risk of $\tilde{\beta}$ satisfies
    \begin{equation*}
        \lim_{a\rightarrow 0} \lim_{n \rightarrow \infty}   \inf_{\tilde{\beta}} \sup_{ \Theta(h,R) } P\left\{ \mcE\left( \tilde{\beta} \right) \geq  a \left( n_{\mcS} ^{-\frac{2r}{2r+1}} + n_{0}^{-\frac{2r}{2r+1}} \xi(h,R)  \right) \right\} = 1.
    \end{equation*}

    \item (Upper Bound) If $n_0/n_{\mcS}\nrightarrow0$ and $\lambda_1 \asymp n_{\mcS} ^{-\frac{2r}{2r+1}}$ and $\lambda_2 \asymp n_0^{-\frac{2r}{2r+1}}$, then for the output $\hat{\beta}$ of TL-FGLM,
    \begin{equation*}\label{eqn: upper for FGLM}
    \sup_{ \Theta(h,R) } \mcE\left(\hat{\beta}\right) = O_{\mbP}
    \left( n_{\mcS} ^{-\frac{2r}{2r+1}} + n_0^{-\frac{2r}{2r+1}}\xi(h,R)  \right).
    \end{equation*}
\end{enumerate}
\end{theorem}
\begin{remark}
    The error bounds of TL-FLR and TL-FGLM are the same, which is consistent with the case in the target-only learning between FLR and FGLM, see \citet{cai2012minimax,du2014penalized}. However, we note that the proof is not a trivial extension of FLR since minimizing the regularized negative likelihood usually will not provide an analytical solution.
\end{remark}
\begin{remark}
    Due to the same upper bound for TL-FLR and TL-FGLM, the upper bound of ATL-FGLM is the same as ATL-TLR, i.e., with the same aggregation cost.
\end{remark}

\section{Experiments}\label{sec: experiments}

We illustrate the proposed transfer learning algorithms by conducting experiments on synthetic data and two real-world applications, including a financial market data application for FLR and a wearable device detection application for FGLM. \footnote{The R code and the application datasets are available in \url{https://github.com/haotianlin/HTL-FLM}.}

\subsection{Simulations}
We consider the following algorithms: OFLR, TL-FLR, ATL-FLR, Detection Transfer Learning (Detect-TL) \citep{tian2022transfer}, and Exponential Weighted ATL-FLR (ATL-FLR (EW)) \citep{li2022transfer}. 
To set up the RKHS, we consider the setting in \citet{cai2012minimax}. Let $\psi_k(s) = \sqrt{2}\cos(\pi k s)$ for $j\geq 1$ and define the reproducing kernel $K$ of $\mcH_{K}$ as $  K(\cdot,\cdot) = \sum_{k=1}^{\infty} k^{-2} \psi_{k}(\cdot) \psi_{k}(\cdot)$.

For the target model, $\beta^{(0)}(s)$ is set to be (1) $\beta_{1}^{(0)}(s) = \sum_{k=1}^{\infty}  4\sqrt{2}(-1)^{k-1}k^{-2}  \psi_{k}(s)$; (2) $\beta_{2}^{(0)}(s) = 4\cos(3\pi s)$; (3) $\beta_{3}^{(0)}(s) = 4\cos(3\pi s) + 4\sin(3 \pi s)$. For a specific $h$, let $\mcS = \{ l:  \|\beta_{0} - \beta_{t}\|_{K} \leq h \}$, then we generate source models as follows. We scale target models such that their RKHS norm is $20$. If $t\in \mcS$, then $\beta_{t}(t)$ is set to be $\beta_{t}(s) = \beta_0(s) + \sum_{k=1}^{\infty} \left( \mathcal{U}_{k} (\sqrt{12}h/\pi k^{2} ) \right) \psi_{k}(s)$ with $\mathcal{U}_{k}$'s i.i.d. uniform random variable on $[-1,1]$. If $t\in \mcS^{c}$, then $\beta_{t}(s)$ is generated from a Gaussian process with mean function $\cos(2\pi s)$ with kernel $\exp(-15|s-t|)$ as covariance kernel. The predictors $X^{(t)}$ are i.i.d. generated from a Gaussian process with the mean function $\sin(\pi s)$ and the covariance kernels are set to be Mat\'ern kernel $C_{\nu, \rho}$ \citep{cressie1999classes} where the parameter $\nu$ controls the smoothness of $X^{(t)}$. We set the covariance kernel of $X^{(t)}$ as $C_{1/2, 1}$ for the target tasks and $C_{3/2, 1}$ for source tasks, to fulfill Assumptions~\ref{assump: eigen assumption}. We note that such a setting is more challenging than assuming that the target and source tasks have the same covariance kernel. All functions are generated on $[0,1]$ with 50 evenly spaced points and we set $n_0 = 150$ and $n_{t} = 100$. For each algorithm, we set the regularization parameters as $\lambda_1$ and $\lambda_2$ as the optimal values in Theorem~\ref{thm: upper bound for TL-FLR} and select the constants using cross-validation. The excess risk for the target tasks is calculated via the Monte-Carlo method by using $1000$ newly generated predictors $X^{(0)}$.

\begin{figure*}[ht]
    \centering
    \includegraphics[width = \textwidth]{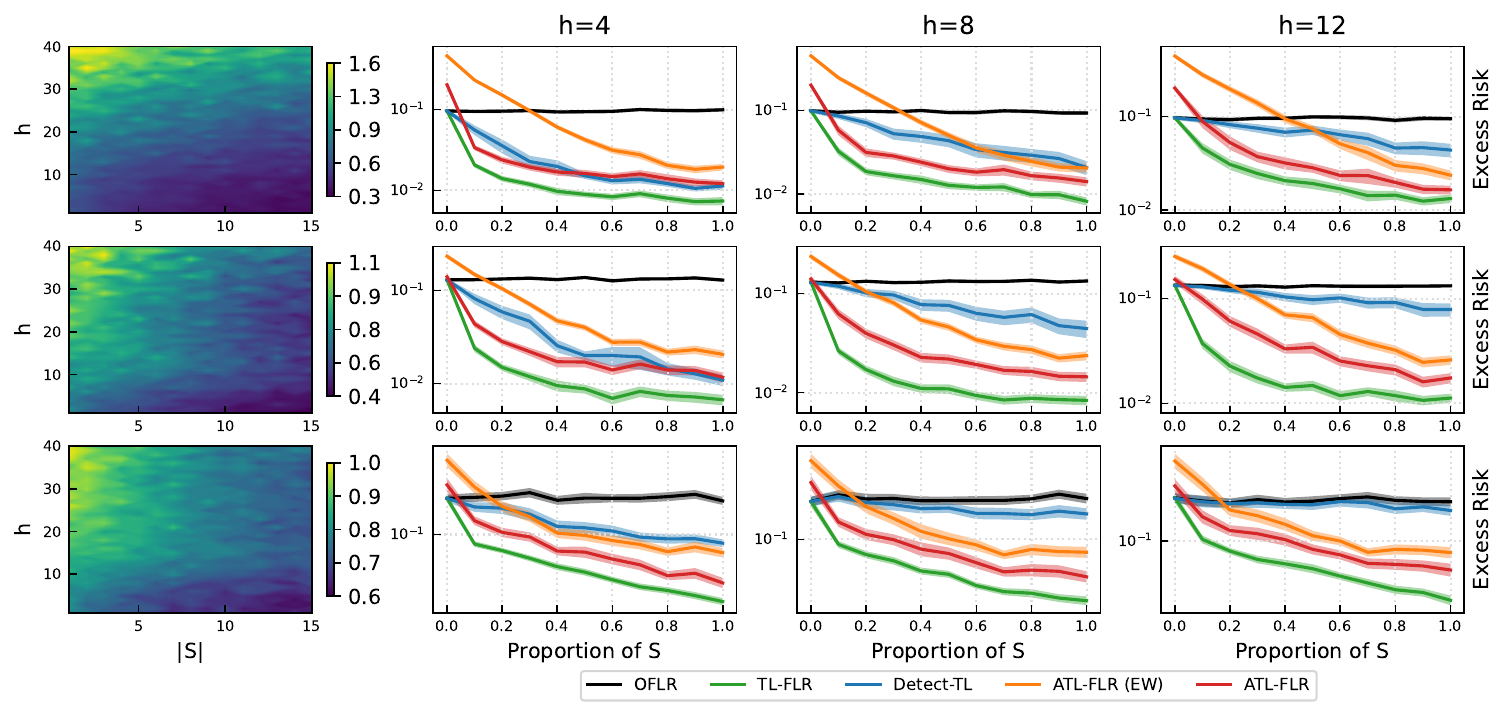}
    \caption{Left panel (Heatmap): log relative excess risk of TL-FLR to OFLR. Right panels (Line Chart): Excess Risk of different transfer learning algorithms. Each row corresponds to a $\beta^{(0)}$, and the y-axes for each row are under the same scale. The result for each sample size is an average of 100 replicate experiments, with the shaded area indicating $\pm$ two standard errors.}
    \label{fig: Combination Results}
\end{figure*}

In the left panel of Figure~\ref{fig: Combination Results}, we compare TL-FLR with OFLR by considering \textit{relative excess risk}, i.e., the ratio of TL-FLR's excess risk to OFLR's. We note that since the RKHS of $\beta^{(0)}$ is fixed, the magnitude of $h$ is proportional to $\xi(h,\beta_{\mcS})$, and a large $h$ indicates less similarity between sources and target tasks. Overall, the effectiveness of TL-FLR for different $\beta^{(0)}$ presents a consistent pattern, i.e., with more transferable sources involved and smaller $h$ (bottom right), TL-FLR has a more significant improvement, while with fewer sources and larger $h$ (top left), the transfer may be worse than OFLR. 

In the right panel of Figure~\ref{fig: Combination Results}, we evaluate ATL-FLR under unknown $\mcS$ cases. We set $\mcS$ as a random subset of $\{1,2,\cdots,20\}$ such that $|\mcS|$ is equal to $0,2,\cdots,20$. We also implement TL-FLR using true $\mcS$ and OFLR as baseline. In all scenarios, ATL-FLR outperforms all its competitors. Comparing ATL-FLR with ATL-FLR(EW), even though ATL-FLR(EW) has similar patterns as ATL-FLR, we can see the gaps between the two curves are larger when the proportion of $\mcS$ is small, showing that ATL-FLR(EW) is more sensitive to source tasks in $\mcS^c$, while ATL-FLR is less affected. Detect-TL only has a considerable reduction on the excess risk with relatively small $h$, but provides limited improvement when $h$ is large, indicating its limited performance when limited knowledge is available in sources.

\subsection{Application for Predicting Stock Sector Returns with FLM}
In this section, we demonstrate an application of the proposed algorithms in the financial market. The goal of portfolio management is to balance future stock returns and risk, and thus, investors can rebalance their portfolios according to their goals. Some investors may be interested in predicting the future stock returns in a specific sector, and transfer learning can borrow market information from other sectors to improve the prediction of interest.

In this stock data application, for two given adjacent months, we focus on utilizing the Monthly Cumulative Return (MCR) of the first month to predict the Monthly Return (MR) of the subsequent month and improving the prediction accuracy on a certain sector by transferring market information from other sectors. Specifically, suppose for a specific stock, the daily price for the first month is $\{s^{1}(t_0),s^{1}(t_1),\cdots,s^{1}(t_m)\}$ and for the second month is $\{s^{2}(t_0),s^{2}(t_1),\cdots,s^{2}(t_m)\}$, then the predictors and responses are expressed as
\begin{equation}\label{eqn: real data formulation}
    X(t) = \frac{s^{1}(t) - s^{1}(t_0)}{s^{1}(t_0)} \quad \text{and} \quad Y = \frac{s^{2}(t_{m}) - s^{2}(t_0)}{s^{2}(t_0)}.
\end{equation}
The stock price data are collected from \href{https://finance.yahoo.com/}{Yahoo Finance}, and we focus on stocks whose corresponding companies have a market cap of over $20$ billion. We divide the sectors based on the division criteria on \href{https://www.nasdaq.com/market-activity/stocks/screener}{Nasdaq}. The raw data obtained from websites is processed to match the format in (\ref{eqn: real data formulation}).

After pre-processing, the dataset consists of total $11$ sectors: Basic Industries (BI), Capital Goods (CG), Consumer Durable (CD), Consumer Non-Durable (CND), Consumer Services (CS), Energy (E), Finance (Fin), Health Care (HC), Public Utility (PU), Technology (Tech), and Transportation (Trans), with the number of stocks in each sector as $60,58,31,30,104,55,70,68,46,103,41$. The period of the stock price is 05/01/2021 to 09/30/2021.

We compare the performance of \textit{Pooled Transfer (Pooled-TL)}, \textit{Naive Transfer (Naive-TL)}, \textit{Detect-TL}, \textit{ATL-FLR(EW)} and \textit{ATL-FLR}. Naive-TL implements TL-FLR by setting all source sectors belonging to $\mcS$, while the Pooled-TL one omits the calibration step in Naive-TL, and the other three are the same as the former simulation section. The learning of each sector is treated as the target task each time, and all the other sectors are sources. We randomly split the target sector into the train ($80\%$) and test ($20\%$) sets and report the ratio of the four approaches' prediction errors to OFLR's on the test set. We consider the Mat\'ern kernel as the reproducing kernel $K$ again. Specifically, we set $\rho=1$ and $\nu = 1/2,3/2,\infty$ (where $\nu = 1/2$ is equivalent to the exponential kernel and $\nu = \infty$ is equivalent to the Gaussian kernel), which endows $K$ with different smoothness properties. The tuning parameters are selected via Generalized Cross-Validation(GCV). Again, we replicate the experiment $100$ times and report the average prediction error with standard error in Figure~\ref{fig: financial market results}.

\begin{figure}[htbp]
    \centering
    \subfloat[$\nu = \frac{1}{2}$ (Exponential Kernel)]{
    \includegraphics[page = 1,width=0.7\textwidth]{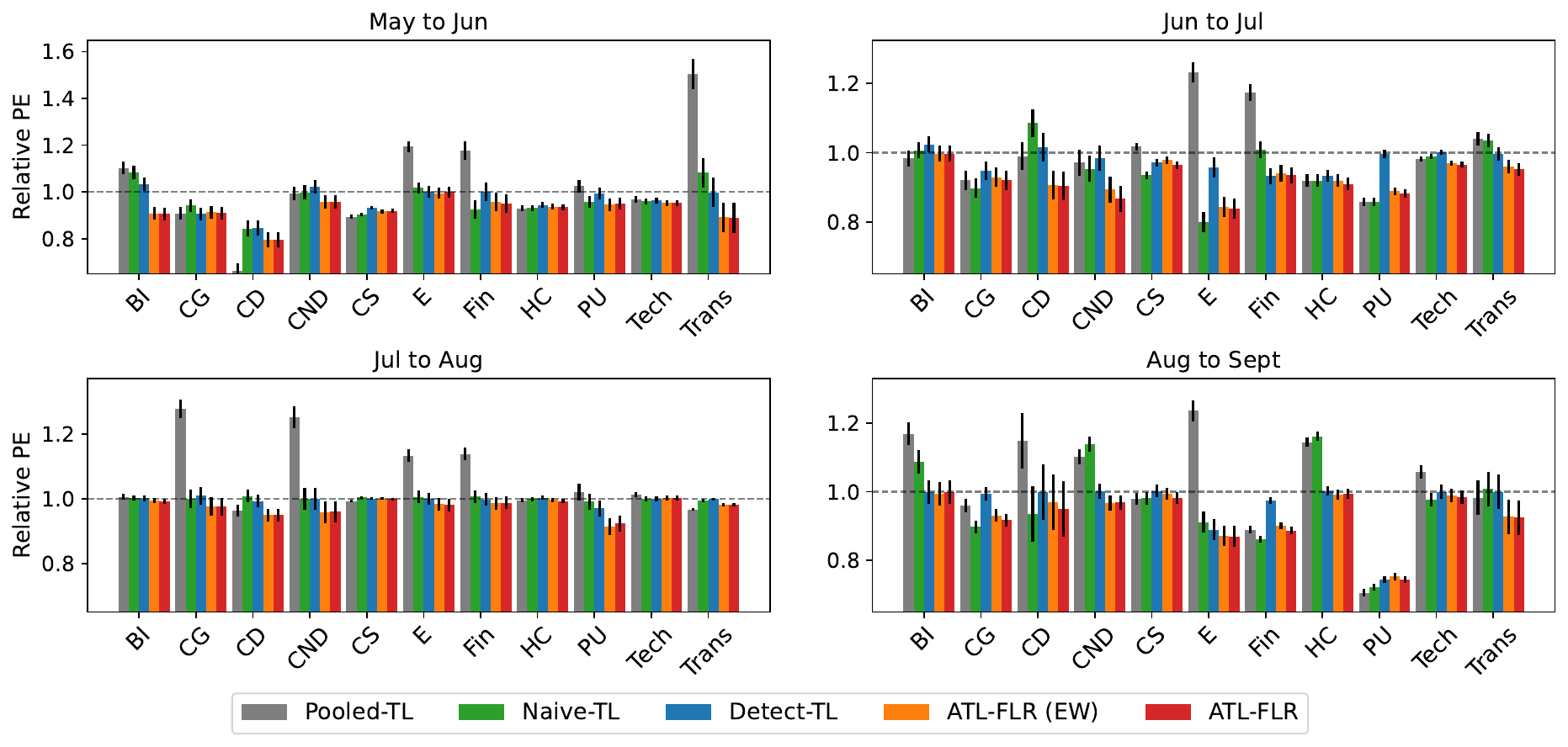}
    }

    \subfloat[$\nu = \frac{3}{2}$ ]{
    \includegraphics[page = 1,width=0.7\textwidth]{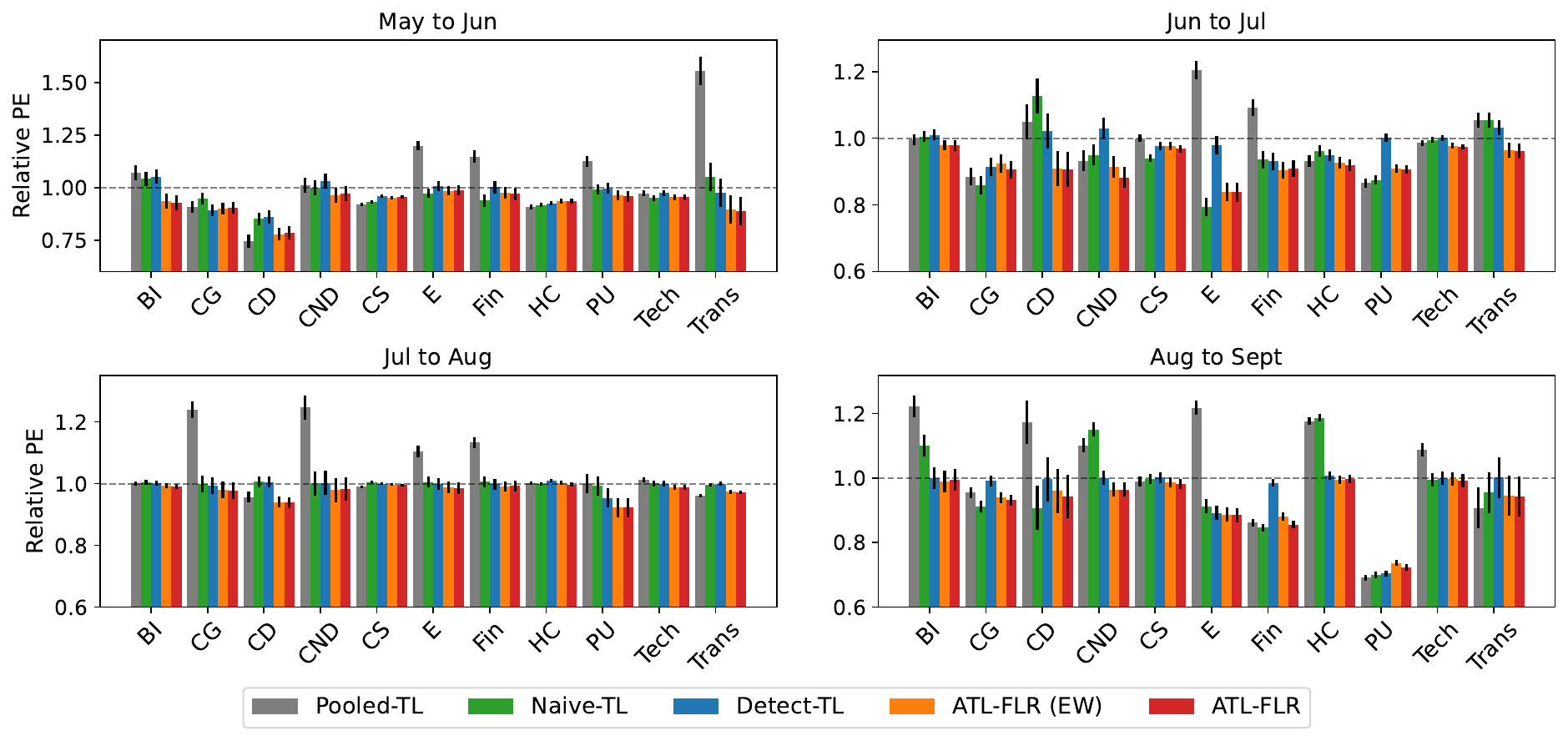}
    }

    \subfloat[$\nu = \infty$ (Gaussian Kernel)]{
    \includegraphics[page = 1,width=0.7\textwidth]{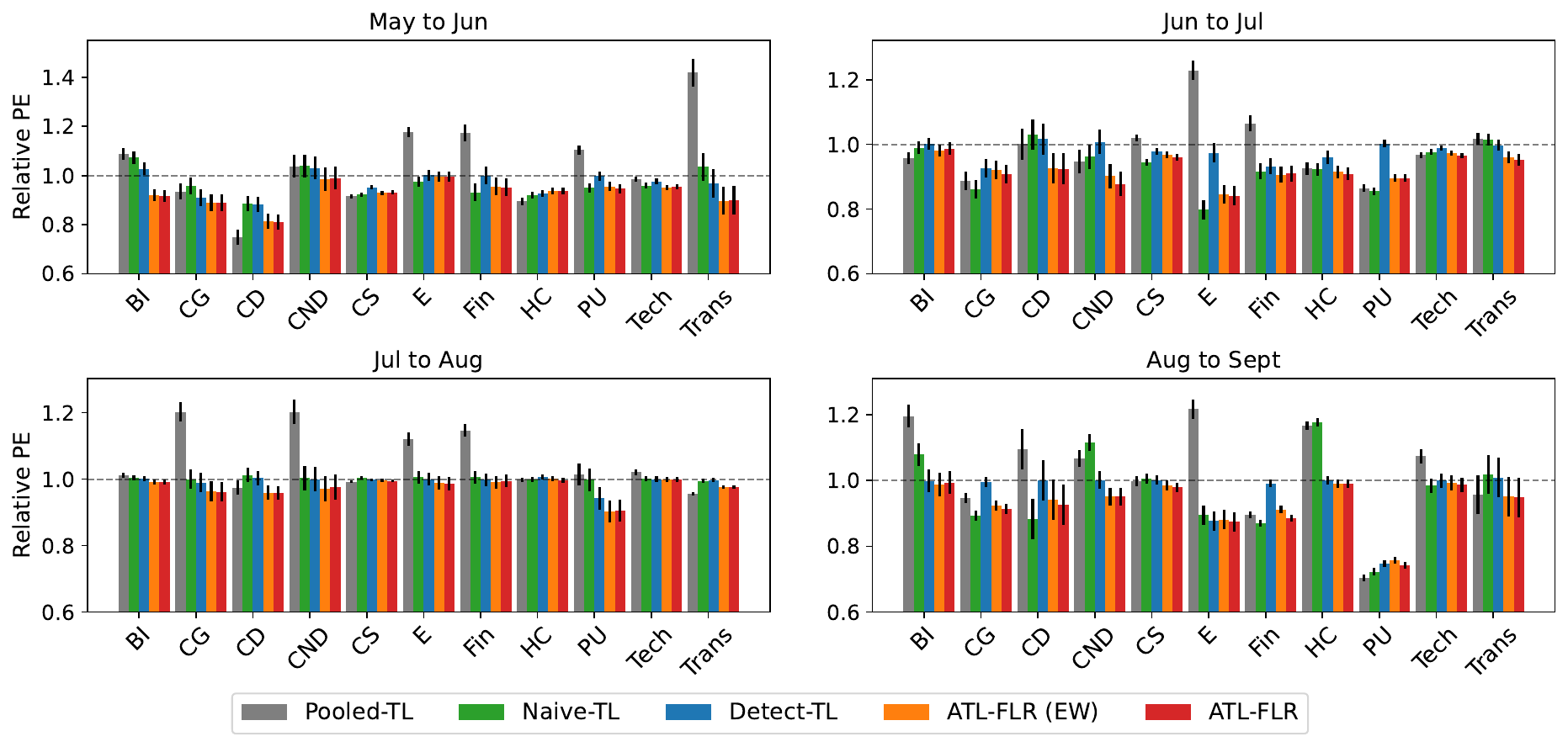}
    }
    
    \caption{Relative prediction error of Pooled-TL, Naive-TL, Detect-TL, ATL-FLR(EW), and ATL-FLR to OFLR for each target sector. Each bar is an average of 100 replications, with the standard error as the black line.}
    \label{fig: financial market results}
\end{figure}

First, we note that the Pooled-TL and Naive-TL only reduce the prediction error in a few sectors, but make no improvement or even downgrade the predictions in most sectors. This implies that the effect of direct transfer learning can be quite random, as it can benefit the prediction of the target sector when it shares high similarities with other sectors, while having worse performance when similarities are low. Besides, Naive-TL shows an overall better performance than Pooled-TL, demonstrating the importance of the calibration step. For Detect-TL, all the ratios are close to $1$, showing its limited improvement, which is as expected, as it can miss positive transfer sources easily. Finally, both ATL-FLR(EW) and ATL-FLR provide more robust improvements on average. We can see both of them have improvements across almost all the sectors, regardless of the similarity between the target sector and the source sectors. Comparing the results from different kernels, we can see the improvement patterns are consistent across all the sectors and adjacent months, showing the proposed algorithms are also robust to different reproducing kernels.

\subsection{Application for Classifying Human Activity with FGLM}

Personal wearable devices have become increasingly popular as they can detect users' movements and provide feedback/records. However, the limited data can make the device's detection inaccurate for a newly registered user. The technology can make detecting new users' actions more accurate by leveraging learned hypotheses from other users with similar features. Under this context, we consider the Human Activity Recognition (HAR) dataset \citep{anguita2013public}, which contains the recordings of volunteers performing daily living activities, including walking, walking upstairs, walking downstairs, sitting, standing, and lying. Each volunteer carried a waist-mounted smartphone with an embedded accelerometer and gyroscope sensors to capture the body acceleration and gravity signal. 

In this section, we evaluate the efficacy of our proposed transfer learning method for FGLM and specifically aim at functional logistic regression. The goal is to distinguish the actions of walking upstairs (marked as $1$) and walking downstairs (marked as $0$) by applying functional logistic regression with the covariate $X$ as the body acceleration signal in the vertical direction over time. We treat the classification for each volunteer as a separate task. After preprocessing, there are a total of $30$ volunteer datasets. Each dataset is balanced in terms of label proportion, and each sample size is between $78$ and $127$. The covariate of body acceleration signals is measured in equal spacing $128$ consecutive time points, and we pick the first 32 points to reduce motion cycles. For each dataset, we randomly split the samples into a training set $(80\%)$ and test set $(20\%)$. Each volunteer is treated as a target each time, and all the other volunteers are used as sources.

We compare the performance of our proposed \textit{ATL-FLR(EW)} and \textit{ATL-FLR} with the non-transfer baseline \textit{OFLR} and some competitors, including \textit{Pooled Transfer (Pooled-TL)}, \textit{Naive Transfer (Naive-TL)}, \textit{Detect-TL}. We consider the Mat\'ern kernel $K_{\nu,\rho}$ \citep{cressie1999classes} as the reproducing kernel and set $\rho=1$ and $\nu = 1/2$. The regularization parameters are selected via Generalized Cross-Validation (GCV). We repeat the experiment $100$ times to assess the variability in train/test data split and report the misclassification rate (in percentage) on test data with standard error. The threshold for all binary classification is set to be 0.5. The results are placed in Table~\ref{table: misclassification rate for FGLM application}.

Based on Table~\ref{table: misclassification rate for FGLM application}, the results demonstrate that the proposed ATL-FGLM can provide the lowest misclassification rate for most of the volunteers. In cases where ATL-FGLM is not the best, its misclassification rate is also close to the lowest (e.g., Volunteer 9,17,27), which verifies its robustness to non-informative sources.

\begin{table}[!htbp]
    \centering
        \caption{Misclassification rates $(\%)$ on the test set for each target volunteer with standard errors in subscript.}
        \resizebox{\columnwidth}{!}{
    \begin{tabular}{c p{20mm}p{20mm}p{20mm}p{20mm}p{20mm}p{20mm}}
    \hline
    \multirow{2}{*}{Target} & \multicolumn{6}{c}{Transfer Learning Algorithms} \\ \cline{2-7}
    & Non-TL & Pooled-TL & Naive-TL & Detect-TL & ATL-FGLM \newline (EW) & ATL-FGLM \\ 
    \hline
    Volunteer1 & $5.548_{0.395}$ & $8.290_{0.408}$ & $5.129_{0.364}$ & $5.194_{0.346}$ & $4.774_{0.365}$ & $\textbf{4.387}_{0.365}$  \\ 
    Volunteer2 & $15.679_{0.522}$ & $17.679_{0.541}$ & $16.536_{0.514}$ & $15.714_{0.520}$ & $15.179_{0.521}$ & $\textbf{15.036}_{0.528}$  \\ 
    Volunteer3 & $21.879_{0.747}$ & $21.576_{0.600}$ & $21.121_{0.628}$ & $21.121_{0.650}$ & $18.667_{0.716}$ & $\textbf{17.758}_{0.680}$  \\ 
    Volunteer4 & $25.000_{0.705}$ & $25.767_{0.693}$ & $23.467_{0.686}$ & $\textbf{21.133}_{0.669}$ & $22.467_{0.701}$ & $21.300_{0.672}$  \\ 
    Volunteer5 & $29.000_{0.913}$ & $43.857_{0.791}$ & $28.786_{0.837}$ & $29.393_{0.804}$ & $24.714_{0.800}$ & $\textbf{24.357}_{0.779}$  \\ 
    Volunteer6 & $16.448_{0.600}$ & $\textbf{15.345}_{0.544}$ & $16.103_{0.572}$ & $16.103_{0.550}$ & $16.414_{0.572}$ & $16.276_{0.565}$  \\ 
    Volunteer7 & $12.759_{0.489}$ & $15.448_{0.638}$ & $14.207_{0.599}$ & $13.379_{0.574}$ & $12.483_{0.543}$ & $\textbf{12.483}_{0.530}$  \\ 
    Volunteer8 & $27.000_{0.697}$ & $24.304_{0.801}$ & $23.913_{0.684}$ & $25.087_{0.762}$ & $25.609_{0.786}$ & $\textbf{23.609}_{0.780}$  \\ 
    Volunteer9 & $\textbf{23.536}_{0.614}$ & $25.679_{0.623}$ & $25.036_{0.667}$ & $25.286_{0.651}$ & $24.714_{0.665}$ & $24.536_{0.657}$  \\ 
    Volunteer10 & $31.120_{0.732}$ & $32.600_{0.674}$ & $30.160_{0.740}$ & $30.600_{0.736}$ & $30.080_{0.731}$ & $\textbf{29.480}_{0.764}$  \\ 
    Volunteer11 & $16.267_{0.575}$ & $14.400_{0.503}$ & $15.000_{0.487}$ & $14.200_{0.552}$ & $13.967_{0.552}$ & $\textbf{13.400}_{0.574}$  \\ 
    Volunteer12 & $6.667_{0.402}$ & $9.100_{0.452}$ & $6.633_{0.403}$ & $5.600_{0.367}$ & $5.467_{0.359}$ & $\textbf{4.767}_{0.352}$  \\ 
    Volunteer13 & $19.903_{0.539}$ & $18.903_{0.543}$ & $18.581_{0.552}$ & $19.065_{0.564}$ & $18.581_{0.567}$ & $\textbf{18.129}_{0.546}$  \\ 
    Volunteer14 & $37.300_{0.578}$ & $\textbf{33.533}_{0.688}$ & $35.833_{0.702}$ & $35.667_{0.692}$ & $35.733_{0.648}$ & $35.867_{0.679}$  \\ 
    Volunteer15 & $21.926_{0.618}$ & $23.000_{0.554}$ & $20.778_{0.595}$ & $20.926_{0.587}$ & $20.481_{0.628}$ & $\textbf{20.333}_{0.604}$  \\ 
    Volunteer16 & $18.379_{0.630}$ & $19.310_{0.600}$ & $17.931_{0.626}$ & $18.241_{0.632}$ & $17.897_{0.625}$ & $\textbf{17.655}_{0.620}$  \\ 
    Volunteer17 & $\textbf{17.964}_{0.614}$ & $20.250_{0.616}$ & $22.393_{0.716}$ & $20.857_{0.635}$ & $19.893_{0.564}$ & $18.821_{0.576}$  \\ 
    Volunteer18 & $24.971_{0.580}$ & $21.412_{0.591}$ & $20.029_{0.579}$ & $20.176_{0.545}$ & $19.559_{0.577}$ & $\textbf{18.471}_{0.567}$  \\ 
    Volunteer19 & $20.833_{0.657}$ & $19.875_{0.672}$ & $19.375_{0.631}$ & $\textbf{18.958}_{0.668}$ & $19.000_{0.668}$ & $19.167_{0.686}$  \\ 
    Volunteer20 & $15.966_{0.617}$ & $18.379_{0.618}$ & $15.897_{0.548}$ & $15.207_{0.565}$ & $15.586_{0.605}$ & $\textbf{15.069}_{0.549}$  \\ 
    Volunteer21 & $16.036_{0.550}$ & $18.321_{0.585}$ & $16.429_{0.585}$ & $15.357_{0.582}$ & $\textbf{15.107}_{0.534}$ & $15.643_{0.546}$  \\ 
    Volunteer22 & $19.500_{0.716}$ & $19.708_{0.758}$ & $19.750_{0.736}$ & $18.750_{0.734}$ & $18.500_{0.686}$ & $\textbf{18.458}_{0.689}$  \\ 
    Volunteer23 & $22.323_{0.648}$ & $20.387_{0.637}$ & $20.871_{0.626}$ & $21.871_{0.641}$ & $20.774_{0.642}$ & $\textbf{19.742}_{0.582}$ \\ 
    Volunteer24 & $16.314_{0.580}$ & $\textbf{15.171}_{0.529}$ & $15.400_{0.536}$ & $15.514_{0.550}$ & $15.543_{0.559}$ & $15.514_{0.554}$  \\ 
    Volunteer25 & $32.139_{0.583}$ & $30.194_{0.648}$ & $29.417_{0.595}$ & $29.444_{0.587}$ & $29.306_{0.593}$ & $\textbf{29.056}_{0.619}$  \\ 
    Volunteer26 & $6.844_{0.376}$ & $8.156_{0.389}$ & $7.437_{0.392}$ & $6.750_{0.392}$ & $6.937_{0.384}$ &  $\textbf{6.719}_{0.383}$ \\ 
    Volunteer27 & $8.500_{0.473}$ & $8.786_{0.414}$ & $7.643_{0.425}$ & $7.536_{0.476}$ & $\textbf{6.286}_{0.394}$ & $6.393_{0.407}$  \\ 
    Volunteer28 & $25.103_{0.709}$ & $24.586_{0.716}$ & $23.448_{0.712}$ & $24.172_{0.739}$ & $23.517_{0.726}$ & $\textbf{23.069}_{0.694}$  \\ 
    Volunteer29 & $39.276_{0.761}$ & $38.586_{0.738}$ & $37.345_{0.775}$ & $37.276_{0.752}$ & $37.552_{0.819}$ & $\textbf{37.103}_{0.846}$  \\ 
    Volunteer30 & $33.474_{0.749}$ & $33.263_{0.610}$ & $30.526_{0.672}$ & $32.632_{0.712}$ & $30.000_{0.777}$ & $\textbf{29.263}_{0.800}$ \\ 
    \hline
    \end{tabular} 
    }
    \label{table: misclassification rate for FGLM application}
\end{table}

\section{Discussion and Future Work}\label{sec: discussion}

This paper studies transfer learning under the functional linear model framework, including FLR and FGLM. We derive the asymptotic rates for the excess risk over the target domain and show a faster statistical rate depending on both source sample size and the magnitude of similarity across tasks. Our theoretical analysis helps researchers better understand the transfer dynamics of OTL. Moreover, we leverage the sparse aggregation to implement the transfer practically, alleviate the negative transfer effect, and achieve nearly optimal statistical rates.

\paragraph{The Role of Smoothness of Offset and Optimality under Misspecification.} In the current analysis, we assume that the offset slope function resides in $\mcH_{K}$ and shares the same smoothness as the target and source functions. This assumption is reasonable given the unknown true smoothness of the offset slope function. However, the success of OTL hinges on the offset slope function possessing a simpler structure (well-regularized) that can be effectively learned from a small sample size. Consequently, a critical open question emerges: if the offset slope function exhibits higher smoothness, how do we identify the different smoothness for the source and offset slope functions and subsequently apply the appropriate kernel to achieve optimal rates? Besides, the optimal rates rely on oracle conditions, i.e., correctly specified hypothesis space and known smoothness order for optimal regularization parameters. Since these conditions are impossible to hold in practice, how do we achieve optimal rates without them?

Recently, \citet{lin2024smoothness} provides the first attempt to explore these questions within the nonparametric regression setting. They find that the Gaussian kernel is an optimally robust solution against model misspecification and develop adaptive optimal rates under varying smoothness scenarios. Although their findings are specific to Sobolev spaces, it is worth investigating whether a similar solution exists for FLR and FGLM since the kernel in these contexts is a composition of the covariance kernel and the RKHS kernel.

\paragraph{Efficiently Building Source Hypotheses.} When constructing the source hypothesis $\hat{\beta}_{\mcS}$, Algorithm~\ref{algo: A transfer algorithm} employs a data fusion technique in the transfer step. Specifically, it consolidates all the source datasets in $\mcS$ and simultaneously trains a unified $\hat{\beta}_{\mcS}$. The time and space complexities for this step are $O(n_{\mcS}^{3})$ and $O(n_{\mcS}^2)$, respectively. In real-world transfer learning scenarios, the involvement of numerous large sample-size source datasets can lead to significant algorithmic scaling issues, thereby limiting the practicality of the transfer step and the proposed algorithm. A potential solution is to perform the learning in a distributed manner \citep{tong2021distributed,liu2022statistical}. Additionally, if the learned source models for each source domain are available, one might consider fusing these learned hypotheses instead of the raw data.

\bibliographystyle{plainnat}
\bibliography{ref}

\newpage
\section*{Content of Appendices}
\appendix
\startcontents[sections]
\printcontents[sections]{}{1}{}

\section{Background of RKHS and Integral Operators}
In this section, we will present some facts about the RKHS and also the integral operator of a kernel that are useful in our proof and refer readers to \citet{wendland2004scattered} for a more detailed discussion. 

Let $\mcT$ be a compact set of $\mbR$. For a real, symmetric, square-integrable, and semi-positive definite kernel $K: \mcT \times \mcT \rightarrow \mbR$, we denote its associated RKHS as $\mcH_{K}$. For the reproducing kernel $K$, we can define its integral operator $L_{K}: L^{2} \rightarrow L^{2}$ as 
\begin{equation*}
    L_{K} (f)(\cdot) = \int_{\mcT} K(s,\cdot)f(s) ds.
\end{equation*}
$L_{K}$ is self-adjoint, positive-definite, and trace class (thus Hilbert-Schmidt and compact). By the spectral theorem for self-adjoint compact operators, there exists an at most countable index set $N$, a non-increasing summable positive sequence $\{\tau_{j}\}_{j\geq 1}$ and an orthonormal basis of $L^{2}$, $\{e_{j}\}_{j\geq1}$ such that the integrable operator can be expressed as 
\begin{equation*}
    L_{K}(\cdot) = \sum_{j\in N} \tau_{j} \langle \cdot, e_{j} \rangle_{L^{2}} e_{j}.
\end{equation*}
The sequence $\{\tau_{j}\}_{j\geq 1}$ and the basis $\{e_{j}\}_{j\geq1}$ are referred as the eigenvalues and eigenfunctions. The Mercer's theorem shows that the kernel $K$ itself can be expressed as 
\begin{equation*}
    K(x,x') = \sum_{j\in N} \tau_{j} e_{j}(x) e_{j}(x'), \quad \forall x,x' \in \mcT,
\end{equation*}
where the convergence is absolute and uniform.

We now introduce the fractional power integral operator and the composite integral operator of two kernels. For any $s\geq 0$, the fractional power integral operator $L_{K}^{s}:L^{2} \rightarrow L^{2}$ is defined as 
\begin{equation*}
    L_{K}^{s}(\cdot) = \sum_{j\in N} \tau_{j}^{s} \langle \cdot, e_{j} \rangle_{L^{2}} e_{j}.
\end{equation*}
For two kernels $K_{1}$ and $K_{2}$, we define their composite kernel as 
\begin{equation*}
    (K_{1}K_{2})(x,x') = \int_{\mcT} K_{1}(x,s)K_{2}(s,x') ds,
\end{equation*}
and thus $L_{K_{2}K_{2}} = L_{K_{1}} \circ L_{K_{2}}$. Given these definitions, for a given reproducing kernel $K$ and covariance function $C$, the definition of $\Gamma$ in the main paper is
\begin{equation*}
   L_{\Gamma} = L_{K^{\frac{1}{2}}} \circ L_{C} \circ L_{K^{\frac{1}{2}}} \quad \text{and}\quad \Gamma := K^{\frac{1}{2}}CK^{\frac{1}{2}}.
\end{equation*}
If both $L_{K^{\frac{1}{2}}}$ and $L_{C}$ are bounded linear operators, the spectral algorithm guarantees the existence of eigenvalues $\{s_{j}\}_{j\geq1}$ and eigenfunctions $\{\psi_{j}\}_{j\geq1}$.

\section{Sparse Aggregation Process}\label{apd: sparse aggregation}
We provide the procedure of sparse aggregation in Step 4 of ATL-FLR (Algorithm~\ref{algo: aggregate transfer algorithm}) for completeness and refer readers to \citet{gaiffas2011hyper} for more details.

\begin{algorithm}[H]
\caption{Sparse Aggregation}\label{algo: sparse aggregation}
\begin{algorithmic}[1]
    \State {\bfseries Input:} The candidate set $\mcF$; target dataset $\mcD_{\mcI^c}^{(0)}$; pre-specified parameter $c$, $\phi$.\\
    
    \State Split $\mcD_{\mcI^c}^{(0)}$ into equal size set, with index set $\mcI_{1}^{c}$ and $\mcI_{2}^{c}$
   
    \State Use $\mcD_{\mcI_{1}^c}^{(0)}$ to define a random subset of $\mcF$ as
    \begin{equation*}
        \mcF_1 = \left\{ \beta \in \mcF : R_{n,\mcI_{1}^c}(\beta) \leq R_{n,\mcI_{1}^c}(\hat{\beta}_{n1}) + c \max \left( \phi \left\| \hat{\beta}_{n1} - \beta \right\|_{n,\mcI_{1}^c} , \phi^2  \right) \right\}
    \end{equation*}
    where 
    \begin{equation*}
    \left\| \beta \right\|_{n,\mcI_{1}^c}^2 = \frac{1}{|\mcI_{1}^c|} \sum_{i\in \mcI_{1}^c} \langle X_{i}^{(0)}, \beta \rangle_{L^2}^2, \quad R_{n,\mcI}(\beta) = \frac{1}{|\mcI|} \sum_{i\in \mcI} ( Y_{i}^{(0)} - \langle X_{i}^{(0)},  \beta \rangle_{L^2} )^2, \quad \hat{\beta}_{n1} = \underset{\beta\in \mcF}{ argmin} R_{n,\mcI_{1}^c}(\beta)
    \end{equation*}
    
    \State Set $\mcF_{2}$ as following:
    \begin{equation*}
        \mcF_2 = \{ c_1 \beta_1 + c_2 \beta_2 : \beta_1,\beta_2\in \mcF_1 \ \textit{and} \ c_1 + c_2 = 1 \}
    \end{equation*}
    then, return 
    \begin{equation*}
        \hat{\beta}_{a}  = \underset{ \beta \in \mcF_2 }{argmin} R_{n, \mcI_2^{c}}(\beta).
    \end{equation*}
\end{algorithmic}
\end{algorithm}

The sparse aggregation algorithm is stated in Algorithm~\ref{algo: sparse aggregation}. For the Oracle inequality and the pre-specified parameters $c$ and $\phi$, we refer the reader to Appendix~\ref{sub appendix: prop} for model details. Generally, the final aggregated estimator $\hat{\beta}_{a}$ will only select two best-performed candidates from the candidate set $\mcF$. This guarantees some of the incorrectly constructed $\hat{\mcS}$ are not involved in building $\hat{\beta}_{a}$ and thus alleviate the negative transfer sources.
\begin{remark}
    In \cite{gaiffas2011hyper}, the authors indicated $\hat{\beta}_{a}$ has an explicit solution with the form 
    \begin{equation*}
        \hat{\beta}_{a} = \hat{t} \hat{\beta}_{1} + (1 - \hat{t}) \hat{\beta}_{2}
    \end{equation*}
    with $\hat{\beta}_{1}$ and $\hat{\beta}_{2}$ in belongs to $\mcF$ and $\hat{t}$ has an analytical form. 
\end{remark}

\section{Proofs for Section~\ref{sec: theoretical section for FLR}}\label{apd: theoretical section for FLR}

\subsection{Proof of Upper Bound for TL-FLR (Theorem~\ref{thm: upper bound for TL-FLR})}
\begin{proof}
We first prove the upper bound under Assumption~\ref{assump: eigen assumption} condition~\ref{commute assumption} and defer the proof under condition~\ref{HS assumption} at the end. WLOG, we assume the eigenfunction of $L_{\Gamma^{(0)}}$ and $L_{\Gamma^{(t)}}$ are perfectly aligned, i.e. $\phi_{j}^{(0)}=\phi_{j}^{(t)}$ for all $j\in \mathbb{N}$. We also recall that we set all the intercepts $\alpha^{(t)}=0$ since $\alpha^{(t)}$ will not affect the convergence rate of estimating $\beta^{(t)}$ \citep{du2014penalized}.

Let $L^{2} = \{ f: \mcT \rightarrow \mathbb{R} : \|f\|_{L^{2}} < \infty \}$ represent the he set of all square integrable functions over $\mcT$.
Since $L_{K^{\frac{1}{2}}}(L^2)= \mcH_{K}$, for any $\beta\in\mcH_{K}$, there exist a $f\in L^2$ such that $\beta = L_{K^{\frac{1}{2}}}(f)$. In following proofs, we denote $f^{(t)}$ as $\beta^{(t)}$'s corresponding element in $L^{2}$. Therefore, we can rewrite the minimization problem in the transfer step and the calibration step as
\begin{equation*}
    \hat{f}_{\mcS \lambda_{1}} = \underset{ f \in L^2}{\operatorname{argmin}}\left\{\frac{1}{n_{\mcS}}\sum_{t\in\mcS}\sum_{i=1}^{n_{t}}\left(Y_{i}^{(t)}- \langle X_i^{(t)}, L_{K^{\frac{1}{2}}}(f) \rangle \right)^{2}+\lambda_{1}\|f\|_{L^2}^{2}\right\},
\end{equation*}
where $n_{\mcS} = \sum_{t\in\mcS}n_{t}$ and 
\begin{equation*}
    \hat{f}_{\delta \lambda_{2}} = \underset{ f_{\delta} \in L^2}{\operatorname{argmin}}\left\{\frac{1}{n_0}\sum_{i=1}^{n_0}\left(Y_{i}^{(0)}- \langle X_i^{(0)}, L_{K^{\frac{1}{2}}}(\hat{f}_{\mcS} + f_{\delta}) \rangle \right)^{2}+\lambda_{2}\|f_{\delta}\|_{L^2}^{2}\right\}.
\end{equation*}
Thus the excess risk of $\hat{\beta}$ can be rewritten as 
\begin{equation*}
    \mathcal{E}(\hat{\beta}) = \left\| (L_{\Gamma^{(0)}})^{\frac{1}{2}} (\hat{f} - f_{0}) \right\|_{L^{2}}^{2} \quad \text{where} \quad \hat{f} = \hat{f}_{\mcS \lambda_{1}} + \hat{f}_{\delta \lambda_{1}}
\end{equation*}
Define the empirical version of $C^{(t)}$ as
\begin{equation*}
    C_{n}^{(t)}(s,t) = \frac{1}{n_{t}} \sum_{i=1}^{n_{t}}X_{i}^{(t)}(s) X_{i}^{(t)}(t),
\end{equation*}
and let
\begin{equation*}
    L_{\Gamma_{n}^{(t)}} = L_{K^{\frac{1}{2}}}L_{C_{n}^{(t)}}L_{K^{\frac{1}{2}}}.
\end{equation*}
To bound the excess risk $\mathcal{E}(\hat{\beta})$, by triangle inequality, 
\begin{equation*}
    \left\| (L_{\Gamma^{(0)}})^{\frac{1}{2}} (\hat{f} - f_{0}) \right\|_{L^{2}} \leq \left\| (L_{\Gamma^{(0)}})^{\frac{1}{2}} (\hat{f}_{\mcS} - f_{\mcS}) \right\|_{L^{2}} + \left\| (L_{\Gamma^{(0)}})^{\frac{1}{2}} (\hat{f}_{\delta} - f_{\delta}) \right\|_{L^{2}}
\end{equation*}
where each term at the r.h.s. corresponds to the excess risk from the transfer and calibration steps, respectively. 

\paragraph{Transfer Step.}
For the transfer step, the solution of minimization is
\begin{equation*}
    \hat{f}_{\mcS \lambda_1} = \left( \sum_{t\in \mcS}\alpha_{t} \Gnl + \lambda_{1} \mathbf{I} \right)^{-1}\left( \sum_{t\in \mcS}\alpha_{t} \Gnl (f^{(t)}) + \sum_{t\in\mcS}g_{n}^{(t)} \right),
\end{equation*}
where $\mathbf{I}$ is identity operator, $\alpha_{t} = \frac{n_{t}}{n_{\mcS}}$ and 
\begin{equation*}
    g_{n}^{(t)} = \frac{1}{n_{\mcS}}  \sum_{i=1}^{n_{t}}\epsilon_{i}^{(t)}L_{K^{\frac{1}{2}}}(X_{i}^{(t)}).
\end{equation*}
Besides, the solution of the transfer learning step, $\hat{f}_{\mcS \lambda_{1}}$, converges to its population version, which is defined by the following moment condition
\begin{equation*}
    \sum_{t\in \mcS} \alpha_{t} \E  \left\{ L_{K^{\frac{1}{2}}}(X^{(t)}) \left( Y^{(t)} - \langle L_{K^{\frac{1}{2}}}(X^{(t)}), f_{\mcS} \rangle_{L^{2}} \right) \right\} = 0,
\end{equation*}
and therefore leads to the explicit form of $f_{\mcS}$ as 
\begin{equation*}
    \left( \sum_{t\in \mcS}\alpha_{t} \Gl \right) f_{\mcS}  = \sum_{t\in \mcS} \alpha_{t} \Gl \left(f^{(t)} \right).
\end{equation*}
Define 
\begin{equation*}
    f_{\mcS \lambda_1} = \left( \sum_{t\in \mcS}\alpha_{t} \Gl + \lambda_{1} \mathbf{I} \right)^{-1}\left( \sum_{t\in \mcS}\alpha_{t} \Gl (f^{(t)})  \right).
\end{equation*}
By the triangle inequality
\begin{equation*}
    \left\| (L_{\Gamma^{(0)}})^{\frac{1}{2}} (\hat{f}_{\mcS \lambda_1} - f_{\mcS})\right\|_{L^2} \leq \underbrace{\left\|(L_{\Gamma^{(0)}})^{\frac{1}{2}}(\hat{f}_{\mcS \lambda_1} - f_{\mcS \lambda_1})\right\|_{L^2}}_{\text{estimation error}} + \underbrace{\left\|(L_{\Gamma^{(0)}})^{\frac{1}{2}}(f_{\mcS \lambda_1} - f_{\mcS})\right\|_{L^2}}_{\text{approximation error}}.
\end{equation*}
For approximation error, by Lemma~\ref{lemma1} and taking $v = \frac{1}{2}$, the second term on r.h.s. can be bounded by 
\begin{equation*}
    \left\|(L_{\Gamma^{(0)}})^{\frac{1}{2}}(f_{\mcS \lambda_1} - f_{\mcS})\right\|_{L^2}^{2}
    =O_{\mbP} \left( \lambda_{1}    \left\|f_{\mcS}\right\|_{L^2}^{2} \right).
\end{equation*}
Now we turn to the estimation error. We further introduce an intermediate term
\begin{equation*}
    \tilde{f}_{\mcS \lambda_1} =  f_{\mcS \lambda_1} + \left( \sum_{t\in \mcS} \alpha_{t} L_{\Gamma^{(t)}} + \lambda_1 \mathbf{I} \right)^{-1} \left( \sum_{t\in \mcS} \alpha_{t} \Gnl (f_{\mcS} - f_{\mcS \lambda_1} ) + \sum_{t\in\mcS}g_{n}^{(t)} - \lambda_1 f_{\mcS \lambda_1} \right).
\end{equation*}
We first bound $\| (L_{\Gamma^{(0)}})^{\frac{1}{2}} (f_{\mcS \lambda_1} - \tilde{f}_{\mcS \lambda_1}) \|_{L^2}^{2}$. Based on he fact that 
\begin{equation*}
    \lambda_1 f_{\mcS \lambda_1} = \sum_{t\in \mcS} \alpha_{t} \Gl \left( f^{(t)} - f_{\mcS \lambda_1} \right) = \sum_{t\in \mcS} \alpha_{t} \Gl \left( f_{\mcS} - f_{\mcS \lambda_1} \right),
\end{equation*}
we have
\begin{equation*}
\begin{aligned}
    & \| (L_{\Gamma^{(0)}})^{\frac{1}{2}} (f_{\mcS \lambda_1} - \tilde{f}_{\mcS \lambda_1}) \|_{L^2} \\
    & \quad \leq \left\| (L_{\Gamma^{(0)}})^{\frac{1}{2}} ( \sum_{t\in \mcS} \alpha_{t} L_{\Gamma^{(t)}} + \lambda_1 \mathbf{I} )^{-1} ( \sum_{t\in \mcS}g_{n}^{(t)}   ) \right\|_{L^{2}} + \\
    & \quad \qquad \left\| (L_{\Gamma^{(0)}})^{\frac{1}{2}} ( \sum_{t\in \mcS} \alpha_{t} L_{\Gamma^{(t)}} + \lambda_1 \mathbf{I} )^{-1} ( \sum_{t\in \mcS} \alpha_{t} (\Gnl - \Gl) (f_{\mcS} - f_{\mcS \lambda_1} )  ) \right\|_{L^2} \\
    & \quad =  \left\{\sum_{j=1}^{\infty}  \left( \left\langle (L_{\Gamma^{(0)}})^{\frac{1}{2}} ( \sum_{t\in \mcS} \alpha_{t} L_{\Gamma^{(t)}} + \lambda_1 \mathbf{I} )^{-1} \left( \sum_{t\in \mcS}g_{n}^{(t)}   \right) , \phi_{j}^{(0)} )  \right\rangle_{L^2} \right)^2\right\}^{\frac{1}{2}} + \\
    & \quad \qquad \left\{ \sum_{j=1}^{\infty}  \left( \left\langle (L_{\Gamma^{(0)}})^{\frac{1}{2}} ( \sum_{t\in \mcS} \alpha_{t} L_{\Gamma^{(t)}} + \lambda_1 \mathbf{I} )^{-1} \left( \sum_{t\in \mcS} \alpha_{t} (\Gnl - \Gl) (f_{\mcS} - f_{\mcS \lambda_1} )  \right) , \phi_{j}^{(0)} )  \right\rangle_{L^2} \right)^2 \right\}^{\frac{1}{2}}
\end{aligned}
\end{equation*}
For the first term in the above inequality, by Lemma~\ref{lemma3}, 
\begin{equation*}
    \left\{  \sum_{j=1}^{\infty}  \left( \left\langle (L_{\Gamma^{(0)}})^{\frac{1}{2}} ( \sum_{t\in \mcS} \alpha_{t} L_{\Gamma^{(t)}} + \lambda_1 \mathbf{I} )^{-1} \left( \sum_{t\in \mcS}g_{n}^{(t)}  \right) , \phi_{j}^{(0)} )  \right\rangle_{L^2} \right)^2 \right\} =O_{\mbP} \left( \sigma^{2} \left(  n_{\mcS}  \right)^{-1} \lambda_{1}^{\frac{1}{2r}} \right).
\end{equation*}
For second one, by Lemma~\ref{lemma2} and~\ref{lemma4},
\begin{equation*}
    \sum_{j=1}^{\infty}  \left( \left\langle (L_{\Gamma^{(0)}})^{\frac{1}{2}} ( \sum_{t\in \mcS} \alpha_{t} L_{\Gamma^{(t)}} + \lambda_1 \mathbf{I} )^{-1} \left( \sum_{t\in \mcS} \alpha_{t} (\Gnl - \Gl) (f_{\mcS} - f_{\mcS \lambda_1} )  \right) , \phi_{j}^{(0)} )  \right\rangle_{L^2} \right)^2  = O_{\mbP} \left( \left( n_{\mcS} \right)^{-1} \lambda_{1}^{\frac{1}{2r}} \right).
\end{equation*}
Therefore, 
\begin{equation*}
    \| (L_{\Gamma^{(0)}})^{\frac{1}{2}} (f_{\mcS \lambda_1} - \tilde{f}_{\mcS \lambda_1}) \|_{L^2}^{2} =O_{\mbP}\left( \|f_{\mcS}\|_{L^{2}}^2 \left( n_{\mcS}  \right)^{-1} \lambda_{1}^{\frac{1}{2r}} \right) .
\end{equation*}
Finally, we bound $\| (L_{\Gamma^{(0)}})^{\frac{1}{2}} (\hat{f}_{\mcS \lambda_1} - \tilde{f}_{\mcS \lambda_1}) \|_{L^2}^{2}$. Once again, by the definition of $\tilde{f}_{\mcS \lambda_1}$
\begin{equation*}
    \hat{f}_{\mcS \lambda_1} - \tilde{f}_{\mcS \lambda_1} = \left( \sum_{t\in \mcS} \alpha_{t} \Gl + \lambda_{1} \mathbf{I} \right)^{-1} \left( \sum_{t\in \mcS} \alpha_{t} (\Gl - \Gnl)( \hat{f}_{\mcS \lambda_1} - f_{\mcS \lambda_1}) \right).
\end{equation*}
Thus, by Lemma~\ref{lemma2}
\begin{equation*}
\begin{aligned}
    & \left\| (L_{\Gamma^{(0)}})^{\frac{1}{2}} (\hat{f}_{\mcS \lambda_1} - \tilde{f}_{\mcS \lambda_1}) \right\|_{L^2}^{2} \\
    & \qquad \leq  \left\| (\Gt)^{\frac{1}{2}} ( \sum_{t\in \mcS} \alpha_{t} \Gl + \lambda_1 \mathbf{I} )^{-1} ( \sum_{t\in \mcS} \alpha_{t} \left(\Gl - \Gnl \right) ) (\Gt)^{-\frac{1}{2}} \right\|_{op}^{2} \left\|(L_{\Gamma^{(0)}})^{\frac{1}{2}}( \hat{f}_{\mcS \lambda_1} - f_{\mcS \lambda_1}) \right \|_{L^2}^{2} \\
    & \qquad = O_{\mbP} \left( n_{\mcS}^{-1} \lambda_{1}^{\frac{1}{2r}}  \left \|(L_{\Gamma^{(0)}})^{\frac{1}{2}}( \hat{f}_{\mcS \lambda_1} - f_{\mcS \lambda_1}) \right \|_{L^2}^{2} \right) \\
    & \qquad = o_{\mbP} \left( \left \|(L_{\Gamma^{(0)}})^{\frac{1}{2}}( \hat{f}_{\mcS \lambda_1} - f_{\mcS \lambda_1}) \right \|_{L^2}^{2} \right).
\end{aligned}
\end{equation*}
Combining three parts, we get 
\begin{equation*}
    \left\| (L_{\Gamma^{(0)}})^{\frac{1}{2}} \left( \hat{f}_{\mcS \lambda_1} - f_{\mcS } \right) \right\|_{L^2}^{2} = O_{\mbP}\left( \|f_{\mcS}\|_{L^{2}}^2 \lambda_1  + \left(\sigma^2 + \|f_{\mcS}\|_{L^{2}}^2\right) n_{\mcS}^{-1} \lambda_{1}^{-\frac{1}{2r}} \right)  ,
\end{equation*}
by taking $\lambda_{1} \asymp ( n_{\mcS})^{-\frac{2r}{2r+1}}$ and notice the fact that $\frac{\sigma^2}{\|\beta_{\mcS}\|_{K}^2}$ is bounded above (This is a reasonable condition since the signal-to-noise ratio can't be 0, otherwise one can hardly learn anything from the data), we have
\begin{equation*}
    \left\| (L_{\Gamma^{(0)}})^{\frac{1}{2}} \left( \hat{f}_{\mcS \lambda_1} - f_{\mcS } \right) \right\|_{L^2}^{2} = O_{\mbP}\left(  \|f_{\mcS}\|_{L^{2}}^2 n_{\mcS}^{-\frac{2r}{2r+1}} \right).
\end{equation*}

The desired upper bound is proved given the fact that $\frac{\sigma^2}{\|\beta_{\mcS}\|_{K}^2}$ and $\frac{\sigma^2}{h^2}$ is bounded above.

\paragraph{Calibration Step.}
The estimation scheme in the calibration step is in the same form as the transfer step, and thus its proof follows the same logic as the transfer step. The solution to the minimization problem in the calibration step is 
\begin{equation*}
    \hat{f}_{\delta \lambda_2} = \left( L_{\Gamma_{n}^{(0)}} + \lambda_2 \mathbf{I} \right)^{-1} \left( L_{\Gamma_{n}^{(0)}} (f_{\mcS }-\hat{f}_{\mcS \lambda_1} + f_{\delta}) + g_{n}^{(0)} \right),
\end{equation*}
where 
\begin{equation*}
    g_{n}^{(0)} = \frac{1}{n_0}\sum_{i=1}^{n_0}\epsilon_{i}^{(0)}L_{K^{\frac{1}{2}}}(X_{i}^{(0)}).
\end{equation*}
Similarly, define
\begin{equation*}
    f_{\delta \lambda_2} = \left( L_{\Gamma^{(0)}} + \lambda_2 \mathbf{I} \right)^{-1} \left( L_{\Gamma^{(0)}}(f_{\mcS}-\hat{f}_{\mcS} + f_{\delta}) \right),
\end{equation*}
where $f_{\delta}$ is the population version of the estimator, i.e. $\hat{f}_{\delta}$.
\begin{equation*}
    \left( \sum_{t\in \mcS}\alpha_{t} \Gl \right) f_{\delta} =  \left( \sum_{t\in \mcS} \alpha_{t} \Gl\left( f_{\delta}^{(t)} \right) \right)
\end{equation*}
By the triangle inequality,
\begin{equation*}
    \left\|(L_{\Gamma^{(0)}})^{\frac{1}{2}} (\hat{f}_{\delta}-f_{\delta})  \right\|_{L^2} \leq \left\|(L_{\Gamma^{(0)}})^{\frac{1}{2}} (\hat{f}_{\delta}-f_{\delta\lambda_2})  \right\|_{L^2} + 
    \left\|(L_{\Gamma^{(0)}})^{\frac{1}{2}} (f_{\delta \lambda_2}-f_{\delta})  \right\|_{L^2}.
\end{equation*}
For the second term in r.h.s., 
\begin{equation*}
    \begin{aligned}
    \left\| (L_{\Gamma^{(0)}})^{\frac{1}{2}} (f_{\delta\lambda_2}-f_{\delta})  \right\|_{L^2}
    & \leq \left\|(L_{\Gamma^{(0)}})^{\frac{1}{2}}\left(L_{\Gamma^{(0)}} + \lambda_2\mathbf{I} \right)^{-1}L_{\Gamma^{(0)}} \left( f_{\mcS } - \hat{f}_{\mcS \lambda_1} \right)  \right\|_{L^2} + \left\|(L_{\Gamma^{(0)}})^{\frac{1}{2}} \left( f_{\delta\lambda_2}^{*} - f_{\delta} \right) \right\|_{L^2}\\
    & \leq \left\|(L_{\Gamma^{(0)}})^{\frac{1}{2}}\left(L_{\Gamma^{(0)}} + \lambda_2\mathbf{I} \right)^{-1}(L_{\Gamma^{(0)}})^{\frac{1}{2}}\right\|_{op}\left\|L_{\Gamma^{(0)}}^{\frac{1}{2}} \left( f_{\mcS} - \hat{f}_{\mcS} \right)  \right\|_{L^2}\\ 
    & \qquad  + \left\|(L_{\Gamma^{(0)}})^{\frac{1}{2}} \left( f_{\delta\lambda_{2}}^{*} - f_{\delta} \right) \right\|_{L^2},
    \end{aligned}
\end{equation*}
where $f_{\delta\lambda_{2}}^{*} = \left(L_{\Gamma^{(0)}}+\lambda_2 \mathbf{I}\right)^{-1}L_{\Gamma^{(0)}}(f_{\delta})$.

By Lemma~\ref{lemma1} with $\mcS = \emptyset$,
\begin{equation*}
\begin{aligned}
    \left\|(L_{\Gamma^{(0)}})^{\frac{1}{2}} \left( f_{\delta\lambda_2}^{*} - f_{\delta} \right) \right\|_{L^2}^2 
    & \leq \frac{\lambda_2 }{4}\left\|f_{\delta}\right\|_{L^2}^{2}\\
    & \lesssim \lambda_2  h^{2},
\end{aligned}
\end{equation*}
where the second inequality holds with the fact the $\mcS = \left\{ 1\leq l\leq L: \|f_{0} - f^{(t)}\|_{L^2}\leq h \right\}$. Therefore, 
\begin{equation*}
    \left\|(L_{\Gamma^{(0)}})^{\frac{1}{2}} (f_{\delta\lambda_2}-f_{\delta})  \right\|_{L^2}^{2}  = O_{\mbP} \left( n_{\mcS}^{-\frac{2r}{2r+1}} + \lambda_2  h^2  \right).
\end{equation*}
For the first term, we play the same game as \textbf{transfer step}. Define 
\begin{equation*}
    \tilde{f}_{\delta \lambda_2} = f_{\delta \lambda_2} +  \left( L_{\Gamma^{(0)}} + \lambda_2 \mathbf{I}\right)^{-1} \left( L_{\Gamma_{n}^{(0)}} (f_{\mcS } - \hat{f}_{\mcS } + f_{\delta} ) + g_{n}^{(0)} - L_{\Gamma_{n}^{(0)}}(f_{\delta \lambda_2}) - \lambda_2 f_{\delta \lambda_2} \right),
\end{equation*}
and the definition of $f_{\delta \lambda_2}$ leads to 
\begin{equation*}
    \tilde{f}_{\delta \lambda_2} - f_{\delta \lambda_2} = \left( L_{\Gamma^{(0)}} + \lambda_2 \mathbf{I}\right)^{-1} \left( (L_{\Gamma_{n}^{(0)}} - L_{\Gamma^{(0)}})( f_{\mcS } - \hat{f}_{\mcS } + f_{\delta } - f_{\delta \lambda_2}  ) + g_{n}^{(0)} \right).
\end{equation*}
Therefore, 
\begin{equation*}
\begin{aligned}
    \left\| (L_{\Gamma^{(0)}})^{\frac{1}{2}} \left( \tilde{f}_{\delta \lambda_2} - f_{\delta \lambda_2} \right) \right\|_{L^2} & \leq \left\| (L_{\Gamma^{(0)}})^{\frac{1}{2}}(L_{\Gamma^{(0)}} + \lambda_2 \mathbf{I})^{-1} g_{n}^{(0)} \right\|_{L^2} +  \\
     & \qquad  \left\| (L_{\Gamma^{(0)}})^{\frac{1}{2}} \left( L_{\Gamma^{(0)}} + \lambda_2 \mathbf{I}\right)^{-1}  (L_{\Gamma_{n}^{(0)}} - L_{\Gamma^{(0)}})  (L_{\Gamma^{(0)}})^{-\frac{1}{2}} \right\|_{op} \cdot \\
     & \qquad \left\{ \left\| (L_{\Gamma^{(0)}})^{\frac{1}{2}} \left( f_{\mcS } - \hat{f}_{\mcS } + f_{\delta} - f_{\delta \lambda_2} \right) \right\|_{L^2}   \right\} 
\end{aligned}
\end{equation*}
leading to 
\begin{equation*}
    \left\| (L_{\Gamma^{(0)}})^{\frac{1}{2}} \left( \tilde{f}_{\delta \lambda_2} - f_{\delta \lambda_2} \right) \right\|_{L^2}  =O_{\mbP} \left( \left( n_0 \lambda_2^{\frac{1}{2r}} \right)^{-1} + \left( n_0 \lambda_2^{\frac{1}{2r}} \right)^{-1} \left[ n_{\mcS}^{-\frac{2r}{2r+1}} + \lambda_2 h^{2} \right] \right)  ,
\end{equation*}
where the first term and operator norm comes from Lemma~\ref{lemma2} and~\ref{lemma3} with $\mcS = \emptyset$, and bounds on $\| (L_{\Gamma^{(0)}})^{\frac{1}{2}} ( f_{\mcS } - \hat{f}_{\mcS } + f_{\delta} - f_{\delta \lambda_2} ) \|_{L^2}^{2}$ comes from \textbf{transfer step} and bias term of \textbf{calibration step}.

Finally, for $\|(L_{\Gamma^{(0)}})^{\frac{1}{2}}(\hat{f}_{\delta \lambda_2} - \tilde{f}_{\delta \lambda_2})\|_{L^2}^{2}$, notice that 
\begin{equation*}
    \hat{f}_{\delta \lambda_2} - \tilde{f}_{\delta \lambda_2} = \left( L_{\Gamma^{(0)}} + \lambda_2 \mathbf{I}\right)^{-1} \left( (L_{\Gamma^{(0)}} - L_{\Gamma_{n}^{(0)}}) (\tilde{f}_{\delta \lambda_2} - f_{\delta \lambda_2})  \right),
\end{equation*}
thus by Lemma~\ref{lemma2},
\begin{equation*}
    \left\|(L_{\Gamma^{(0)}})^{\frac{1}{2}}(\hat{f}_{\delta \lambda_2} - \tilde{f}_{\delta \lambda_2})\right\|_{L^2}^{2} = o_{\mbP} \left( \|(L_{\Gamma^{(0)}})^{\frac{1}{2}}(\tilde{f}_{\delta \lambda_2} - f_{\delta \lambda_2})\|_{L^2}^{2} \right).
\end{equation*}
Combine three parts, we get 
\begin{equation*}
    \left\| (L_{\Gamma^{(0)}})^{\frac{1}{2}} \left( \hat{f}_{\delta \lambda_2} - f_{\delta} \right) \right\|_{L^2}^{2} = O_{\mbP} \left( \lambda_2 h^{2} + \left(h^2 + \sigma^2\right) (n_0 \lambda_{2}^{\frac{1}{2r}})^{-1} \right) ,
\end{equation*}
taking $\lambda_2 \asymp n_0^{-\frac{2r}{2r+1}}$ and notice the fact that $\frac{\sigma^2}{h^2}$ is bounded above (similar reasoning as the transfer step), we have
\begin{equation*}
    \left\| (L_{\Gamma^{(0)}})^{\frac{1}{2}} \left( \hat{f}_{\delta \lambda_2} - f_{\delta} \right) \right\|_{L^2}^{2} = O_{\mbP} \left( h^2  n_0^{-\frac{2r}{2r+1}}  \right) .
\end{equation*}
Combining the results from \textbf{transfer step} and \textbf{calibration step}, and reorganizing the constants for each term, we have
\begin{equation*}
    \mcE (\hat{\beta}) =  \left\| (L_{\Gamma^{(0)}})^{\frac{1}{2}} (\hat{f} - f_{0}) \right\|_{L^{2}} = O_{\mbP} \left(  n_{\mcS}^{-\frac{2r}{2r+1}} + n_0^{-\frac{2r}{2r+1}} \xi(h,\mcS)   \right).
\end{equation*}
To prove the same upper bound under Assumption 2, we only need to show that Lemmas~\ref {lemma1} to~\ref {Tony's lemma6} still hold under Assumption 2. Let $\{(s_j^{(0)},\phi_{j}^{(0)})\}_{j\geq 1}$ be the eigen-pairs of $L_{\Gamma^{(0)}}$. We show that 
\begin{equation}\label{asymptotic eigenvalue}
    \left \langle L_{\Gamma^{(t)}} (\phi_{j}^{(0)}), \phi_{j}^{(0)}  \right\rangle_{L^2} = s_{j}^{(0)} (1 + o(1)).
\end{equation}
Consider
\begin{equation*}
\begin{aligned}
    \left| \left \langle (L_{\Gamma^{(t)}} - L_{\Gamma^{(0)}} ) \phi_{j}^{(0)}, \phi_{j}^{(0)}  \right\rangle_{L^2} \right| & = \left| \left \langle (L_{\Gamma^{(0)}})^{\frac{1}{2}}\left( (L_{\Gamma^{(0)}})^{-\frac{1}{2}} L_{\Gamma^{(t)}}(L_{\Gamma^{(0)}})^{-\frac{1}{2}} - \mathbf{I} \right)(L_{\Gamma^{(0)}})^{\frac{1}{2}} \phi_{j}^{(0)}, \phi_{j}^{(0)}  \right\rangle_{L^2} \right| \\
    & = \lambda_j^{(0)} \left|  \left \langle \left( (L_{\Gamma^{(0)}})^{-\frac{1}{2}} L_{\Gamma^{(t)}}(L_{\Gamma^{(0)}})^{-\frac{1}{2}} - \mathbf{I} \right) \phi_{j}^{(0)}, \phi_{j}^{(0)}  \right\rangle_{L^2} \right|.
\end{aligned}
\end{equation*}
Since $(L_{\Gamma^{(0)}})^{-\frac{1}{2}} L_{\Gamma^{(t)}}(L_{\Gamma^{(0)}})^{-\frac{1}{2}} - \mathbf{I}$ is Hilbert–Schmidt, then 
\begin{equation*}
    \left\|   (L_{\Gamma^{(0)}})^{-\frac{1}{2}} L_{\Gamma^{(t)}}(L_{\Gamma^{(0)}})^{-\frac{1}{2}} - \mathbf{I} \right\|_{HS}^{2} = \sum_{i,j} \left| \langle \phi_{i}^{(0)}, \left( (L_{\Gamma^{(0)}})^{-\frac{1}{2}} L_{\Gamma^{(t)}}(L_{\Gamma^{(0)}})^{-\frac{1}{2}} - \mathbf{I} \right) \phi_{j}^{(0)} \rangle_{L^2} \right|^2 < \infty
\end{equation*}
which leads to
\begin{equation*}
\left|  \left \langle \left( (L_{\Gamma^{(0)}})^{-\frac{1}{2}} L_{\Gamma^{(t)}}(L_{\Gamma^{(0)}})^{-\frac{1}{2}} - \mathbf{I} \right) \phi_{j}^{(0)}, \phi_{j}^{(0)}  \right\rangle_{L^2} \right| = o(1) \quad \textit{as} \quad j\to \infty.
\end{equation*}
Therefore, Equation (\ref{asymptotic eigenvalue}) holds. One can now replace the common eigenfunctions $\phi_j$ by $\phi_j^{(0)}$ in the proofs of Lemma~\ref{lemma1} to Lemma~\ref{Tony's lemma6}, and it is not hard to check that the results still hold.
\end{proof}

\subsection{Proof of Lower Bound for TL-FLR (Theorem~\ref{thm: lower bound for TL-FLR})}
In this part, we prove the alternative version for the lower bound, i.e. 
\begin{equation*}
    \lim_{a\rightarrow 0} \lim_{n \rightarrow \infty}   \inf_{\tilde{\beta}} \sup_{\Theta(h) } P\left\{ \mcE\left( \tilde{\beta} \right) \geq  a \left( (n_{\mcS} + n_{0})^{-\frac{2r}{2r+1}} + n_{0}^{-\frac{2r}{2r+1}} \xi(h,\mcS)  \right) \right\} = 1.
\end{equation*}
This alternative form is also proven in other contexts, like high-dimensional linear regression or GLM, to show optimality. However, the upper bound we derive for TL-FLR can still be sharp since in the TL regime, it is always assumed $n_{\mcS} \gg n_{0}$, and leads to $(n_{\mcS} + n_{0})^{-\frac{2r}{2r+1}} \asymp n_{\mcS}^{-\frac{2r}{2r+1}}$. 

On the other hand, one can modify the transfer step in TL-FLR by including the target dataset $\mcD^{(0)}$ to estimate $\beta_{\mcS}$, which produces an alternative upper bound $(n_{\mcS} + n_{0})^{-\frac{2r}{2r+1}} + n_{0}^{-\frac{2r}{2r+1}} \xi(h,\mcS)$, and mathematically aligns with the alternative lower bound we mention above. However, we would like to note that such a modified TL-FLR is not computationally efficient for transfer learning, since for each new upcoming target task, TL-FLR needs to recalculate a new $\hat{\beta}_{\mcS}$ with the massive datasets $\{ \mcD^{(t)}: t\in 0\cup\mcS \}$.

\begin{proof}
Note that any lower bound for a specific case will immediately yield a lower bound for the general case. Therefore, we consider the following two cases.

(1) Consider $h=0$, i.e.
\begin{equation*}
    y_{i}^{(t)} = \left\langle X_{i}^{(t)}, \beta \right\rangle + \epsilon_{i}^{(t)},\quad \forall  t\in \left\{0 \right\}\cup \mcS.
\end{equation*}
In this case, all the source model shares the same coefficient function as the target model, i.e. $\beta^{(t)} = \beta^{(0)}$ for all $t\in \mcS$. Therefore, the estimation process is equivalent to estimating $\beta$ under the target model with sample size equal to $n_{\mcS}$. The lower bound in \citet{cai2012minimax} can be applied here directly and leads to 
\begin{equation*}
    \lim_{a\rightarrow 0} \lim_{n \rightarrow \infty} \inf_{\tilde{\beta}} \sup_{ \Theta(h) } P\left\{ \mcE\left( \tilde{\beta} \right) \geq a  \left(n_{\mcS} + n_{0}\right)^{-\frac{2r}{2r+1}}   \right\} = 1.
\end{equation*}

(2) Consider $\beta^{(0)} \in \mathcal{B}_{\mathcal{H}}(h)$ where $\mathcal{B}_{\mathcal{H}}(h)$ is a ball in RKHS centered at $0$ with radius $h$, and $\beta^{(t)} = 0$ for all $t\in \left\{0 \right\}\cup \mcS$ and $\sigma\geq h$. That is all the source datasets contain no information about $\beta^{(0)}$. Consider slope functions $\beta_1,\cdots,\beta_M \in \mathcal{B}_{\mathcal{H}}(h)$ and $P_{1},\cdots,P_{M}$ as the probability distribution of $\{ (X_{i}^{(0)},Y_{i}^{(0)}) : i = 1,\cdots,n_0 \}$ under $\beta_1,\cdots,\beta_M$. Then the KL divergence between $P_i$ and $P_j$ is 
\begin{equation*}
    KL\left( P_{i} | P_{j} \right) = \frac{n_0}{2\sigma^2} \left\| L_{(C^{(0)})^{\frac{1}{2}}} \left( \beta_{i} - \beta_{j} \right)  \right\|_{K}^{2} \quad \textit{for} \quad i,j \in \{1,\cdots,K\}.
\end{equation*}
Let $\tilde{\beta}$ be any estimator based on $\{ (X_{i}^{(0)},Y_{i}^{(0)}) : i = 1,\cdots,n_0 \}$ and consider testing multiple hypotheses, by Markov's inequality and Lemma~\ref{lemmaFano}
\begin{equation} \label{fanolowerbound}
\begin{aligned}
    & \left\| L_{(C^{(0)})^{\frac{1}{2}}} \left( \tilde{\beta} - \beta_{i} \right)  \right\|_{K}^{2}  \geq P_{i} \left( \tilde{\beta} \neq \beta_{i} \right) \min_{i,j} \left\| L_{(C^{(0)})^{\frac{1}{2}}} \left( \beta_{i} - \beta_{j} \right)  \right\|_{K}^{2}\\
    & \qquad \qquad \geq \left( 1 - \frac{\frac{n_0}{2\sigma^{2}} \max_{i,j} \left\| L_{(C^{(0)})^{\frac{1}{2}}} \left( \beta_{i} - \beta_{j} \right)  \right\|_{K}^{2} + log(2)}{log(M-1)} \right)\min_{i,j} \left\| L_{(C^{(0)})^{\frac{1}{2}}} \left( \beta_{i} - \beta_{j} \right)  \right\|_{K}^{2}.
\end{aligned}
\end{equation}
Our target is to construct a sequence of $\beta_1, \cdots, \beta_M\in \mathcal{B}_{\mathcal{H}}(h)$ such that the above lower bound matches with the upper bound. We consider the Varshamov-Gilbert bound in \citet{varshamov1957estimate}, which we state as Lemma~\ref{lemmaVG}. Now we define,
\begin{equation*}
    \beta_i = \sum_{k=N+1}^{2N} \frac{b_{i,k-N} h }{\sqrt{N}} L_{K^{\frac{1}{2}}} (\phi_k) \quad \textit{for} \quad i = 1,2,\cdots,M.
\end{equation*}
where $(b_{i,1}, b_{i,2}, \cdots, b_{i,N}) \in {0,1}^{N}$. Then, 
\begin{equation*}
    \left\| \beta_{i} \right\|_{K}^{2} = \sum_{k = N+1}^{2N} \frac{b_{i,k-N}^2 h^2}{N} \left\| L_{K^{\frac{1}{2}}}(\phi_{k} ) \right\|_{K}^{2} \leq h^2,
\end{equation*}
hence $\beta_{\theta} \in \mathcal{B}_{\mcH_{K}}(h)$. Besides,
\begin{equation*}
\begin{aligned}
    \left\| L_{(C^{(0)})^{\frac{1}{2}}} \left( \beta_{i} - \beta_{j} \right)  \right\|_{K}^{2} & = \frac{h^2}{N} \sum_{k=N+1}^{2N} \left(b_{i,k-N} - b_{j,k-N} \right)^2 s_l^{(0)} \\
    & \geq \frac{h^{2} s_{2N}^{(0)} }{N} \sum_{k=N+1}^{2N} \left(b_{i,k-N} - b_{j,k-N} \right)^2 \\
    & \geq \frac{h^{2} s_{2N}^{(0)}}{4}  ,
\end{aligned}
\end{equation*}
where the last inequality is by Lemma~\ref{lemmaVG}, and 
\begin{equation*}
\begin{aligned}
    \left\| L_{(C^{(0)})^{\frac{1}{2}}} \left( \beta_{i} - \beta_{j} \right)  \right\|_{K}^{2} & = \frac{h^2}{N} \sum_{k=N+1}^{2N} \left(b_{i,k-N} - b_{j,k-N} \right)^2 s_{k}^{(0)} \\
    & \leq \frac{h^{2} s_{N}^{(0)} }{N} \sum_{k=M+1}^{M} \left(b_{i,k-N} - b_{j,k-N} \right)^2 \\
    &  \leq  h^{2} s_{N}^{(0)}.
\end{aligned}
\end{equation*}
Therefore, one can bound the KL divergence by
\begin{equation*}
    KL\left( P_{i} | P_{j} \right) \leq \max_{i,j} \left\{ \frac{n_0}{2\sigma^2} \left\| L_{(C^{(0)})^{\frac{1}{2}}} \left( \beta_{i} - \beta_{j} \right)  \right\|_{K}^{2}   \right\}.
\end{equation*}
Using the above results, the r.h.s. of Equation~\ref{fanolowerbound} becomes
\begin{equation*}
   \left( 1 - \frac{\frac{4 n_0 h^2}{\sigma^{2}} s_{N}^{(0)} + 8log(2)}{N} \right) \frac{s_{2N}^{(0)} h^2}{4}.
\end{equation*}
Taking $N = \frac{8h^2}{\sigma^2} n_{0}^{\frac{1}{2r+1}}$, which implies $N\rightarrow \infty$, would produce
\begin{equation*}
\begin{aligned}
    \left( 1 - \frac{\frac{4 n_0 h^2}{\sigma^{2}} s_{N}^{(0)} + 8log(2)}{N} \right) \frac{s_{2N}^{(0)} h^2}{4} & \asymp
    \left( \frac{1}{2} - \frac{8log(2)}{N} \right) h^2 N^{-2r} \\
    & \asymp n_0^{-\frac{2r}{2r+1}} h^2
\end{aligned}
\end{equation*}
Combining the lower bound in case (1) and case (2), we obtain the desired lower bound.
\end{proof}

\subsection{Proof of Consistency (Theorem~\ref{thm: aggregation consistency})}
\begin{proof}
Under Assumption 3, 
\begin{equation*}
    \max_{t\in \mcS} \Delta_{t} < \min_{t\in \mcS^{c}} \Delta_{t}
\end{equation*}
holds automatically. To prove 
\begin{equation*}
    \mathbb{P}\left( \hat{\mcS}_j = \mcS \right) \rightarrow 1,
\end{equation*}
we only need to show that the following fact holds 
\begin{equation*}
    \mathbb{P}\left( \max_{t\in \mcS} \hat{\Delta}_{t} < \min_{t\in \mcS^{c}} \hat{\Delta}_{t} \right) \rightarrow 1.
\end{equation*}
Observe that
\begin{equation*}
    \left\| f \right\|_{K^{M}} = \sum_{j=1}^{M} \frac{f_j^2}{v_j} \leq \frac{1}{v_M} \sum_{j=1}^{M} f_j^2 \leq \frac{1}{v_M} \left\| f \right\|_{L^{2}} \lesssim \left\| f \right\|_{L^{2}}
\end{equation*}
for any finite $M$, then by Corollary 10 in \citet{yuan2010reproducing}
\begin{equation*}
    \left\| (\hat{\beta}_0 - \hat{\beta}_{t}) - ( \beta_0 - \beta_{t} ) \right\|_{K^M} \lesssim \left\| (\hat{\beta}_0 - \hat{\beta}_{t}) - ( \beta_0 - \beta_{t} ) \right\|_{L^2} = o_{\mbP}(1).
\end{equation*}
Therefore, for $t\in \mcS^{c}$
\begin{equation*}
    \left\| \hat{\beta}_0 - \hat{\beta}_{t} \right\|_{K^M} \geq (1 - o_{\mbP}(1))   \left\| \beta_0 - \beta_{t} \right\|_{K^M}
\end{equation*}
and also for $t\in \mcS$
\begin{equation*}
        \left\| \hat{\beta}_0 - \hat{\beta}_{t} \right\|_{K^M} \leq  (1 + o_{\mbP}(1))   \left\| \beta_0 - \beta_{t} \right\|_{K^M} \leq (1 + o_{\mbP}(1))   \left\| \beta_0 - \beta_{t} \right\|_{K}
\end{equation*}
with high probability. Finally, 
\begin{equation*}
\begin{aligned}
    \mathbb{P}\left( \max_{t\in \mcS} \hat{\Delta}_{t} < \min_{t\in \mcS^{c}} \hat{\Delta}_k \right) & \geq \mathbb{P}\left( (1 + o(1)) \max_{t\in \mcS} \left\| \beta_0 - \beta_{t} \right\|_{K} <  (1 - o(1))  \min_{t\in \mcS^{c}} \left\| \beta_0 - \beta_{t} \right\|_{K^M}\right) \\ & \rightarrow 1,
\end{aligned}
\end{equation*}
where the convergence in probability is guaranteed by Assumption~\ref{assump: identifiability}.
\end{proof}

\subsection{Proof of Upper Bound for ATL-FLR (Theorem~\ref{thm: upper bound ATL-FLR})}
The Theorem directly holds by combining Theorem~\ref{thm: upper bound for TL-FLR}, Proposition~\ref{prop: oracle inequality of sparse aggregation} with setting (2), and Theorem~\ref{thm: aggregation consistency}.

\subsection{Proposition}\label{sub appendix: prop}
\begin{proposition}[\citet{gaiffas2011hyper}] \label{prop: oracle inequality of sparse aggregation}
Given a confidence level $\delta$, assume either setting (1) or (2) holds for a constant $b$,
\begin{enumerate}
    \item  $\max \{|Y^{(0)}|, \max_{\beta \in \mcH_{K}}|\langle X^{(0)},\beta \rangle_{L^2}|\} \leq b$
    \item  $\max \{ \|\epsilon^{(0)}\|_{\psi}, \sup_{\beta \in \mcH_{K}}\| \langle X^{(0)}, \beta - \beta^{(0)} \rangle \| \} \leq b$
\end{enumerate}  
where $\|\epsilon^{(0)}\|_{\psi} := inf\{c>0 : \E[ \exp\{ |\epsilon^{(0)}|/c \} ] \leq 2 \} $. Let $\sigma_{\epsilon^{(0)}}^{2}$ be the upper bound for $\E \left[(\epsilon^{(0)})^2 | X^{(0)} \right]$. The pre-specified parameter $\phi$ in Algorithm~\ref{algo: sparse aggregation} is defined as 
\begin{equation*}
    \phi = \left\{ \begin{aligned}
        &b \sqrt{\frac{(log( |\mcF| + \delta ) ) }{|\mcI^{c}|}}, &  \text{if setting (1) holds}   \\
        &\left(\sigma_{\epsilon^{(0)}} + b \right) \sqrt{\frac{(log( |\mcF| + \delta ) log(| \mcI^{c} |)) }{|\mcI^{c}|}}. &  \text{if setting (2) holds} 
    \end{aligned} \right. 
\end{equation*}

Let $\hat{\beta}_{a}$ be the output of Algorithm~\ref{algo: aggregate transfer algorithm}, then
\begin{equation}\label{sparse-agg oracle inequality}
    \mcE\left(\hat{\beta}_{a}\right) \leq \min_{t = 0,1,\cdots,T} \mcE\left(\hat{\beta}(\hat{\mcS}_{t})\right) + r_{\delta}(\mcF,n_{0})
\end{equation}
holds with probability at least $1-4e^{-\delta}$ where
\begin{equation*}
    r_{\delta}(\mcF,n_{0}) = \left\{ \begin{aligned}
        &C_{b1}  \frac{(1+\delta)log(T)}{n_0}, &  \text{if setting (1) holds}   \\
        &C_{b2} \frac{(1 + \log(4\delta^{-1}) log(T) log(n_{0})}{n_0}, &  \text{if setting (2) holds} 
    \end{aligned} \right. 
\end{equation*}
and $C_{b1},C_{b2}$ are some constants depend on $b$.
\end{proposition}
\begin{remark}
We call the setting (1) bounded setting and (2) sub-exponential setting. The latter one is milder but leads to a suboptimal cost. We refer readers to \citet{gaiffas2011hyper} for more detailed discussions about the optimal cost in sparse aggregation.
\end{remark}

\subsection{Lemmas}
In this part, we prove the lemmas that are used in the proof of Theorem~\ref{thm: upper bound for TL-FLR}. We prove them under the Assumption~\ref{assump: eigen assumption} condition~\ref{commute assumption} and let $\phi_{j}$ denote perfectly aligned eigenfunctions of $\Gamma^{(t)}$ with $t\in \{0\} \cup \mcS$.
\begin{lemma} \label{lemma1}
\begin{equation*}
    \left\| (\Gt)^{v}  \left( f_{\mcS\lambda_1} - f_{\mcS} \right)  \right\|_{L^2}^{2} \leq (1-v)^{2(1-v)}v^{2v}\lambda_{1}^{2v} \left\|f_{\mcS} \right\|_{L^2}^{2} \max_{j} \left\{  \left( \frac{s_{j}^{(0)}}{\sum_{t\in \mcS} \alpha_{t} s_{j}^{(t)}} \right)^{2v} \right\}.
\end{equation*}
\end{lemma}
\begin{proof}
By the definition of $f_{\mcS}$ and $f_{\mcS \lambda_1}$, \begin{equation*}
    \left( \sum_{t\in \mcS}\alpha_{t} \Gl + \lambda_1 I  \right) f_{\mcS\lambda_1}  = \sum_{t\in \mcS} \alpha_{t} \Gl \left(f^{(t)} \right)
    \quad \textit{and} \quad 
    \left( \sum_{t\in \mcS}\alpha_{t} \Gl \right) f_{\mcS}  = \sum_{t\in \mcS} \alpha_{t} \Gl \left(f^{(t)} \right)
\end{equation*}
then
\begin{equation*}
    f_{\mcS \lambda_{1}} - f_{\mcS} = - \left( \sum_{t\in \mcS}\alpha_{t} \Gl + \lambda_1 I \right)^{-1} \lambda_1 f_{\mcS}.
\end{equation*}
Hence,
\begin{equation*}
\begin{aligned}
    \left\| (\Gt)^{v}  \left( f_{\mcS\lambda_1} - f_{\mcS} \right)  \right\|_{L^2}^{2} 
    & \leq \lambda_1^{2} \left\| (\Gt)^{v} (\sum_{t\in \mcS}\alpha_{t} \Gl + \lambda_1 I )^{-1} \right\|_{op}^{2} \left\| f_{\mcS} \right\|_{L^{2}}^{2} \\
    & \leq \lambda_1^{2} \max_{j} \left\{ \frac{(s_{j}^{(0)})^{2v}}{ (\sum_{t\in \mcS} \alpha_{t} s_{j}^{(t)} + \lambda_1)^{2}  } \right\} \left\| f_{\mcS} \right\|_{L^2}^{2}
\end{aligned}
\end{equation*}
By Young's inequality, $\lambda_1 + \sum_{t\in \mcS} \alpha_{t} s_{j}^{(t)} \geq (1-v)^{-(1-v)} v^{-v} \lambda_1^{1-v} (\sum_{t\in \mcS} \alpha_{t} s_{j}^{(t)})^{v}$
\begin{equation*}
\begin{aligned}
    \left\| (\Gt)^{v}  \left( f_{\mcS\lambda_{1}} - f_{\mcS} \right)  \right\|_{L^2}^{2} 
    & \leq (1-v)^{2(1-v)}v^{2v}\lambda_{1}^{2v} \left\|f_{\mcS} \right\|_{L^2}^{2} \max_{j} \left\{  \left( \frac{s_{j}^{(0)}}{\sum_{t\in \mcS} \alpha_{t} s_{j}^{(t)}} \right)^{2v} \right\}.
\end{aligned}
\end{equation*}
\end{proof}

\begin{lemma}\label{lemma2}
\begin{equation*}
    \left\| (\Gt)^{v} \left( \sum_{t\in \mcS} \alpha_{t} \Gl + \lambda_1 \mathbf{I} \right)^{-1} \left( \sum_{t\in \mcS} \alpha_{t} \left(\Gl - \Gnl \right) \right) (\Gt)^{-v} \right\|_{op} = O_{\mbP} \left(  \left( n_{\mcS} \lambda_1^{1 - 2v + \frac{1}{2r}} \right)^{-\frac{1}{2}} \right)
\end{equation*}
\end{lemma}
\begin{proof}
\begin{equation*}
\begin{aligned}
    & \left\| (\Gt)^{v} \left( \sum_{t\in \mcS}\alpha_{t} \Gl + \lambda_1 \mathbf{I} \right)^{-1}  \left( \sum_{t\in \mcS} \alpha_{t} \left(\Gl - \Gnl \right) \right) (\Gt)^{-v} \right\|_{op} \\
    & = \sup_{h:\|h\|_{L^2}=1} \left| \left\langle h, (\Gt)^{v} \left( \sum_{t\in \mcS}\alpha_{t} \Gl + \lambda_1 \mathbf{I} \right)^{-1} \left( \sum_{t\in \mcS} \alpha_{t} \left(\Gl - \Gnl \right) \right)(\Gt)^{-v} h \right\rangle_{L^2} \right|.
\end{aligned}
\end{equation*}
Let 
\begin{equation*}
    h = \sum_{j\geq 1} h_j \phi_j,
\end{equation*}
then
\begin{equation*}
\begin{aligned}
    & \left\langle h, (\Gt)^{v} \left( \sum_{t\in \mcS}\alpha_{t} \Gl + \lambda_1 \mathbf{I} \right)^{-1} \left( \sum_{t\in \mcS} \alpha_{t} \left(\Gl - \Gnl \right) \right)(\Gt)^{-v} h \right\rangle_{L^2} \\
    & = \sum_{j,k} \frac{ (s_{j}^{(0)})^{v} (s_{k}^{(0)})^{-v}h_{j} h_{k}}{ \sum_{t\in \mcS} \alpha_{t} s_{j}^{(t)} + \lambda_1} \left\langle \phi_j, \sum_{t\in \mcS} \left(\Gl - \Gnl \right) \phi_{k} \right\rangle_{L^2}.
\end{aligned}
\end{equation*}
By Cauchy-Schwarz inequality, 
\begin{equation*}
\begin{aligned}
    & \left\| (\Gt)^{v} \left( \sum_{t\in \mcS}\alpha_{t} \Gl + \lambda_1 \mathbf{I} \right)^{-1}  \left( \sum_{t\in \mcS} \alpha_{t} \left(\Gl - \Gnl \right) \right)(\Gt)^{-v}  \right\|_{op} \\
    & \quad \leq \left( \sum_{j,k} \frac{ (s_{j}^{(0)})^{2v} (s_{k}^{(0)})^{-2v}}{(\sum_{t\in \mcS} \alpha_{t} s_{j}^{(t)} + \lambda_1)^2} \left\langle \phi_j, \sum_{t\in \mcS} \alpha_{t} \left(\Gl - \Gnl \right) \phi_{k} \right\rangle_{L^2}^{2} \right)^{\frac{1}{2}}. 
\end{aligned}
\end{equation*}
Consider $\E\langle \phi_j, \sum_{t\in \mcS} (\Gl - \Gnl ) \phi_{k} \rangle_{L^2}^{2}$, note that
\begin{equation*}
\begin{aligned}
    & \E\left\langle \phi_j, \sum_{t\in \mcS} \alpha_{t} \left(\Gl - \Gnl \right) \phi_{k} \right\rangle_{L^2}^2 \\ 
     & = \E \left( \sum_{t\in \mcS}\alpha_{t} \left\langle L_{K^{\frac{1}{2}}}(\phi_k), (C^{(t)} - L_{C_{n}^{(t)}})L_{K^{\frac{1}{2}}}(\phi_j) \right\rangle_{L^{2}} \right)^2 \\
     & = \E \left( \sum_{t\in \mcS}\alpha_{t} \frac{1}{n_{t}} \sum_{i=1}^{n_{t}} \int_{\mathcal{T}^2} L_{K^{\frac{1}{2}}}(\phi_k)(s) \left( X_i^{(t)}(s)X_i^{(t)}(t) - \E X_i^{(t)}(s)X_i^{(t)}(t) \right) L_{K^{\frac{1}{2}}}(\phi_j)(t) dt ds \right)^2 \\
     & \leq |\mcS| \sum_{t\in \mcS}  \frac{\alpha_{t}^{2}}{n_{t}} s_{j}^{(t)} s_{k}^{(t)}
\end{aligned}
\end{equation*}
By Jensen's inequality
\begin{equation*}
\begin{aligned}
    &\E \left( \sum_{j,k} \frac{ (s_{j}^{(0)})^{2v} (s_{k}^{(0)})^{-2v}}{(\sum_{t\in \mcS} \alpha_{t} s_{j}^{(t)} + \lambda_1)^2} \left\langle \phi_j, \sum_{t\in \mcS}\alpha_{t} \left(\Gl - \Gnl \right) \phi_{k} \right\rangle_{L^2}^{2} \right)^{\frac{1}{2}}\\
    & \leq \left( \sum_{j,k} \frac{ (s_{j}^{(0)})^{2v} (s_{k}^{(0)})^{-2v}}{(\sum_{t\in \mcS} \alpha_{t} s_{j}^{(t)} + \lambda_1)^2} \E \left\langle \phi_j, \sum_{t\in \mcS}\alpha_{t} \left(\Gl - \Gnl \right) \phi_{k} \right\rangle_{L^2}^{2} \right)^{\frac{1}{2}},
\end{aligned}
\end{equation*}
thus,
\begin{equation*}
\begin{aligned}
    &\E \left( \sum_{j,k} \frac{ (s_{j}^{(0)})^{2v} (s_{k}^{(0)})^{-2v}}{(\sum_{t\in \mcS} \alpha_{t} s_{j}^{(t)} + \lambda_1)^2} \left\langle \phi_j, \sum_{t\in \mcS}\alpha_{t} \left(\Gl - \Gnl \right) \phi_{k} \right\rangle_{L^2}^{2} \right)^{\frac{1}{2}}\\
    & \leq \left( \sum_{j,k} \frac{ (s_{j}^{(0)})^{2v} (s_{k}^{(0)})^{-2v}}{(\sum_{t\in \mcS} \alpha_{t} s_{j}^{(t)} + \lambda_1)^2} \left(\sum_{t\in \mcS} \alpha_{t} s_{j}^{(t)} s_{k}^{(t)}  \right) \frac{|\mcS|}{(n_{\mcS})} \right)^2 \\
    & \leq \max_{j,k}\left(\frac{\sum_{t\in \mcS} \alpha_{t} s_{j}^{(t)} s_{k}^{(t)}}{ s_{j}^{(0)} s_{k}^{(0)} }  \right)  \left( \sum_{j,k} \frac{ (s_{j}^{(0)})^{1+2v} (s_{k}^{(0)})^{1-2v}}{(\sum_{t\in \mcS} \alpha_{t} s_{j}^{(t)} + \lambda_1)^2}  \frac{|\mcS|}{(n_{\mcS})} \right)^2 
\end{aligned}
\end{equation*}
By assumptions of eigenvalues, $\max_{j,k}\left(\frac{\sum_{t\in \mcS} \alpha_{t} s_{j}^{(t)} s_{k}^{(t)}}{ s_{j}^{(0)} s_{k}^{(0)} }  \right) \leq C_{1}$ for some constant $C_{1}$. Finally, by Lemma~\ref{Tony's lemma6}
\begin{equation*}
    \E \left( \sum_{j,k} \frac{ (s_{j}^{(0)})^{2v} (s_{k}^{(0)})^{-2v}}{(\sum_{t\in \mcS} \alpha_{t} s_{j}^{(t)} + \lambda_1)^2} \left\langle \phi_j, \sum_{t\in \mcS}\alpha_{t} \left(\Gl - \Gnl \right) \phi_{k} \right\rangle_{L^2} \right)^{\frac{1}{2}} \lesssim \left( (n_{\mcS}) \lambda_{1}^{1 - 2v + \frac{1}{2r}} \right)^{-1}.
\end{equation*}

The rest of the proof can be completed by Markov's inequality.
\end{proof}

\begin{lemma} \label{lemma3}
\begin{equation*}
    \left\| (\Gt)^{v} ( \sum_{t\in \mcS}\alpha_{t} \Gl + \lambda_{1} \mathbf{I})^{-1} \sum_{t\in\mcS}g_{n}^{(t)} \right\|_{L^2}^2  = O_{\mbP}\left( \left( \left(n_{\mcS} \right) \lambda_1^{1-2v+\frac{1}{2r}} \right)^{-1} \right)
\end{equation*}
\end{lemma}
\begin{proof}
\begin{equation*}
\begin{aligned}
     \left\| (\Gt)^{v} ( \sum_{t\in \mcS}\alpha_{t} \Gl + \lambda_{1} \mathbf{I})^{-1}  \sum_{t\in\mcS}g_{n}^{(t)} \right\|_{L^2}^2
    &  = \sum_{j\geq 1} \left\langle (\Gt)^{v} ( \sum_{t\in \mcS}\alpha_{t} \Gl + \lambda_{1} \mathbf{I})^{-1}\sum_{t\in\mcS}g_{n}^{(t)},\phi_j \right\rangle_{L^2}^{2} \\
    & = \sum_{j\geq 1} \left\langle \sum_{t\in\mcS}g_{n}^{(t)},\frac{ (s_{j}^{(0)})^{v} }{ \sum_{t\in \mcS}\alpha_{t} s_{j}^{(t)} +\lambda_1} \phi_j \right\rangle_{L^2}^{2} \\
    & = \sum_{j\geq 1} \frac{ (s_{j}^{(0)})^{2v} }{( \sum_{t\in \mcS}\alpha_{t} s_{j}^{(t)} +\lambda_1)^2} \\  & \quad \qquad \left( \frac{1}{n_{\mcS}} \sum_{t\in \mcS}\sum_{i=1}^{n_{t}} \left\langle \epsilon_{i}^{(t)} X_{i}^{(t)}, L_{K^{\frac{1}{2}}}(\phi_j)  \right\rangle_{L^2} \right)^{2}.
\end{aligned}
\end{equation*}
Therefore,
\begin{equation*}
\begin{aligned}
    \E \left\| (\Gt)^{v} ( \sum_{t\in \mcS}\alpha_{t} \Gl + \lambda_{1} \mathbf{I})^{-1}  \sum_{t\in\mcS}g_{n}^{(t)} \right\|_{L^2}^2
     & = \sum_{j\geq 1} \frac{ (s_{j}^{(0)})^{2v} }{( \sum_{t\in \mcS}\alpha_{t} s_{j}^{(t)} +\lambda_1)^2}\cdot \\  & \quad \qquad \E \left( \frac{1}{n_{\mcS}} \sum_{t\in \mcS}\sum_{i=1}^{n_{t}} \left\langle \epsilon_{i}^{(t)} X_{i}^{(t)}, L_{K^{\frac{1}{2}}}(\phi_j)  \right\rangle_{L^2} \right)^{2} \\
     & = \sum_{j\geq 1} \frac{ (s_{j}^{(0)})^{2v} }{( \sum_{t\in \mcS}\alpha_{t} s_{j}^{(t)} +\lambda_1)^2} \frac{1}{(n_{\mcS})^2} \cdot\\
     & \quad \qquad \sum_{t\in \mcS} n_{t} \E \left( \langle \epsilon_{i}^{(t)} X_{i}^{(t)}, L_{K^{\frac{1}{2}}}(\phi_j) \rangle_{L^2} \right)^2  \\
    & = \sum_{j\geq 1} \frac{ (s_{j}^{(0)})^{2v} }{( \sum_{t\in \mcS}\alpha_{t} s_{j}^{(t)} +\lambda_1)^2} \frac{\left( \sum_{t\in\mcS}\sigma^{2}n_{t} s_{j}^{(t)} \right)}{(n_{\mcS})^2}  \\
    & \leq \max_{j} \left\{ \frac{\alpha_0 s_{j}^{(0)}+ \sum_{t\in \mcS}\alpha_{t} s_{j}^{(t)}}{s_{j}^{(0)}} \right\} \cdot \\
    & \quad \qquad \left( \frac{C_{1}}{n_{\mcS}} \sum_{j\geq 1} \frac{ (s_{j}^{(0)})^{1+2v} }{( \sum_{t\in \mcS}\alpha_{t} s_{j}^{(t)} +\lambda_1)^2} \right),
\end{aligned}
\end{equation*}
thus by assumption on eigenvalues and Lemma~\ref{Tony's lemma6} with $v = \frac{1}{2}$, 
\begin{equation*}
    \E \left\| (\Gt)^{v} ( \sum_{t\in \mcS}\alpha_{t} \Gl + \lambda_{1} \mathbf{I})^{-1}  \sum_{t\in\mcS}g_{n}^{(t)} \right\|_{L^2}^2 \lesssim \left( \left(n_{\mcS} \right)\lambda_1^{1 - 2v + \frac{1}{2r}} \right)^{-1},
\end{equation*}
with the constant proportional to $\sigma^{2}$. The rest of the proof can be completed by Markov's inequality.
\end{proof}

\begin{lemma}\label{lemma4}
\begin{equation*}
     \left\| \sum_{t\in\mcS} \alpha_{t} \Gnl \left( f^{(t)} - f_{\mcS } \right) \right\|_{L^2}^{2} = O_{\mbP} \left( (n_{\mcS})^{-1}\right)
\end{equation*}
\end{lemma}
\begin{proof}
\begin{equation*}
\begin{aligned}
    \E\left\| \sum_{t\in\mcS} \alpha_{t} \Gnl \left( f^{(t)} - f_{\mcS } \right) \right\|_{L^2}^{2}
    & = \sum_{j=1}^{\infty} \E \left( \sum_{t\in \mcS} \alpha_{t}  \left\langle C_n^{(t)} L_{K^{\frac{1}{2}}} (f^{(t)} - f_{\mcS }) , L_{K^{\frac{1}{2}}}(\phi_j) \right\rangle_{L^2}  \right)^2 \\
    & \lesssim  \sum_{j=1}^{\infty} \sum_{t\in \mcS} \frac{\alpha_{t}}{n_{\mcS}} \langle f^{(t)} - f_{\mcS }, \phi_j \rangle_{L^2}^2 (s_{j}^{(t)})^2 \\
    & \lesssim (n_{\mcS})^{-1} \max_{j,l}\left\{\alpha_{t} (s_{j}^{(t)})^2\right\} \sum_{t\in \mcS}\left\| f^{(t)} - f_{\mcS } \right\|_{L^2}^2 \\
     & \lesssim (n_{\mcS})^{-1},
\end{aligned}
\end{equation*}
with the universal constant proportional to $\|f_{\mcS}\|_{L^{2}}^2$. The rest of the proof can be completed by Markov's inequality.
\end{proof}

\begin{lemma}\label{Tony's lemma6}
\begin{equation*}
    \lambda_{1}^{-\frac{1}{2r}} \lesssim \sum_{j\geq 1} \frac{ (s_{j}^{(0)})^{1+2v} }{( \sum_{t\in \mcS}\alpha_{t} s_{j}^{(t)} +\lambda_1)^{1+2v}} \lesssim  1 + \lambda_{1}^{-\frac{1}{2r}}.
\end{equation*}
\end{lemma}
\begin{proof}
The proof is exactly the same as Lemma 6 in \citet{cai2012minimax} once we know that $\max_{j}\left(\frac{ s_{j}^{(0)}  } {\sum_{t\in \mcS} \alpha_{t} s_{j}^{(t)} } \right) \leq C$, which got satisfied under the assumptions of eigenvalues.
\end{proof}

\begin{lemma}[Fano's Lemma]\label{lemmaFano}
Let $P_{1},\cdots,P_{M}$ be probability measure such that 
\begin{equation*}
    KL(P_{i} | P_{j}) \leq \alpha, \quad 1 \leq i\neq j \leq K
\end{equation*}
then for any test function $\psi$ taking value in $\{ 1,\cdots,M \}$, we have 
\begin{equation*}
    P_i(\psi \neq i) \geq 1 - \frac{\alpha + log(2)}{log(M-1)}.
\end{equation*}
\end{lemma}

\begin{lemma}(Varshamov-Gilbert) \label{lemmaVG}
For any $N \geq 1$, there exists at least $M = exp(N/8)$ N-dimenional vectors, $b_1,\cdots,b_M \subset \{0,1\}^{N}$ such that
\begin{equation*}
    \sum_{l=1}^{N} \mathbf{1}\left\{b_{ik} \neq b_{jk} \right\} \geq N/4.
\end{equation*}
\end{lemma}

\section{Proofs for Section~\ref{sec: extension to FGLM}}\label{apd: theoretical section for FGLM}
We prove the upper bound and the lower bound of TL-FGLM. We first note that under Assumption~\ref{assump: FGLM bounded second derivative}, the excess risk $\mcE(\hat{\beta})$ is equivalent to $\E_{X^{(0)}} \langle \hat{\beta} - \beta^{(0)}, X^{(0)} \rangle_{L^{2}}^2$ up to universal constants. Thus we focus on $\E_{X^{(0)}} \langle \hat{\beta} - \beta^{(0)}, X^{(0)} \rangle_{L^{2}}^2$ in following proofs. 

Although we are focusing on $\E_{X^{(0)}} \langle \hat{\beta} - \beta^{(0)}, X^{(0)} \rangle_{L^{2}}^2$, which is exactly the same as FLR. However, minimizing the regularized negative log-likelihood will not provide an analytical solution of $\hat{\beta}$ as those in FLR, meaning that the proof techniques we used in proving TL-FLR and ATL-FLR are not applicable. Therefore, we use the empirical process to prove the upper bound. 

We abbreviate $\langle\cdot ,\cdot \rangle_{L^2{}}$ as $\langle\cdot , \cdot \rangle$ in following proofs. We first introduce some notations. Let
\begin{equation*}
\begin{aligned}
    & \mcL^{\mcS}(\beta) = \sum_{t\in \mcS} \alpha_{t} \E \left[ Y^{(t)} \langle X^{(t)}, \beta \rangle - \psi(\langle X^{(t)}, \beta \rangle )  \right], \\
    & \mcL(\beta) = \E \left[ Y^{(0)} \langle X^{(0)}, \beta + \hat{\beta}_{\mcS} \rangle - \psi(\langle X^{(0)}, \beta + \hat{\beta}_{\mcS} \rangle )  \right]
\end{aligned}
\end{equation*}
and their empirical version are denoted as 
\begin{equation*}
\begin{aligned}
    & \mcL_{n_{\mcS}}^{\mcS}(\beta) = \frac{1}{n_0 + n_{\mcS}} \sum_{t\in \mcS} \sum_{i=1}^{n_{t}} \left[ Y_{i}^{(t)} \langle X_{i}^{(t)}, \beta \rangle - \psi(\langle X_{i}^{(t)}, \beta \rangle )  \right], \\
    & \mcL_{n}(\beta) = \frac{1}{n_{0}} \sum_{i=1}^{n_{0}} \left[ Y_{i}^{(0)} \langle X_{i}^{(0)}, \beta + \hat{\beta}_{\mcS} \rangle - \psi(\langle X_{i}^{(0)}, \beta + \hat{\beta}_{\mcS} \rangle )  \right]
\end{aligned}
\end{equation*}
Let $P^{\mcS}$ and $P$ be the conditional distribution of $\cup_{t\in \mcS} Y^{(t)}|X^{(t)}$ and $Y^{(0)}|X^{(0)}$ respectively, and $P_{n_{\mcS}}^{\mcS}$ and $P_{n}$ as their empirical version, by define
\begin{equation*}
\begin{aligned}
    & \ell^{\mcS}(\beta) = \sum_{t\in \mcS} \alpha_{t}  \left[ Y^{(t)} \langle X^{(t)}, \beta \rangle - \psi(\langle X^{(t)}, \beta \rangle )  \right], \\
    & \ell(\beta) =  \left[ Y^{(0)} \langle X^{(0)}, \beta + \hat{\beta}_{\mcS} \rangle - \psi(\langle X^{(0)}, \beta + \hat{\beta}_{\mcS} \rangle )  \right]
\end{aligned}
\end{equation*}
we get 
\begin{equation*}
    P_{n_{\mcS}}^{\mcS}\ell^{\mcS}(\beta) = \mcL_{n_{\mcS}}^{\mcS}(\beta), \quad
    P^{\mcS}\ell^{\mcS}(\beta) = \mcL^{\mcS}(\beta), \quad
    P_{n} \ell(\beta) = \mcL_{n}(\beta), \quad
    P \ell(\beta) = \mcL(\beta).
\end{equation*}

\subsection{Proof of Upper bound for TL-FGLM (Theorem~\ref{thm: bounds for TL-FGLM})}
\begin{proof}
As mentioned before, we are focusing on $ \| \hat{\beta} - \beta^{(0)} \|_{C^{(0)}}^{2} $, i.e. 
\begin{equation*}
\begin{aligned}
    \| \hat{\beta} - \beta^{(0)} \|_{C^{(0)}}^{2}  & = \int_{\mcT} \int_{\mcT} (\hat{\beta}(s) - \beta^{(0)}(s) ) C^{(0)}(s,t) (\hat{\beta}(t) - \beta^{(0)}(t) ) ds dt \\
    & = \E \langle X^{(0)}, \hat{\beta} - \beta^{(0)} \rangle^{2}.
\end{aligned}
\end{equation*}
Therefore, we only need to show  $ \| \hat{\beta} - \beta^{(0)} \|_{C^{(0)}}^{2} $ is bounded by the error terms in Theorem~\ref{thm: bounds for TL-FGLM}. Notice that 
\begin{equation}\label{upperbound terms}
    \left\| \hat{\beta} - \beta^{(0)} \right\|_{C^{(0)}}  \leq \left\| \hat{\beta}_{\mcS} - \beta_{\mcS} \right\|_{C^{(0)}} + \left\| \hat{\delta}_{\mcS} - \delta_{\mcS} \right\|_{C^{(0)}},
\end{equation}
we then bound the two terms in the r.h.s. separately. We denote $\|a - b\|_{C^{(t)}}^{2}:=d_{t}^{2}(a,b)$ for all $t\in 0\cup[T]$ and $a,b \in \mcH_{K}$.

We first focus on the transfer learning error. Based on the Theorem 3.4.1 in \citep{vaart1996weak}, if the following three conditions hold,
\begin{enumerate}
    \item  
    $\E \sup_{\rho/2 \leq d_{0}(\beta,\beta_{\mcS})\leq \rho} \sqrt{n_{\mcS}} |(\mcL_{n_{\mcS}}^{\mcS} -\mcL^{\mcS})(\beta-\beta_{\mcS})| \lesssim \rho^{\frac{2r-1}{2r}} $;
    \item 
    $\sup_{\rho/2 \leq d_{0}(\beta,\beta_{\mcS})\leq \rho}P^{\mcS} \ell^{\mcS}(\beta)  - P^{\mcS} \ell^{\mcS}(\beta_{\mcS}) \lesssim -\rho^2 $;
    \item 
    $\mcL_{n_{\mcS}}^{\mcS}(\hat{\beta}_{\mcS})\geq \mcL^{\mcS}(\beta_{\mcS}) - O_{\mbP}\left(r_{n_{\mcS}}^{-2} \|\beta_{\mcS}\|_{K}^2 \right)$.
\end{enumerate}
then 
\begin{equation*}
    d_{0}^2(\hat{\beta}_{\mcS},\beta_{\mcS}) = O_{\mbP}(r_{n_{\mcS}}^{-2}\|\beta_{\mcS}\|_{K}^2).
\end{equation*}

For part (1), define
\begin{equation*}
    \Pi_{\rho}^{\mcS} = \{ \ell^{\mcS} (\beta) - \ell^{\mcS}(\beta_{\mcS}): \beta \in \mcB_{\rho} \} \quad \text{where} \quad \mcB_{\rho} = \{ \beta \in \mcH_{K}: d_{0}^{2}(\beta, \beta_{\mcS}) \in [\frac{\rho}{2}, \rho] \}.
\end{equation*}
Then $\sup_{\beta\in \mcB_{\rho}}|(\mcL_{n_{\mcS}}^{\mcS} -\mcL^{\mcS})(\beta-\beta_{\mcS})| = \sup_{f\in \Pi_{\rho}^{\mcS}} |(P_{n_{\mcS}}^{\mcS} - P^{\mcS})f|$ and by Cauchy-Schwarz inequality, 
\begin{equation*}
    \E \sup_{f \in \Pi_{\rho}^{\mcS}}|(P_{n_{\mcS}}^{\mcS} - P^{^{\mcS}})f | \leq \left\{ \E [ \sup_{f \in \Pi_{\rho}^{\mcS}}|(P_{n_{\mcS}}^{\mcS} - P^{\mcS})f | ]^2  \right\}^{1/2} := \left\| \sup_{f \in \Pi_{\rho}^{\mcS}}|(P_{n_{\mcS}}^{\mcS} - P^{\mcS})f | \right\|_{P^{\mcS},2}.
\end{equation*}
To bound the right hand side, by Theorem 2.14.1 in \citep{vaart1996weak}, we need to find the covering number of $\Pi_{\rho}^{\mcS}$, i.e. $\mcN(\epsilon,\Pi_{\rho}^{\mcS},\|\cdot\|_{P^{\mcS},2})$. We first show that 
\begin{equation*}
    log\left(\mcN(\epsilon,\Pi_{\rho}^{\mcS},\|\cdot\|_{P^{\mcS},2})\right) \leq O\left( \epsilon^{-\frac{1}{r}} log(\frac{\rho}{\epsilon}) \right).
\end{equation*}
Suppose there exist functions $\beta_1,\cdots,\beta_{M} \in \mcB_{\rho}$ such that
\begin{equation*}
    \min_{1\leq m \leq M} \|  \ell^{\mcS}(\beta) - \ell^{\mcS}(\beta_m) \|_{P^{\mcS},2} < \epsilon, \quad \forall  \beta \in \mathcal{B}_{\rho}.
\end{equation*}
Since
\begin{equation*}
    \begin{aligned}
     \left(\ell^{\mcS}(\beta)-\ell^{\mcS}(\beta_i)\right)^2 & = \left[ \sum_{t\in \mcS} \alpha_{t} Y^{(t)} \langle X^{(t)},\beta-\beta_i \rangle - \left( \psi(\langle X^{(t)},\beta \rangle) -\psi(\langle X^{(t)},\beta_i \rangle) \right) \right]^2\\
     & \leq |\mcS| \sum_{t\in \mcS} \alpha_{t}^{2} ((Y^{(t)})^2 + (C^{(t)})^2) \langle X^{(t)} , \beta - \beta_{i} \rangle_{\mcL^{2}}^{2} \\
     & \leq |\mcS| \max_{t\in \mcS}\left\{ (Y^{(t)})^2 + (C^{(t)})^2 \right\} \sum_{t\in \mcS} \alpha_{t}^{2} \langle X^{(t)} , \beta - \beta_{i} \rangle_{\mcL^{2}}^{2} \\
    \end{aligned}
\end{equation*}
thus 
\begin{equation*}
    \|l^{\mcS}(\beta)-l^{\mcS}(\beta_i) \|^2_{P^{\mcS},2} \leq |\mcS| \max_{t\in \mcS}\left\{ \E \psi'(\langle X^{(t)}, \beta^{(t)} \rangle) + (C^{(t)})^2 \right\} d_{0}^{2}(\beta,\beta_{i}) := C_{1} d_{0}^{2}(\beta,\beta_{i}),
\end{equation*}
where the inequality follows the fact for all $t\in \mcS$, and $d_{t}(\beta,\beta_{i}) \asymp d_{0}(\beta,\beta_{i})$ under Assumption~\ref{assump: eigen assumption}. Hence, the covering number of $\Pi_{\rho}^{\mcS}$ under norm $\|\cdot\|_{P^{\mcS},2}$ is bounded by covering number of $\mcB_{\rho}$ under norm $d_{0}$, i.e.
\begin{equation*}
     \mathcal{N}(\epsilon,\Pi_{\rho}^{\mcS}, \|\cdot\|_{P^{\mcS},2}) \leq  \mathcal{N}(\frac{\epsilon}{C_{1}},\mcB_{\rho}, d_{0}).
\end{equation*}
Define $\tilde{\mathcal{B}}_{\rho} = \left\{  \beta\in \mcH_{K} : d_{0}(\beta,\beta_\mcS) \in [0,\rho]  \right\}$, then 
\begin{equation*}
    \mathcal{N}(\frac{\epsilon}{C_1},\mathcal{B}_{\rho}, d_{0}) \leq \mathcal{N}(\frac{\epsilon}{C_1},\tilde{\mathcal{B}}_{\rho}, d_{0}).
\end{equation*}
Next, we will show $\mathcal{N}(\frac{\epsilon}{C_1},\tilde{\mathcal{B}}_{\rho}, d_{0})$ can be bounded by covering number for a ball in $\mathbb{R}^{J}$ for some finite integer $J$. Notice that $\mcH_{K} = L_{K^{1/2}}(L^2) = \{ \sum_{j \geq 1} b_j L_{K^{1/2}}(\phi_j) : (b_j)_{j\geq 1}\in \ell^2 \}$, hence for any $\beta = \sum_{j \geq 1} b_j L_{K^{1/2}}(\phi_j) \in \mcH_{K}$,
\begin{equation*}
    \begin{aligned}
     d_{0}^{2}(\beta,\beta_\mcS) &= \langle \beta-\beta_\mcS , L_{C^{(0)}} (\beta -\beta_{\mcS} )\rangle^{2} \\
     & = \sum_{j=1}^{\infty} \langle  b_{j} - b_{j}^{\mcS} , L_{K^{1/2} C K^{1/2}}( b_{j} - b_{j}^{\mcS} )  \rangle \\
     & = \sum_{j=1}^{\infty} s_j^{0} (b_{j} - b_{j}^{\mcS})^2
    \end{aligned}
\end{equation*}
which allows one to rewrite $\tilde{\mcB}_{\rho}$ as 
\begin{equation*}
    \tilde{\mcB}_{\rho} = \left\{  \sum_{j \geq 1} b_{j} L_{K^{1/2}}(\phi_{j}) :  \sum_{j=1}^{\infty} s_j^{0} (b_{j} - b_{j}^{\mcS})^2 \leq \rho^{2} \right\}.
\end{equation*}
Let $J = \lfloor ( \frac{\epsilon}{2C_{1}} )^{-\frac{1}{r}} \rfloor$ be a truncation number, and define
\begin{equation*}
    \tilde{\mcB}_{\rho}^{*} = \left\{  \sum_{j=1}^{J} b_{j} L_{K^{1/2}}(\phi_{j}) :  \sum_{j=1}^{J} s_j^{0} (b_{j} - b_{j}^{\mcS})^2 \leq \rho^{2} \right\}.
\end{equation*}
For any $\beta \in \tilde{\mcB}_{\rho}$, let $\beta^* \in \tilde{\mcB}_{\rho}^{*} $ be its counterpart, then
\begin{equation*}
    d_{0}^{2}(\beta,\beta^*) = \sum_{j=J+1}^{\infty} s_j^{0} b_{j}^{2} \leq s_{J}^{0} \sum_{j=J+1}^{\infty} b_{j}^{2} \asymp J^{-2r} = ( \frac{\epsilon}{2C_{1}})^{2}.
\end{equation*}
Suppose there exist function $\beta_1^*,\cdots,\beta_M^* \in \tilde{\mcB}_{\rho}^*$ such that
\begin{equation*}
    \min_{1\leq m \leq M}  d_{0}(\beta^*,\beta_m^{*})< \frac{\epsilon}{2C_{1}} \quad \forall \beta \in \tilde{\mathcal{B}}_{\rho}^*,
\end{equation*}
then by triangle inequality
\begin{equation*}
    \min_{1\leq m \leq M}  d_{0}(\beta,\beta_i^{*})< \frac{\epsilon}{C_1} \quad  \forall \beta \in \mathcal{B}_{\rho}.
\end{equation*}
The above inequality indeed shows that the covering number of $\tilde{\mcB}_{\rho}$ with radius $\frac{\epsilon}{C_{1}}$ can be bounded by the covering of  $\tilde{\mcB}_{\rho}^*$ with radius $\frac{\epsilon}{2C_{1}}$, i.e.
\begin{equation*}
    \mathcal{N}(\frac{\epsilon}{C_1},\tilde{\mathcal{B}}_{\rho}, d_{0}) \leq \mathcal{N}(\frac{\epsilon}{2C_1},\tilde{\mathcal{B}}_{\rho}^*, d_{0}).
\end{equation*}
It is known that the covering number for a unit ball in $\mathbb{R}^N$, then the covering number is less than $( \frac{2}{\epsilon} + 1)^{N}$. Therefore, 
\begin{equation*}
    \mathcal{N}(\frac{\epsilon}{2C_1},\tilde{\mathcal{B}}_{\rho}^*, d_{0}) \leq 
\left(  \frac{2\rho + \frac{\epsilon}{2C_1}}{\frac{\epsilon}{2C_1}} \right)^{J}
\end{equation*}
which leads to 
\begin{equation*}
log \mathcal{N}(\frac{\epsilon}{2C_1},\tilde{\mathcal{B}}_{\rho}^*, d_{0}) \leq 
O_{\mbP} \left(  \epsilon^{-\frac{1}{r}} log( \frac{\rho}{\epsilon} )    \right).
\end{equation*}
By the Dudley entropy integral, we know 
\begin{equation*}
    \begin{aligned}
     \sup_{f\in L_{\rho}^{\mcS}} |(P_{n_{\mcS}}^{\mcS}-P^{\mcS})f| & \lesssim \int_{0}^{\rho} \sqrt{\frac{log \mathcal{N}(\frac{\epsilon}{2C_1},\tilde{\mathcal{B}}_{\rho}^*, d_{0}(\cdot,\cdot))}{n_{\mcS}}} d\epsilon \\
     & = \rho^{\frac{2r-1}{2r}} n_{\mcS}^{-\frac{1}{2}} \int_{1}^{\infty}exp\{ (1-\frac{1}{2h}) u^2 \}u^2 du \\
     & = O(\rho^{\frac{2r-1}{2r}} n_{\mcS}^{-\frac{1}{2}})
    \end{aligned}
\end{equation*}
Hence, by Theorem 2.14.1 in \citep{vaart1996weak}, we finish the proof of (1).

For part (2), let $G(t) = \ell^{\mcS}(\beta_{\mcS} + t \tilde{\beta})$ where $\tilde{\beta} = \beta - \beta_{\mcS}$, then we notice $G(1) = \ell^{\mcS}(\beta)$ and $G(0) = \ell^{\mcS}(\beta_{\mcS})$. We further notice 
\begin{equation*}
    G^{'}(t) = - \sum_{t\in \mcS} \alpha_{t} \left\{ Y^{(t)} \langle X^{(t)}, \tilde{\beta} \rangle - \psi'( \langle X^{(t)}, \beta_{\mcS} + t \tilde{\beta} \rangle ) \langle X^{(t)}, \tilde{\beta} \rangle \right\}
\end{equation*}
and thus 
\begin{equation*}
\begin{aligned}
     \E G^{'}(0) & =  \sum_{t\in \mcS} \alpha_{t} \E \left\{ Y^{(t)} \langle X^{(t)}, \tilde{\beta} \rangle - \psi'( \langle X^{(t)}, \beta_{\mcS}  \rangle ) \langle X^{(t)}, \tilde{\beta} \rangle \right\} \\
     & =   \sum_{t\in \mcS} \alpha_{t} \E \left\{ \E \left\{Y^{(t)}  - \psi'( \langle X^{(t)}, \beta_{\mcS}  \rangle )  | X^{(t)} \right\} \langle X^{(t)}, \tilde{\beta} \rangle \right\}\\
     & = 0
\end{aligned}
\end{equation*}
Besides, by direct calculation, 
\begin{equation*}
    G^{''}(t) = -\sum_{t\in \mcS} \alpha_{t} \left\{ \psi''( \langle X^{(t)}, \beta_{\mcS} +  t \tilde{\beta}  \rangle ) \langle X^{(t)}, \tilde{\beta} \rangle^{2} \right\}.
\end{equation*}
By Taylor expansion, there exists a $\gamma \in [0,1]$ such that
\begin{equation*}
\begin{aligned}
    G(1) - G(0) & = G^{'}(0) + \frac{1}{2}G^{''}(\gamma)\\
    & = G^{'}(0) - \frac{1}{2}\sum_{t\in \mcS} \alpha_{t} \left\{ \psi''( \langle X^{(t)}, \beta_{\mcS} + \gamma \tilde{\beta} \rangle ) \langle X^{(t)}, \tilde{\beta} \rangle^{2} \right\}.
\end{aligned}
\end{equation*}
Notice that $P^{\mcS}\ell^{\mcS}(\beta) - P^{\mcS}\ell^{\mcS}(\beta_{\mcS})= \E [G(1) - G(0)]$, and then
\begin{equation*}
\begin{aligned}
    P^{\mcS}\ell^{\mcS}(\beta) - P^{\mcS}\ell^{\mcS}(\beta_{\mcS}) &= \E [G(1) - G(0)]\\
    & =  -\frac{1}{2}\sum_{t\in \mcS} \alpha_{t} \E \left\{ \psi''( \langle X^{(t)}, \beta_{\mcS} + \gamma \tilde{\beta} \rangle ) \langle X^{(t)}, \tilde{\beta} \rangle^{2} \right\} \\
    & \leq -\frac{\min_{t\in \mcS}\{ \mcA_{1} \} }{2} \sum_{t\in \mcS} \alpha_{t} \langle X^{(t)} , \tilde{\beta} \rangle^{2},
\end{aligned}
\end{equation*}
and 
\begin{equation*}
\begin{aligned}
    P^{\mcS}\ell^{\mcS}(\beta) - P^{\mcS}\ell^{\mcS}(\beta_{\mcS}) &= \E [G(1) - G(0)]\\
    & =  -\frac{1}{2}\sum_{t\in \mcS} \alpha_{t} \E \left\{ \psi''( \langle X^{(t)}, \beta_{\mcS} + \gamma \tilde{\beta} \rangle ) \langle X^{(t)}, \tilde{\beta} \rangle^{2} \right\} \\
    & \geq -\frac{\max_{t\in \mcS}\{ \mcA_{2} \} }{2} \sum_{t\in \mcS} \alpha_{t} \langle X^{(t)} , \tilde{\beta} \rangle^{2},
\end{aligned}
\end{equation*}
which leads to 
\begin{equation*}
    P^{\mcS}\ell^{\mcS}(\beta) - P^{\mcS}\ell^{\mcS}(\beta_{\mcS}) \asymp - \sum_{t\in \mcS} \alpha_{t} d_{t}^{2}(\beta, \beta_{\mcS})
\end{equation*}
Hence, we get 
\begin{equation*}
\begin{aligned}
    \sup_{\rho/2 \leq d_{0}(\beta,\beta_0)\leq \rho} \left\{ P^{\mcS}\ell^{\mcS}(\beta) - P^{\mcS}\ell^{\mcS}(\beta_{\mcS}) \right\} & \asymp - \left(\rho^2 +  \sum_{t\in \mcS}\alpha_{t}d_{t}^{2}(\beta, \beta_{\mcS}) \right) \\
    & \lesssim -\rho^{2},
\end{aligned}
\end{equation*}
which proves part (2).

Finally for part (3), we pick $r_{n_{\mcS}} = n_{\mcS}^{\frac{r}{2r+1}} \|\beta_{\mcS}\|_{K}^{-\frac{2r}{2r+1}} $ which satisfies $r_{n_{\mcS}}^{2} \phi_{n}(r_{n_{\mcS}}^{-1}) \leq \sqrt{n_{\mcS}}$ where $\phi_{n}(x) = \|\beta_{\mcS}\|_{K} x^{\frac{2r-1}{2r}}$. Let $\lambda_{1} = O(r_{n_{\mcS}}^{-2})$, 
since 
\begin{equation*}
    -\mcL_{n}^{\mcS}(\hat{\beta}_{\mcS}) + \lambda_{1} \|\hat{\beta}_{\mcS}\|_{K}^2 \leq -\mcL_{n}^{\mcS}(\beta_{\mcS}) + \lambda_{1} \|\beta_{\mcS}\|_{K}^2,
\end{equation*}
hence
\begin{equation*}
\begin{aligned}
    \mcL_{n}^{\mcS}(\hat{\beta}_{\mcS}) & \geq \mcL_{n}^{\mcS}(\beta_{\mcS}) + \lambda_{1} \left( \|\hat{\beta}_{\mcS}\|_{K}^2 - \|\beta_{\mcS}\|_{K}^2 \right) \\
    & \geq \mcL_{n}^{\mcS}(\beta_{\mcS}) - \lambda_{1}  \|\beta_{\mcS}\|_{K}^2 \\
    &\geq \mcL_{n}^{\mcS}(\beta_{\mcS}) - O(r_{n_{\mcS}}^{-2}\|\beta_{\mcS}\|_{K}^2).
\end{aligned}
\end{equation*}
Combining part (1)-(3), based on the Theorem 3.4.1 in \citep{vaart1996weak}, we know
\begin{equation*}
    d_{0}^{2}(\hat{\beta}_{\mcS}, \beta_{\mcS}) = O_{p} \left( r_{n_{\mcS}}^{-2}\|\beta_{\mcS}\|_{K}^2) \right).
\end{equation*}

To bound the second term in the r.h.s. of (\ref{upperbound terms}), we follow the same proof procedure as the proof of bounding the first term. Specifically, we need to show
\begin{enumerate}
    \item  
    $\E \sup_{\rho/2 \leq d_{0}(\delta,\delta_{\mcS})\leq \rho} \sqrt{n_{0}} |(\mcL_{n_{0}} -\mcL)(\delta-\delta_{\mcS})| \lesssim \rho^{\frac{2r-1}{2r}} $;
    \item 
    $\sup_{\rho/2 \leq d_{0}(\delta,\delta_{\mcS})\leq \rho}P \ell(\delta)  - P \ell(\delta_{\mcS}) \lesssim -\rho^2 $;
    \item 
    $\mcL_{n_{0}}(\hat{\delta}_{\mcS})\geq \mcL(\delta_{\mcS}) - O_{\mbP}\left(r_{n_{0}}^{-2} \|\delta_{\mcS}\|_{K}^2 \right)$.
\end{enumerate}
It is not hard to check that including the estimator from the transfer step $\hat{\beta}_{\mcS}$ into the loss function for the calibration step defined at the beginning of the proof will not affect the statements (1)-(3). For example, in part (1), the $\hat{\beta}_{\mcS}$ will vanish when calculating $(\ell(\delta)-\ell(\delta_i))^2$; in part (2), its effect will vanish since our assumption of the second order derivatives of $\psi$s are bounded from infinity and zero; in part (3), the inequality holds as $\hat{\delta_{S}}$ is the minimizer of the regularized loss function. Therefore, in the end, we have 
\begin{equation*}
    d_{0}^{2}(\hat{\delta}_{\mcS}, \delta_{\mcS}) = O_{\mbP}( r_{n_{0}}^{-2}\|\delta_{\mcS}\|_{K}^2 )
\end{equation*}
where $r_{n_{0}} = n_{0}^{\frac{r}{2r+1}} \|\delta_{\mcS}\|_{K}^{-\frac{2r}{2r+1}}$.

Combining the bounds of $d_{0}(\hat{\beta}_{\mcS}, \beta_{\mcS})$ and $d_{0}(\hat{\delta}_{\mcS}, \delta_{\mcS})$, we reach to 
\begin{equation*}
    \mcE(\hat{\beta}) = O_{p} \left( n_{\mcS}^{-\frac{2r}{2r+1}} + \left(\frac{h^2}{\|\beta_{\mcS}\|_{K}^2}\right)^a n_{0}^{-\frac{2r}{2r+1}} \right),
\end{equation*}
for some $a>0$.

\end{proof}

\subsection{Proof of Lower Bound for TL-FGLM (Theorem~\ref{thm: bounds for TL-FGLM})}
\begin{proof}
Similar to the lower bound of TL-FLR, we consider the following two cases.

(1) Consider $h=0$, i.e., all the source datasets come from the target domain, and thus $\beta^{(t)} = \beta^{(0)}$ for all $t\in \mcS$. This can also be viewed as finding the lower bound of the estimator on $\beta^{(0)}$ with the size of the target dataset as $n_0 + n_{\mcS}$.

We first calculate the Kullback–Leibler divergence between $P_{i}$ and $P_{j}$ under the exponential family. By the definition of KL divergence and the density function of the exponential family, we have
\begin{equation*}
    \begin{aligned}
     KL(P_{i}||P_{j}) & = (n_{0}+n_{\mcS}) \E \bigg\{ \langle X^{(0)} , \beta_{i} -\beta_{j} \rangle \psi'(\langle X^{(0)} , \beta_{i}  \rangle) \\
     & \qquad \qquad \qquad \qquad - \bigg( \psi(\langle X^{(0)} , \beta_{i}  \rangle) - \psi(\langle X^{(0)} , \beta_{j}  \rangle) \bigg)  \bigg\}\\
     & = (n_{0}+n_{\mcS}) \E \bigg\{ \frac{1}{2}\psi''(\langle X^{(0)} , \hat{\beta} \rangle)< X^{(0)} , \beta_i - \beta_j  >^2 \bigg\}\\
     & \lesssim (n_{0}+n_{\mcS}) d_{0}^{2}(\beta_i,\beta_j),
    \end{aligned}
\end{equation*}
for some $\hat{\beta}$ between $\beta_{i}$ and $\beta_{j}$. For any estimator $\tilde{\beta}$ based on $\{(X_{i}^{(0)},Y_{i}^{(0)})\}_{i=1}^{n_{0} + n_{\mcS}}$, by Markov's inequality and Fano's Lemma, we have
\begin{equation} \label{GLM fanolowerbound}
    \begin{aligned}
     d_{0}^{2}(\tilde{\beta},\beta_i) & \geq P_{i} \left( | \langle X^{(0)},\tilde{\beta}-\beta_i \rangle | \geq \min_{i,j}d_{0}(\beta_i,\beta_j) \right) \min_{i,j}d_{0}^{2}(\beta_i,\beta_j)\\
     & = P_{i}(\tilde{\beta} \neq \beta_i)  \min_{i,j}d_{0}^{2}(\beta_i,\beta_j)\\
     & \geq \bigg( 1 - \frac{ (n_{0}+n_{\mcS}) \max_{i,j}d_{0}^{2}(\beta_i,\beta_j) + log(2)}{log(M-1)}\bigg) \min_{i,j}d_{0}^{2}(\beta_i,\beta_j).
    \end{aligned}
\end{equation}
To have the lower bound matches with the upper bound, we need to construct the series of $\beta_{1}, \beta_{M} \in \mcB_{\mcH}{h}$ such that the r.h.s. of the above inequality is equal to $(n_0 + n_{\mcS})^{-\frac{2r}{2r+1}}$ up to a constant. Let $N$ be a fixed integer and 
\begin{equation*}
    \beta_{i} = \sum_{k=N+1}^{2N} \frac{b_{i,k-N} }{\sqrt{N}} L_{K^{\frac{1}{2}}}(\phi_{k}) \quad \textit{for} \quad i = 1,2,\cdots,M.
\end{equation*}
Then,
\begin{equation*}
\begin{aligned}
    d_{0}^{2}(\beta_{i}, \beta_{j}) & = \frac{1}{N} \sum_{k=N+1}^{2N} \left(b_{i,k-N} - b_{j,k-N} \right)^2 s_l^{(0)} \\
    & \geq \frac{ s_{2N}^{(0)} }{N} \sum_{k=N+1}^{2N} \left(b_{i,k-N} - b_{j,k-N} \right)^2 \\
    & \geq \frac{ s_{2N}^{(0)}}{4}  ,
\end{aligned}
\end{equation*}
where the last inequality is by Lemma~\ref{lemmaVG}, and 
\begin{equation*}
\begin{aligned}
    d_{0}^{2}(\beta_{i}, \beta_{j}) & = \frac{1}{N} \sum_{k=N+1}^{2N} \left(b_{i,k-N} - b_{j,k-N} \right)^2 s_{k}^{(0)} \\
    & \leq \frac{ s_{N}^{(0)} }{N} \sum_{k=M+1}^{M} \left(b_{i,k-N} - b_{j,k-N} \right)^2 \\
    &  \leq   s_{N}^{(0)}.
\end{aligned}
\end{equation*}
Combining the upper and lower bound of $d_{0}^{2}(\beta_{i}, \beta_{j})$ with the r.h.s. of (\ref{GLM fanolowerbound}), we obtain
\begin{equation*}
    d_{0}^{2}(\beta_{i}, \beta_{j}) \geq    \left( 1 - \frac{ 4 (n_0+n_{\mcS})  s_{N}^{(0)} + 8log(2)}{N} \right) \frac{s_{2N}^{(0)} }{4}.
\end{equation*}
Taking $N = 8(n_0 + n_{\mcS})^{\frac{1}{2r+1}}$, which implies $N\rightarrow \infty$, would produce
\begin{equation*}
    \left( 1 - \frac{4 (n_0+n_{\mcS})  s_{N}^{(0)} + 8log(2)}{N} \right) \frac{s_{2N}^{(0)} }{4} \asymp
    \left( \frac{1}{2} - \frac{8log(2)}{N} \right)  N^{-2r} \asymp (n_{\mcS} + n_{0})^{-\frac{2r}{2r+1}} 
\end{equation*}
Now we finish the proof of the first half of the lower bound. To prove the second half, we consider the following case.

(2) Consider $\beta^{(t)} = 0$ for all $t\in \mcS$, i.e., all the source domains have no useful information about the target domain. Then we know the $\beta^{(0)} \in \mcB_{\mcH}(h)$, and our goal is to show $d_{0}^{2}(\tilde{\beta}, \beta_{i})$ is bounded by $n_{0}^{-\frac{2r}{2r+1}}$ up to a constant related to $h$ by constructing a sequence of $\beta_{1},\cdots,\beta_{M} \in \mcB_{\mcH}(h)$.

Again, let $N$ be a fixed integer and 
\begin{equation*}
    \beta_{i} = \sum_{k=N+1}^{2N} \frac{b_{i,k-N} h }{\sqrt{N}} L_{K^{\frac{1}{2}}}(\phi_{k}) \quad \textit{for} \quad i = 1,2,\cdots,M.
\end{equation*}
Then, similar to case (1), we can prove that
\begin{equation*}
    d_{0}^{2}(\beta_{i}, \beta_{j}) \geq \frac{ s_{2N}^{(0)} h^{2}}{4} \quad \textit{and} \quad 
    d_{0}^{2}(\beta_{i}, \beta_{j})\leq   s_{N}^{(0)} h^{2}.
\end{equation*}
Then for any estimator $\tilde{\beta}$ based on $\{(X_{i}^{(0)},Y_{i}^{(0)})\}_{i=1}^{n_{0} }$ and follows a similar process as case (1),
\begin{equation*}
    d_{0}^{2}(\beta_{i}, \beta_{j}) \geq    \left( 1 - \frac{ 4 n_0 h^2 s_{N}^{(0)} + 8log(2)}{N} \right) \frac{s_{2N}^{(0)} h^2}{4}.
\end{equation*}
Again, taking $N = 8h^{2} (n_0 + n_{\mcS})^{\frac{1}{2r+1}}$ leads to 
\begin{equation*}
    d_{0}^{2}(\beta_{i}, \beta_{j}) \gtrsim n_{0}^{-\frac{2r}{2r+1}} h^{2}
\end{equation*}

Combining the lower bound in case (1) and case (2), we obtain the desired lower bound.
\end{proof}

\section{Additional Experiments for TL-FLR/ATL-FLR}\label{appendix: additional experiments}

In this section, we explore how the smoothness of the coefficient functions, $\beta^{(t)}$, for $t\in \mcS^{c}$, will affect the performance of ATL-FLR. We also explore how different temperatures will affect the performance of Exponential Weighted ATL-FLR (ATL-FLR (EW)).

We consider the setting that $\beta^{(t)}$ with $t\in \mcS^{c}$ are generated from a much rougher Gaussian process, i.e. $\beta_{t}$ are generated from a Gaussian process with mean function $\cos(2\pi t)$ with covariance kernel $\min(s,t)$, which is exactly Wiener process, and thus the $\beta_{t}$s are less smooth than $\beta_{t}$s that are generated from Ornstein–Uhlenbeck process (the one we used in main paper). For ATL-FLR (EW), we consider three different temperatures, i.e., $T = 0.2, 2, 10$, where a lower temperature will usually produce small aggregation coefficients. All the other settings are the same as the simulation section.

The results are presented in Figure~\ref{fig: diff_temp ATL-FLR}. In general, the patterns of using the Wiener process are consistent with using the Ornstein–Uhlenbeck process, which demonstrates the robustness of the proposed algorithms to negative transfer source models. We also note that while the temperature is low ($T = 0.2$), the small convex combination coefficients $\{c_{j}\}$ will make ATL-FLR(EW) have almost the same performance as ATL-FLR, but it still cannot beat ATL-FLR. While we increase the temperature ($T = 2, T = 10$), the gap between ATL-FLR(EW) and ATL-FLR increases, especially when the proportion of $|\mcS|$ is small. Therefore, selecting the wrong $T$ can hugely degrade the performance of ATL-FLR(EW). This demonstrates the superiority of sparse aggregation in practice since its performance does not depend on the selection of any hyperparameters.

\begin{figure}[htbp]
    \centering
    \subfloat[Low Temperature $(T = 0.2)$]{
    \includegraphics[page = 1, width=0.5\textwidth]{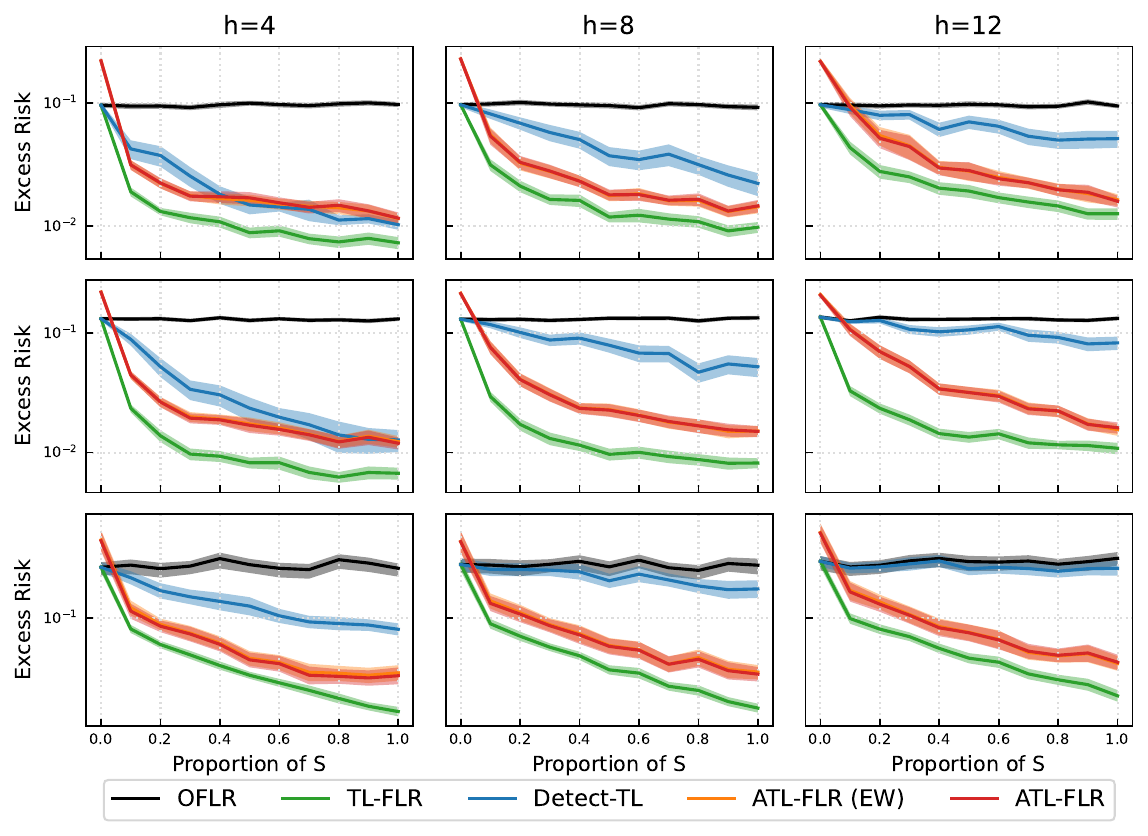}
    }

    \subfloat[Middle Temperature $(T = 2)$]{
    \includegraphics[page = 1, width=0.5\textwidth]{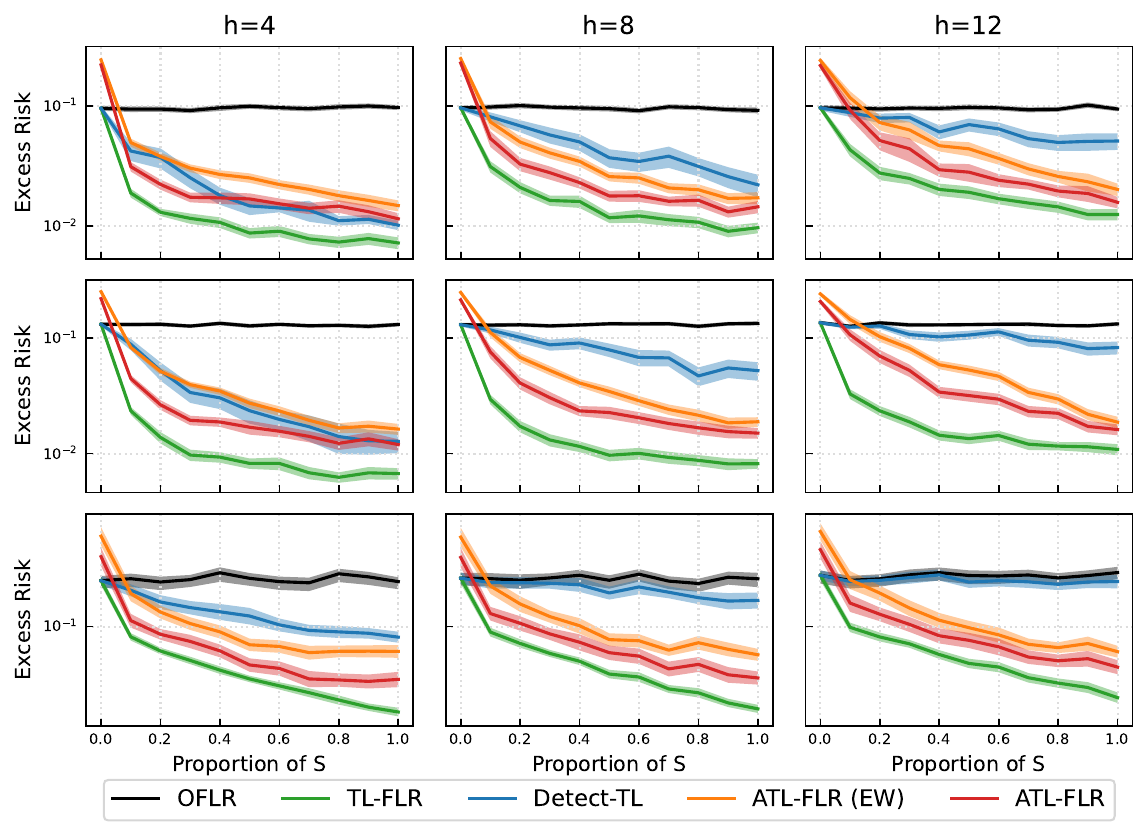}
    }

    \subfloat[High Temperature $(T = 10)$]{
    \includegraphics[page = 1, width=0.5\textwidth]{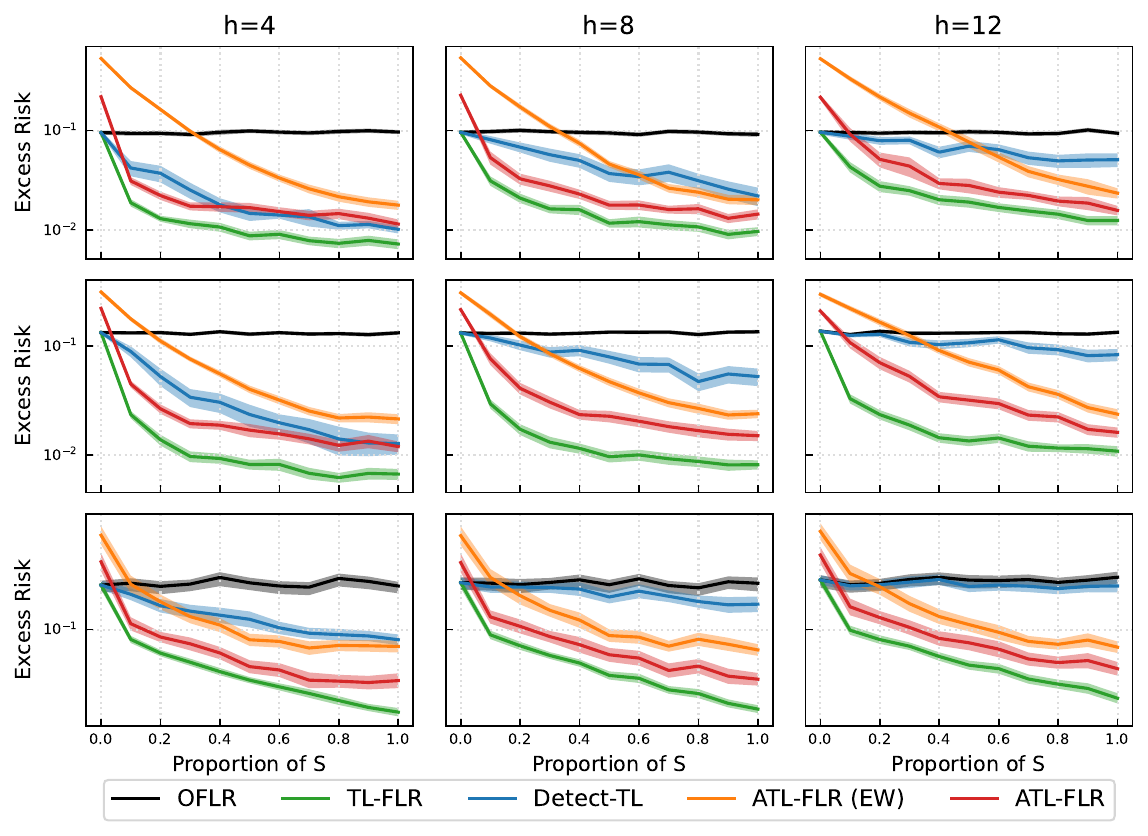}
    }
    
    \caption{Excess Risk of different transfer learning algorithms. Each row corresponds to a different $\beta^{(0)}$, and the y-axes for each row are under the same scale. The result for each sample size is an average of 100 replicate experiments, with the shaded area indicating $\pm$ two standard errors.}
    \label{fig: diff_temp ATL-FLR}
\end{figure}

\end{document}